\documentclass[11pt]{book}

\usepackage{amsmath}
\usepackage{amsthm}
\usepackage{amsfonts}
\usepackage{amssymb}
\usepackage{url}
\usepackage{hyperref}
\usepackage{algorithm,algorithmic}
\usepackage{xspace}
\usepackage{color}
\usepackage{mathtools}
\usepackage{enumerate}

\newtheorem{theorem}{Theorem}[chapter]
\newtheorem{definition}[theorem]{Definition}

\newtheorem{proposition}[theorem]{Proposition}
\newtheorem{corollary}[theorem]{Corollary}
\newtheorem{lemma}[theorem]{Lemma}
\newtheorem{remark}[theorem]{Remark}

\newtheorem{fact}[theorem]{Fact}
\newtheorem{assumption}[theorem]{Assumption}

\DeclareMathOperator*{\argmin}{\arg\min}

\def\hestinv{\tilde{\nabla}^{-2}}

\def\hessinv{\nabla^{-2}}
\def\hess{\nabla^2}

\newcommand{\defeq}{\triangleq}

\def\H{{\mathcal H}}

\def\D{{\mathcal D}}

\def\F{{\mathcal F}}

\def\reals{{\mathbb R}}
\def\R{{\mathcal R}}

\newcommand{\proj}{\mathop{\Pi}}

\newcommand{\err}{\mathop{\mbox{\rm error}}}
\newcommand{\rank}{\mathop{\mbox{\rm rank}}}
\newcommand{\ellipsoid}{{\mathcal E}}

\newcommand{\sign}{\mathop{\mbox{\rm sign}}}
\newcommand{\poly}{\mathop{\mbox{\rm poly}}}

\newcommand{\ignore}[1]{}

\newcommand{\equaltri}{\triangleq}

\def\trace{{\bf Tr}}

\def\reals{{\mathbb R}}
\newcommand{\E}{\mathop{\mbox{\bf E}}}

\def\bzero{\mathbf{0}}

\def\bold0{\mathbf{0}}

\newcommand\mycases[4] {{
\left\{
\begin{array}{ll}
    {#1}, & {#2} \\\\
    {#3}, & {#4}
\end{array}
\right. }}
\newcommand\mythreecases[6] {{
\left\{
\begin{array}{ll}
    {#1}, & {#2} \\\\
    {#3}, & {#4} \\\\
    {#5}, & {#6}
\end{array}
\right. }}
\def\bb{\mathbf{b}}

\def\bB{\mathbf{B}}
\def\bC{\mathbf{C}}

\def\bx{\mathbf{x}}
\def\bh{\mathbf{h}}
\def\w{\mathbf{w}}
\def\by{\mathbf{y}}

\def\bv{\mathbf{v}}

\def\ba{\mathbf{a}}
\def\bA{\mathbf{A}}
\def\bB{\mathbf{B}}
\def\bC{\mathbf{C}}
\def\bD{\mathbf{D}}

\def\bone{\mathbf{1}}

\def\xbar{\bar{\mathbf{x}}}

\def\trace{{\bf Tr}}

\newcommand{\eps}{\varepsilon}

\def\bone{\mathbf{1}}

\newcommand{\diag}{\mbox{diag}}
\newcommand{\K}{\ensuremath{\mathcal K}}

\newcommand{\x}{\ensuremath{\mathbf x}}

\newcommand{\y}{\ensuremath{\mathbf y}}
\newcommand{\z}{\ensuremath{\mathbf z}}
\newcommand{\h}{\ensuremath{\mathbf h}}

\newcommand{\xv}[1][t]{\ensuremath{\mathbf x_{#1}}}

\newcommand{\yv}[1][t]{\ensuremath{\mathbf y_{#1}}}

\newcommand{\fv}[1][t]{\ensuremath{\mathbf f_{#1}}}

\newcommand{\uv}{\ensuremath{\mathbf u}}
\newcommand{\vv}{\ensuremath{\mathbf v}}

\def\regret{\ensuremath{\mathrm{{regret}}}}
\def\bb{\mathbf{b}}

\def\bx{\mathbf{x}}
\def\by{\mathbf{y}}

\def\bv{\mathbf{v}}
\def\ba{\mathbf{a}}

\def\eps{\varepsilon}
\def\epsilon{\varepsilon}

\def\ogd{{online gradient descent}\xspace}
\def\R{\ensuremath{\mathcal R}}

\title{ {\it lecture notes: }   \\  Optimization for Machine Learning \\ {\small \it version 0.57}\\  \bigskip
{\small All rights reserved.  }
}

\date{}
\author{
Elad Hazan \thanks{\url{www.cs.princeton.edu/\~ehazan}}}

\begin{document}

\frontmatter  

\maketitle

\chapter*{Preface}
     \addcontentsline{toc}{chapter}{Preface}
     \markboth{\sffamily\slshape Preface}
       {\sffamily\slshape Preface}

This text was written to accompany a series of lectures given at the Machine Learning Summer School Buenos Aires, following a lecture series at the Simons Center for Theoretical Computer Science, Berkeley.   It was extended for the course COS 598D - Optimization for Machine Learning, Princeton University, Spring 2019. 

I am grateful to Paula Gradu for proofreading parts of this manuscript. I'm also thankful for the help of the following students and colleagues for corrections and suggestions to this text: Udaya Ghai, John Hallman, No\'{e} Pion, Xinyi Chen. 

\begin{figure}[!htb]
        \center{\includegraphics[scale = 1]
        {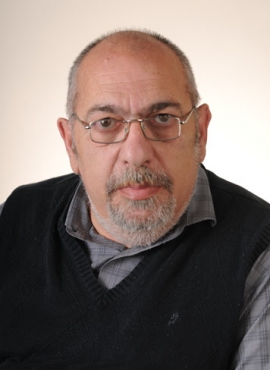}}
        \caption{\label{fig:my-label} Professor Arkadi Nemirovski, Pioneer of mathematical optimization}
\end{figure}

\tableofcontents

\mainmatter

\chapter*{Notation }

We use the following mathematical notation in this writeup:
\begin{itemize}
\item
$d$-dimensional  Euclidean space is denoted $\reals^d$. 
\item
Vectors are denoted by boldface lower-case letters such as $\x \in \reals^d$.  Coordinates of vectors are denoted by underscore notation $\x_i$ or regular brackets $\x(i)$.
\item
Matrices are denoted by boldface upper-case letters such as $\mathbf{X}  \in \reals^{m \times n}$.  Their coordinates by $\mathbf{X}(i,j)$, or $\mathbb{X}_{ij}$.  
\item
Functions are denoted by lower case letters $f: \reals^d \mapsto \reals$. 

\item 
The $k$-th differential of function $f$ is denoted by $\nabla^k f \in \reals^{d^k}$.  The gradient is denoted without the superscript, as $\nabla f$. 

\item
We use the mathcal macro for sets, such as $\K \subseteq \reals^d$.  

\item
We denote the gradient at point $\x_t$ as $\nabla_{\x_t}$, or simply $\nabla_t$.

\item
We denote the global or local optima of functions by $\x^\star$. 

\item
We denote distance to optimality for iterative algorithms by $h_t = f(\x_t) - f(\x^\star)$. 

\item
Euclidean distance to optimality is denoted $d_t = \|\x_t - \x^\star\| $.

\end{itemize}



\chapter{Introduction}  \label{chap:intro}

The topic of this lecture series is the mathematical optimization approach to machine learning. 

In standard algorithmic theory, the burden of designing an efficient algorithm for solving a problem at hand is on the algorithm designer. In the decades since in the introduction of computer science, elegant algorithms have been designed for tasks ranging from finding the shortest path in a graph, computing the optimal flow in a network, compressing a computer file containing an image captured by digital camera, and replacing a string in a text document. 

The design approach, while useful to many tasks, falls short of more complicated problems, such as identifying a particular person in an image in bitmap format, or translating text from English to Hebrew.  There may very well be an elegant algorithm for the above tasks, but the algorithmic design scheme does not scale. 

As Turing promotes in his paper \cite{turing}, it is potentially easier to teach a computer to learn how to solve a task, rather than teaching it the solution for the particular tasks. In effect, that's what we do at school, or in this lecture series...

The machine learning approach to solving problems is to have an automated mechanism for learning an algorithm. Consider the problem of classifying images into two categories: those containing cars and those containing chairs (assuming there are only two types of images in the world). In ML we train (teach) a machine to achieve the desired functionality. The same machine can potentially solve any algorithmic task, and differs from task to task only by a set of parameters that determine the functionality of the machine. This is much like the wires in a computer chip determine its functionality. Indeed, one of the most popular machines are artificial neural networks. 

The mathematical optimization approach to machine learning is to view the process of machine training as an optimization problem. If we let $w \in \reals^d$ be the parameters of our machine (a.k.a. model), that are constrained to be in some set $\K \subseteq \reals^d$, and $f$ the function measuring success in mapping examples to their correct label, then the problem we are interested in is described by the mathematical optimization problem of 

\begin{equation} \label{eqn:MP}
 \boxed{ \min_{w \in \K  } f (w)  }
\end{equation}

This is the problem that the lecture series focuses on, with particular emphasis on functions that arise in machine learning and have special structure that allows for efficient algorithms.

\section{Examples of optimization problems in machine learning}

\subsection{Empirical Risk Minimization}

Machine learning problems exhibit special structure. For example, one of the most basic optimization problems in supervised learning is that of fitting a model to data, or examples, also known as the optimization problem  of Empirical Risk Minimization (ERM). The special structure of the problems arising in such formulations is separability across different examples into individual losses.  

An example of such formulation is the supervised learning paradigm of linear classification.  In this model, the learner is presented with  positive and negative examples of a concept. Each example, denoted by $\ba_i$, is represented in Euclidean space
by a $d$ dimensional feature vector. For example, a common representation for emails in the spam-classification problem  are binary vectors in Euclidean space, where the dimension of the space is the number of words in the language. The $i$'th email is a vector $\ba_i$ whose entries are given as ones for coordinates corresponding to words that appear in the email, and zero otherwise\footnote{Such a representation may seem na\"ive at first as it completely ignores the words' order of appearance and their context.  Extensions to capture these features are indeed studied in the Natural Language Processing literature.}. In addition, each example has a label  $b_i \in \{-1,+1\}$, corresponding to whether the email has been labeled spam/not spam. The goal is to find a hyperplane separating the two classes of vectors: those with positive labels and those with negative labels. If such a hyperplane, which completely  separates the training set according to the labels, does not exist, then the goal is to find a 
hyperplane that achieves a separation of the training set with the smallest number of mistakes. 

Mathematically speaking, given a set of $m$ examples to train on, we seek $\x \in \reals^d$ that minimizes the number of incorrectly classified examples, i.e.
\begin{equation} \label{eqn:linear-classification}
\min_{\x \in \reals^d}  \frac{1}{m} \sum_{i \in [m]}  \delta( \sign(\x^\top \ba_i ) \neq b_i)  
\end{equation}
where $\sign(x) \in \{-1,+1\}$ is the sign function, and $\delta(z) \in \{0,1\}$ is the indicator function that takes the value $1$ if the condition $z$ is satisfied and zero otherwise.

The mathematical formulation of the linear classification above is a special case of mathematical programming \eqref{eqn:MP}, in which 
$$f(\x) = \frac{1}{m} \sum_{i \in [m]}  \delta( \sign(\x^\top \ba_i ) \neq b_i)  = \E_{i \sim [m]} [ \ell_i(\x)]  ,$$
where we make use of the expectation operator  for simplicity, and denote $\ell_i(\x)  = \delta( \sign(\x^\top \ba_i ) \neq b_i)$ for brevity. Since the program above is non-convex and non-smooth, it is common to take  a convex relaxation and replace $\ell_i$ with convex loss functions. Typical choices include the means square error function and the hinge loss, given by
$$ \ell_{\ba_i, b_i } (\x) = \max\{ 0, 1 - b_i \cdot \x^\top \ba_i \} . $$

This latter loss function in the context of binary classification gives rise to the popular soft-margin SVM problem.


Another important optimization problem is that of training a deep neural network for binary classification.
For example, consider a dataset of images, represented in bitmap format and denoted by $\{ \ba_i \in \reals^{d}  | i \in [m]\}$, i.e. $m$ images over $n$ pixels. We would like to find a mapping from images to the two categories, $\{b_i \in \{0,1\} \}$ of cars and chairs. The mapping is given by a set of parameters of a machine class, such as weights in a neural network, or values of a support vector machine. We thus try to find the optimal parameters that match $\ba_i$ to $b$, i..e 
$$ \min_{\w \in \reals^d} f(\w) =  \E_{ \ba_i , b_i } \left[ \ell( f_\w (\ba_i) , b_ i ) \right]  . $$

\subsection{Matrix completion and recommender systems}

Media recommendations have changed significantly with the advent of the Internet and rise of online media stores. The large amounts of data collected allow for efficient clustering and accurate prediction of users' preferences for a variety of media. A well-known example is the so called ``Netflix challenge''---a competition of automated tools for recommendation from a large dataset of users' motion picture preferences.

One of the most successful approaches for automated recommendation systems, as proven in the Netflix competition, is matrix completion. Perhaps the simplest version of the problem can be described as follows.  

The entire dataset of user-media preference pairs is thought of as a partially-observed matrix. Thus, every person is represented by a row in the matrix, and every column represents a media item (movie). For simplicity, let us think of the observations as binary---a person either likes or dislikes a particular movie. Thus, we have a matrix $M \in \{0,1,*\}^{n \times m}$  where $n$ is the number of persons considered, $m$ is the number of movies at our library, and $0/1$ and $*$ signify ``dislike'', ``like'' and ``unknown'' respectively:
$$ M_{ij} = \mythreecases {0}{\mbox{person $i$ dislikes movie $j$}}{1}{\mbox{person $i$ likes movie $j$}}{*}{\mbox{preference unknown}} .$$ 

The natural goal is to complete the matrix, i.e. correctly assign $0$ or $1$ to the unknown entries. As defined so far, the problem is ill-posed, since any completion would be equally good (or bad), and no restrictions have been placed on the completions.  

The common restriction on completions is that the ``true'' matrix has low rank. Recall that if a matrix $X \in \reals^{n \times m}$ has rank $k \leq \rho = \min \{n,m\} $ then it can be written as 
$$ X = U V \ , \ U \in \reals^{n \times k} , V \in \reals^{k \times m}.  $$

The intuitive interpretation of this property is that each entry in $M$ can be explained by only $k$ numbers. In matrix completion this means, intuitively, that there are only $k$ factors that determine a persons preference over movies, such as genre, director, actors and so on. 

Now the simplistic matrix completion problem can be well-formulated as in the following mathematical program. Denote by $\| \cdot \|_{OB}$ the Euclidean norm only on the observed (non starred) entries of $M$, i.e., 
$$\|X\|_{OB}^2 = \sum_{M_{ij} \neq *} X_{ij}^2.$$ 
The mathematical program for matrix completion is given by
\begin{align*}
& \min_{X \in \reals^{n \times m} } \frac{1}{2} \| X - M \|_{OB}^2 \\
& \text{s.t.} \quad \rank(X) \leq k. 
\end{align*}

\subsection{Learning in Linear Dynamical Systems }

Many learning problems require memory, or the notion of state. This is captured by the paradigm of reinforcement learning, as well of the special case of control in Linear Dynamical Systems (LDS).  

LDS model a variety of control and robotics problems in continuous variables. The setting is that of a time series, with following parameters:
\begin{enumerate}
\item
Inputs to the system, also called controls, denoted by $\uv_1,...,\uv_T \in \reals^n$. 
\item
Outputs from the system, also called observations, denoted $\y_1,...,\y_T \in \reals^m$.
\item
The state of the system, which may either be observed or hidden, denoted $\x_t,...,\x_T \in \reals^d$.
\item
The system parameters, which are transformations matrices $\bA,\bB,\bC,\bD$ in appropriate dimensions. 
\end{enumerate}

In the online learning problem of LDS, the learner iteratively observes $\uv_t,\y_t$, and has to predict $\hat{\y}_{t+1}$.  The actual $\y_t$ is generated according to the following dynamical equations:
\begin{align*}
& \x_{t+1} = \bA \x_t + \bB \uv_t + \epsilon_t \\
& \y_{t+1} = \bC \x_{t+1} + \bD \uv_t + \zeta_t  , 
\end{align*}
where $\epsilon_t,\zeta_t$ are noise which is distributed as a Normal random variable. 

Consider an online sequence in which the states are visible. At time $t$, all system states, inputs and outputs are visible up to this time step. The learner has to predict $\y_{t+1}$, and only afterwards observes $\uv_{t+1}.\x_{t+1},\y_{t+1}$. 

One reasonable way to predict $\y_{t+1}$ based upon past observations is to compute the system, and use the computed transformations to predict. This amounts to solving the following mathematical program: 
$$ \min_{\bA,\bB,\hat{\bC},\hat{\bD}} \left\{  \sum_{\tau < t }  ( \x_{\tau+1} - \bA \x_\tau + \bB \uv_\tau )^2 + (\y_{\tau+1} - \hat{\bC} \x_\tau + \hat{\bD} \uv_\tau )^2 \right \} , $$
and then predicting  $\hat{\y}_{t+1} = \hat{\bC} \hat{\bA} (\x_t + \bB \uv_t)  + \hat{\bD} \uv_t$.

\section{Why is mathematical programming hard?} 

The general formulation \eqref{eqn:MP} is NP hard. To be more precise, we have to define the computational model we are working in as well as and the access model to the function. 

Before we give a formal proof, the intuition to what makes mathematical optimization hard is simple to state. In one line: it is the fact that global optimality cannot be verified on the basis of local properties. 

Most, if not all, efficient optimization algorithms are iterative and based on a local improvement step. By this nature, any optimization algorithm will terminate when the local improvement is no longer possible, giving rise to a proposed solution. However, the quality of this proposed solution may differ significantly, in general, from that of the global optimum. 

This intuition explains the need for a property of objectives for which global optimality is locally verifiable. Indeed, this is exactly the notion of {\bf convexity}, and the reasoning above explains its utmost importance in mathematical optimization. 

We now to prove that mathematical programming is NP-hard. This requires discussion of the computational model as well as access model to the input. 

\subsection{The computational model}

The computational model we shall adopt throughout this manuscript is that of a RAM machine equipped with oracle access to the objective function $f: \reals^d \mapsto \reals$ and constraints set $\K \subseteq \reals^d$. The oracle model for the objective function can be one of the following, depending on the specific scenario:
\begin{enumerate}
\item
{\bf Value oracle: } given a point $x \in \reals^d$, oracle returns $f(x) \in \reals$. 
\item
{\bf Gradient (first-order) oracle: } given a point $x \in \reals^d$, oracle returns the gradient $\nabla f(x) \in \reals^d$. 
\item
{\bf $k$-th order differential oracle: } given a point $x \in \reals^d$, oracle returns the tensor $\nabla^k f(x) \in \reals^{d^k} $. 
\end{enumerate}

The oracle model for the constraints set is a bit more subtle. We distinguish between the following oracles:
\begin{enumerate}
\item
{\bf Membership oracle: } given a point $x \in \reals^d$, oracle returns one if $x \in \K$ and zero otherwise. 
\item
{\bf Separating hyperplane oracle: } given a point $x \in \reals^d$, oracle either returns "Yes" if $x \in \K$, or otherwise returns a hyperplane $h \in \reals^d$ such that $ h^\top x > 0 $ and $ \forall y \in \K \ , \  h^\top y \leq 0$. 
\item
{\bf Explicit sets: } the most common scenario in machine learning is one in which $\K$ is ``natural", such as the Euclidean ball or hypercube, or the entire Euclidean space. 
\end{enumerate}

\subsection{Hardness of constrained mathematical programming}

Under this computational model, we can show:
\begin{lemma}
Mathematical programming is NP-hard, even for a convex continuous constraint set $\K$ and quadratic objective functions. 
\end{lemma}
\begin{proof}[Informal sketch]
Consider the MAX-CUT problem: given a graph $G = (V,E)$, find a subset of the vertices that maximizes the number of edges cut. Let $A$ be the negative adjacency matrix of the graph, i.e. 
$$ A_{ij} = \mycases {-1} {(i,j) \in E}{0}{o/w} $$
Also suppose that $A_{ii} = 0$. 

Next, consider the mathematical program: 
\begin{align}\label{eqn:first-max-cut}
\min  \left\{ f_A(\x) = \frac{1}{4} ( \x^\top A \x  - 2 |E| ) \right\} \\
\|\x\|_\infty = 1  \ .  \notag
\end{align}
Consider the cut defined by the solution of this program, namely 
$$ S_{\x}  = \{ i \in V |  \x_i = 1 \}  , $$
for $\x = \x^\star$.  Let $C(S)$ denote the size of the cut specified by the subset of edges $S \subseteq E$. 
Observe that the expression $\frac{1}{2} \x ^\top A \x$, is exactly equal to the number of edges that are cut by $S_\x$ minus the number of edges that are uncut. Thus, we have
$$ \frac{1}{2}  \x A \x =  C(S_\x) -  (E - C(S_\x)) = 2 C(S_\x) - E , $$
and hence $f(\x) = C(S_\x)$. Therefore, maximizing $f(\x)$ is equivalent to the MAX-CUT problem, and is thus NP-hard. We proceed to make the constraint set convex and continuous. Consider the mathematical program
\begin{align} \label{eqn:second-max-cut}
\min  \left \{  f_A(\x)  \right\} \\
\|\x\|_\infty \leq 1 \ .  \notag
\end{align}
This is very similar to the previous program, but we relaxed the equality to be an inequality, consequently the constraint set is now the hypercube. We now claim that the solution is w.l.o.g. a vertex. To see that, consider $\y(\x) \in \{ \pm 1\}^d$ a rounding of $\x$ to the corners defined by:
$$ \y_i =  \y(\x)_i = \mycases { 1 } { w.p. \frac{1 + \x_i}{2} } {-1} {w.p. \frac{1- \x_i}{2} }  $$ 
Notice that 
$$ \E[ \y] = \x \ , \  \forall i \neq j \  . \ \E[ \y_i \y_j] = \x_i \x_j , $$
and therefore $\E[ \y(\x) ^\top A \y(\x) ] = \x^\top A \x $. We conclude that the optimum of mathematical program \ref{eqn:second-max-cut} is the same as that for \ref{eqn:first-max-cut}, and both are NP-hard. 
\end{proof}

\chapter{Basic concepts in optimization and analysis}  \label{chap:opt}
\chaptermark{Basic concepts}

\section{Basic definitions and the notion of convexity}  \label{sec:optdefs}
\sectionmark{Basics}

We consider minimization of a continuous function over a convex subset of Euclidean space. We mostly consider objective functions that are convex. In later chapters we relax this requirement and consider non-convex functions as well. 

Henceforth,  let $\K \subseteq \reals^d$ be a bounded convex and compact set in Euclidean space. We denote by $D$ an upper bound on the diameter of $\K$:
$$ \forall \x,\y \in \K , \ \|\x-\y\| \leq D.$$
A set $\K$ is convex if for any  $\x,\y \in \K$, all the points on the line segment connecting $\x$ and $\y$ also belong to $\K$, i.e., 
$$ \forall \alpha \in [0,1]  , \ \alpha \x + (1-\alpha)\y \in \K.$$
A function $f: \K \mapsto \reals$ is convex if for  any $\x,\y \in \K$  
$$\forall \alpha \in [0,1] , \  f(  \alpha \x + (1-\alpha) \y) \leq  \alpha f(\x) + (1-\alpha) f(\y).$$

\paragraph{Gradients and subgradients.}
The set of all subgradients of a function $f$ at $\x$, denoted $\partial f(\x)$, is the set of all vectors $\uv$ such that
$$ f(\y) \geq  f(\x) + \uv^\top (\y-\x) . $$  
It can be shown that the set of subgradients of a convex function is always non-empty. 

Suppose $f$ is differentiable, let $\nabla f(\x) [i] = \frac{\partial} {\partial \x_i} f(\x)$ be the vector of partial derivatives according to the variables, called the gradient. If the gradient $\nabla f(\x)$ exists, then $\nabla f(\x) \in \partial f(\x)$ and  $\forall \y \in \K$
$$  f(\y) \geq f(\x) + \nabla f(\x)^\top (\y-\x).$$
Henceforth we shall denote by $\nabla f(\x)$ the gradient, if it exists, or any member of $\partial f(\x)$ otherwise. 

We denote by $G > 0$ an upper bound on the norm of the subgradients of $f$ over $\K$, i.e., $\|\nabla f(\x)\| \leq G$ for all $\x \in \K$. The existence of Such an upper bound implies that the function  $f$ is Lipschitz continuous with parameter $G$, that is, for all $\x,\y \in \K$
$$ |f(\x) - f(\y)| \leq G \|\x-\y\|.$$

\paragraph{Smoothness and strong convexity.}

The optimization and machine learning literature studies special types of convex functions that admit useful properties, which in turn allow for more efficient optimization. Notably, we say that a function is $\alpha$-strongly convex if
$$  f(  \y) \geq  f(\x) + \nabla f(\x)^\top (\y-\x)  + \frac{\alpha}{2} \|\y-\x\|^2.  $$
A function is $\beta$-smooth if
$$  f(  \y) \leq  f(\x) + \nabla f(\x)^\top (\y-\x)  + \frac{\beta}{2} \|\y-\x\|^2.  $$
The latter condition is implied by a slightly stronger  Lipschitz condition over the gradients, which is sometimes used to defined smoothness, i.e., 
$$ \| \nabla f(\x) - \nabla f(\y) \| \leq {\beta} \|\x-\y\|.$$

If the function is twice differentiable and admits a second derivative, known as a Hessian for a function of several variables, the above conditions are equivalent to the following condition on the Hessian, denoted $\nabla^2 f(\x)$:
\begin{align*}
\text{Smoothness: }   \ \  -\beta  I \preccurlyeq &  \nabla^2 f(\x) \preccurlyeq \beta I \\
\text{Strong-convexity: }   \ \  \alpha  I \preccurlyeq  & \nabla^2 f(\x) ,
\end{align*}
where $A\preccurlyeq B$ if the matrix $B-A$ is positive semidefinite.

When the function $f$ is both $\alpha$-strongly convex and $\beta$-smooth, we say that it is $\gamma$-well-conditioned where $\gamma$ is the ratio between  strong convexity and smoothness, also called the {\it condition number} of $f$
$$ \gamma = \frac{\alpha}{\beta} \leq 1$$

\subsection{Projections onto convex sets}  \label{sec:projections}
In the following algorithms we shall make use of a projection operation onto a convex set, which is defined as the closest point inside the convex set to a given point. Formally,
$$ \proj_\K (\y) \equaltri \argmin_{\x \in \K} \| \x - \y \|.$$
When clear from the context, we shall remove the $\K$ subscript. It is left as an exercise to the reader to prove that the projection of a given point over a closed non-empty convex set exists and is unique. 

The computational complexity of projections is a subtle issue that depends much on the characterization of $\K$ itself. Most generally, $\K$ can be represented by a membership oracle---an efficient procedure that is capable of deciding whether a given $\x$ belongs to $\K$ or not. In this case, projections can be computed in polynomial time. In certain special cases, projections can be computed very efficiently in near-linear time. 

A crucial property of projections that we shall make extensive use of  is the Pythagorean theorem, which we state here for completeness:
\begin{figure}[h!]
\begin{center}
\includegraphics[width=3.5in]{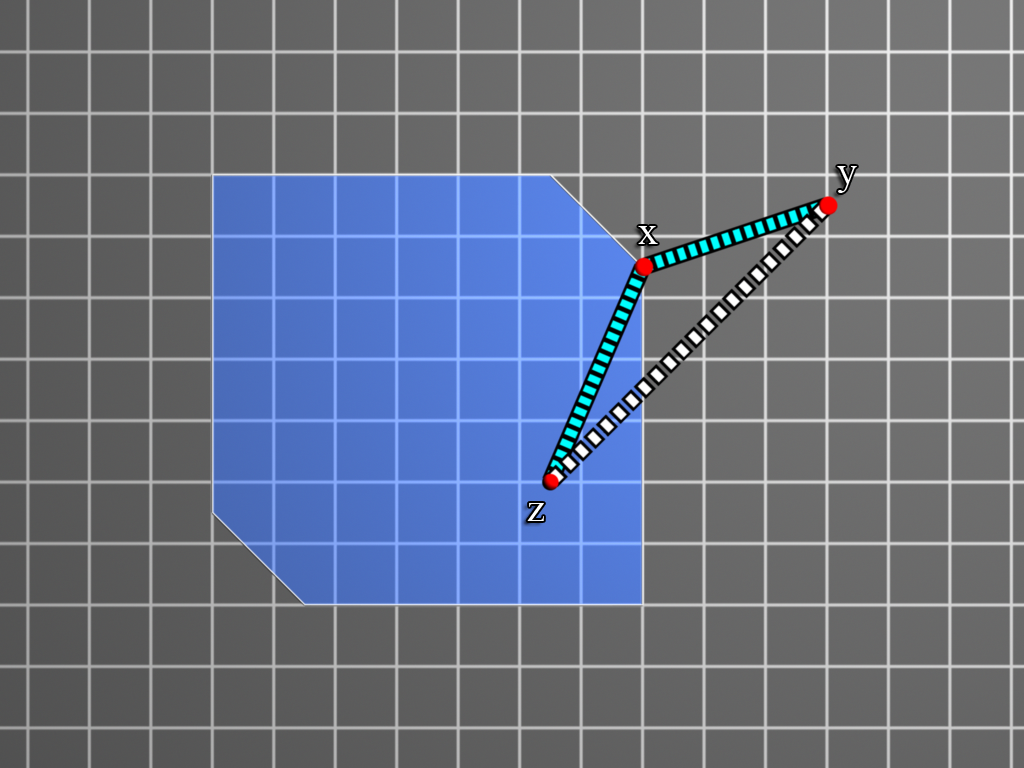}
\end{center}
\caption{Pythagorean theorem.}
\end{figure}
\begin{theorem}[Pythagoras, 	circa 500 BC] \label{thm:pythagoras}
	Let $\K \subseteq \reals^d$ be a convex set, $\y \in \reals^d$ and $\x = \proj_\K(\y)$. Then for any $\z \in \K$ we have
	$$ \| \y - \z \| \geq \| \x - \z \|.$$
\end{theorem}

We note that there exists a more general version of the Pythagorean theorem. The above theorem and the definition of projections are true and valid not only for Euclidean norms, but for projections according to other distances that are not norms. In particular, an analogue of the Pythagorean theorem remains valid with respect to Bregman divergences. 

\subsection{Introduction to optimality conditions} 

The standard curriculum of high school mathematics contains the basic facts concerning when a function (usually in one dimension) attains a local optimum or saddle point. The KKT (Karush-Kuhn-Tucker) conditions generalize these facts to more than one dimension, and the reader is referred to the bibliographic material at the end of this chapter for an in-depth rigorous discussion of optimality conditions in general mathematical programming. 

For our purposes, we describe only briefly and intuitively the main facts that we will require henceforth. We separate the discussion into convex and non-convex programming. 

\subsubsection{Optimality for convex optimization} 

A local minimum of a convex function is also a global minimum  (see exercises at the end of this chapter). 
We say that $\x^\star$ is an $\epsilon$-approximate optimum if the following holds:
$$ \forall \x \in \K \ . \  f(\x^\star) \leq f(\x) + \epsilon .$$ 

The generalization of the  fact that a minimum of a convex differentiable function on $\reals$ is a point in which its derivative is equal to zero, is given by the multi-dimensional analogue that its gradient is zero:
$$ \nabla f(\x ) =  0  \ \ \Longleftrightarrow  \ \ \x \in \argmin_{\x \in \reals^n} f(\x).$$

We will require a slightly more general, but equally intuitive, fact for constrained optimization: at a minimum point of a constrained convex function, the inner product between the negative gradient and direction towards the interior of $\K$ is non-positive. This is depicted in Figure \ref{fig:optimality}, which shows that $-\nabla f(\x^\star)$ defines a supporting hyperplane to $\K$. The intuition is that if the inner product were positive, one could improve the objective by moving in the direction of the projected negative gradient. This fact is stated formally in the following theorem.
\begin{theorem}[Karush-Kuhn-Tucker]  \label{thm:optim-conditions}
Let $\K \subseteq \reals^d$ be a convex set, $\x^\star \in \argmin_{\x \in  \K} f(\x)$.  Then for any $\y \in \K$ we have
$$ \nabla f(\x^\star) ^\top ( \y - \x^\star ) \geq 0.  $$
\end{theorem}
\begin{figure}[h!]
\begin{center}
\includegraphics[width=4in]{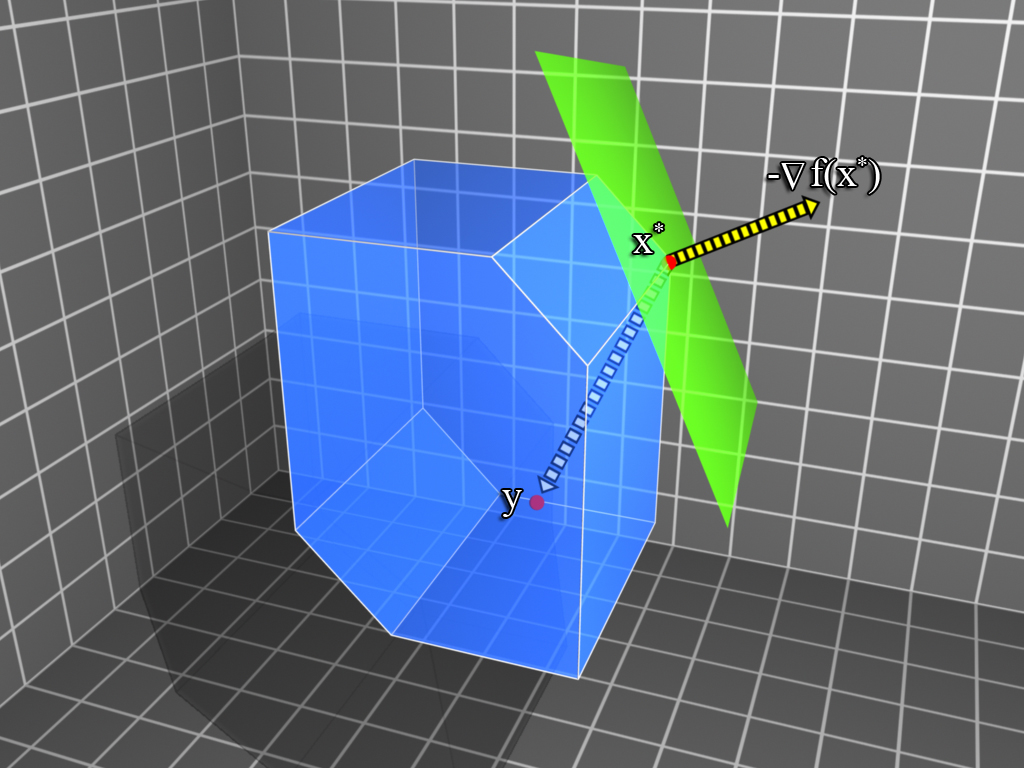}
\end{center}
\caption{Optimality conditions: negative (sub)gradient pointing outwards. \label{fig:optimality}}
\end{figure}

\subsection{Solution concepts for non-convex optimization}

We have seen in the previous chapter that mathematical optimization is NP-hard. This implies that finding global solutions for non-convex optimization is NP-hard, even for smooth functions over very simple convex domains. We thus consider other trackable concepts of solutions. 

The most common solution concept is that of first-order optimality, a.k.a. saddle-points or stationary points. These are points that satisfy
$$  \| \nabla f(\x^\star ) \| = 0 . $$
Unfortunately, even finding such stationary points is NP-hard.  We thus settle for approximate stationary points, which satisify
$$  \| \nabla f(\x^\star ) \| \leq \epsilon . $$

\begin{figure}[h!]
\begin{center}
\includegraphics[width=3in]{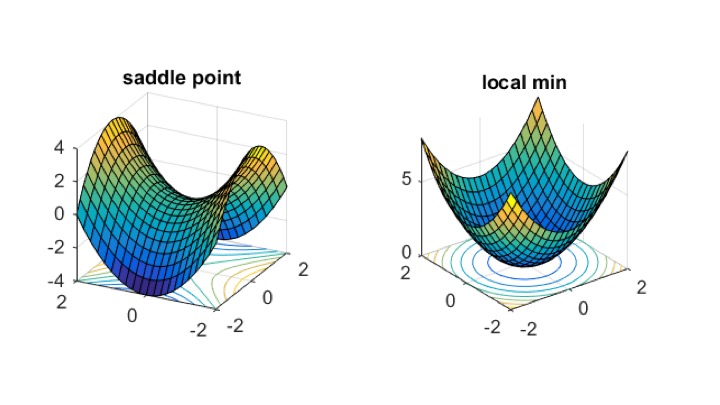}
\end{center}
\caption{First and second-order local optima. \label{fig:optimums}}
\end{figure}

A more stringent notion of optimality we may consider is obtained by looking at the second derivatives. We can require they behave as for global minimum, see figure \ref{fig:optimums}. Formally, we say that a point $\x^\star$ is a second-order local minimum if it satisfies the two conditions:
$$  \| \nabla f(\x^\star ) \| \leq \epsilon  \ , \   \nabla^2 f(\x^\star) \succeq -\sqrt{\eps} I . $$
The differences in approximation criteria for first and second derivatives is natural, as we shall explore in non-convex approximation algorithms henceforth.

We note that it is possible to further define optimality conditions for higher order derivatives, although this is less useful in the context of machine learning.

\section{Potentials for distance to optimality}

When analyzing convergence of gradient methods, it is useful to use potential functions in lieu of function distance to optimality, such as gradient norm and/or Euclidean distance to optimality. The following relationships hold between these quantities. 

\begin{lemma} \label{lem:elementary_properties}
The following properties hold for $\alpha$-strongly-convex functions and/or $\beta$-smooth functions over Euclidean space $\reals^d$. 
\begin{enumerate}
    \item $\frac{\alpha}{2} d_t^2 \leq h_t$
    \item $ h_t \leq \frac{\beta}{2} d_t^2$
    \item $\frac{1}{2 \beta} \|\nabla_t\|^2 \leq h_t$
    \item $ h_t \leq \frac{1}{2 \alpha} \|\nabla_t\|^2 $
\end{enumerate}
\end{lemma}

\begin{proof}
\begin{enumerate}
    \item  $h_t \geq \frac{\alpha}{2} d_t^2 $: 

    By strong convexity, we have 
\begin{eqnarray*}
h_t & =  f(\x_t) - f(\x^{\star}) \\
& \geq  \nabla f(\x^{\star})^\top (\x_t - \x^{\star}) + \frac{\alpha}{2} \|\x_t - \x^{\star}\|^2  \\
 & = \frac{\alpha}{2} \|\x_t - \x^{\star}\|^2 
\end{eqnarray*}
where the last inequality follows since the gradient at the global optimum is zero. 
    
    \item  $h_t \leq \frac{\beta}{2} d_t^2 $:  
    
By smoothness, 
\begin{eqnarray*}
h_t & =  f(\x_t) - f(\x^{\star}) \\
& \leq  \nabla f(\x^{\star})^\top (\x_t - \x^{\star}) + \frac{\beta}{2} \|\x_t - \x^{\star}\|^2  \\
 & = \frac{\beta}{2} \|\x_t - \x^{\star}\|^2 
\end{eqnarray*}
where the last inequality follows since the gradient at the global optimum is zero.

\item  $h_t \geq \frac{1}{2\beta} \|\nabla_t\|^2 $: 
Using smoothness, and let $\x_{t+1} = \x_t - \eta \nabla_t$ for $\eta = \frac{1}{\beta}$, 
\begin{eqnarray*}
h_t =  & f(\x_t) - f(\x^{\star}) \\
& \geq  f(\x_t) - f(\x_{t+1})   \\
 & \geq   \nabla f(\x_t)^\top (\x_{t} - \x_{t+1}) - \frac{\beta}{2} \|\x_t - \x_{t+1} \|^2   \\
 & = \eta \|\nabla_t\|^2  - \frac{\beta}{2} \eta^2 \|\nabla_t\|^2 \\
 & = \frac{1}{2\beta} \|\nabla_t\|^2  .
\end{eqnarray*}

\item   $h_t \leq \frac{1}{2\alpha} \|\nabla_t\|^2 $:  
    
We have for any pair $\x,\y \in \reals^d$:
\begin{align*}
f(\y)  & \ge  f(\x) +   \nabla f(\x)^\top  (\y - \x ) + \frac{\alpha}{2}  \|\x - \y\|^2  \\
&\ge  \min_{\z \in \reals^d } \left\{ f(\x) +   \nabla f(\x)^\top  (\z - \x ) + \frac{\alpha}{2}  \|\x - \z\|^2 \right\} \\
& =  f(\x) - \frac{1}{2  \alpha} \| \nabla f(\x)\|^2. \\
& \text{ by taking $\z = \x - \frac{1}{ \alpha} \nabla f(\x) $ }
\end{align*}
In particular, taking $\x = \x_t \ , \ \y = \x^\star$, we get
\begin{equation} \label{eqn:gradlowerbound}
 h_t =  f(\x_t) - f(\x^\star)  \leq \frac{1}{2 \alpha} \|\nabla_t\|^2  .
\end{equation}

\end{enumerate}
\end{proof}

\section{Gradient descent and the Polyak stepsize} 

The simplest iterative optimization algorithm is gradient descent, as given in Algorithm \ref{alg:basic}. We analyze GD with the Polyak stepsize, which has the advantage of not depending on the strong convexity and/or smoothness parameters of the objective function. 
\begin{algorithm}[h!]
\caption{GD with the Polyak stepsize}
\label{alg:basic}
\begin{algorithmic}[1]
\STATE Input: time horizon $T$, $x_0$
\FOR{$t = 0, \ldots, T-1$}
\STATE Set $\eta_t =  \frac{h_t}{\|\nabla_t\|^2} $
\STATE  $ \x_{t+1}  =   \x_t - \eta_t \nabla_t $
\ENDFOR
\STATE Return $\xbar = \argmin_{\x_t} \{ f(\x_t) \}$
\end{algorithmic}
\end{algorithm}

To prove convergence bounds, assume $\|\nabla_t\| \leq G$, and define:
\begin{eqnarray*}
  B_T    &=&  \min\left\{
  \frac{G d_0}{\sqrt{ T}},
  \frac {2 \beta d_0^2}{ T },
  \frac{3 G^2}{  \alpha  T } ,
   \beta d_0^2\left(1-\frac{\alpha}{4\beta}\right)^T
 \right\}
\end{eqnarray*}

\begin{theorem} \label{thm:simple}
(GD with the Polyak Step Size) 
Algorithm \ref{alg:basic} attains the following regret bound after $T$ steps: 
\begin{eqnarray*}
h(\xbar)  &=& \min_{ 0 \leq t \leq  T} \{ h_t \} \leq B_T 
\end{eqnarray*}
\end{theorem}

Theorem~\ref{thm:simple} directly follows from the following lemma. Let $0\leq \gamma \leq 1 $,   define $R_{T,\gamma}$ as follows:
\[
  R_{T,\gamma} = \min\left\{
  \frac{G d_0}{\sqrt{\gamma T}},
  \frac {2 \beta d_0^2}{\gamma T },
  \frac{ 3 G^2}{{\gamma}  \alpha  T } ,
   \beta d_0^2\left(1-\gamma\frac{\alpha}{4 \beta}\right)^T
 \right\}
\, .
\]

\begin{lemma} \label{lemma:shalom2}
For $0\leq \gamma \leq 1 $, suppose that a sequence $\x_0,
\ldots \x_t$ satisfies:
\begin{equation} \label{eqn:shalom3}
d_{t+1}^2 \leq d_t^2 -  \gamma \frac{ h_t^2}{\|\nabla_t\|^2}
\end{equation}
then for $\xbar$ as defined in the algorithm, 
we have:
\[
h(\xbar)  \leq   R_{T,\gamma}\, .
\]
\end{lemma}

\begin{proof}
The proof analyzes different cases:
\begin{enumerate}
\item
For convex functions with gradient bounded by $G$, 
\begin{eqnarray*}
d_{t+1}^2 -  d_t^2 & \leq - \frac{\gamma h_t^2}{\|\nabla_t\|^2} \leq -
                     \frac{\gamma h_t^2}{G^2}  
\end{eqnarray*}
Summing up over $T$ iterations, and using Cauchy-Schwartz, we have
\begin{eqnarray*}
\frac{1}{T} \sum_t h_t 
& \leq&  \frac{1}{\sqrt{T}} \sqrt{\sum_t h_t^2} \\
& \leq& \frac{ G}{\sqrt{\gamma T}} \sqrt{\sum_t (d_{t}^2 - d_{t+1}^2)} \leq
\frac{ G d_0 }{\sqrt{\gamma T}} \, .
\end{eqnarray*}

\item
For smooth functions whose gradient is bounded by $G$,  Lemma~\ref{lem:elementary_properties} implies:
\[
d_{t+1}^2 - d_t^2 \leq - \frac{\gamma h_t^2}{\|\nabla_t\|^2} \leq -
\frac{\gamma h_t}{2 \beta} \, .
\]
This implies
\[
\frac{1}{T} \sum_t h_t  \leq \frac{2 \beta d_0^2}{\gamma T}\, .
\]

\item
For strongly convex functions, Lemma~\ref{lem:elementary_properties} implies:
\[ d_{t+1}^2 - d_t^2
  \leq - \gamma \frac{h_t^2}{\|\nabla_t\|^2}
  \leq - \gamma \frac{h_t^2}{G^2}
  \leq  - \gamma  \frac{\alpha^2 d_t^4 }{4 G^2} \, .
\]
In other words,
$d_{t+1}^2  \leq  d_t^2 ( 1- \gamma \frac{\alpha^2 d_t^2}{4 G^2} ) \, .$ 
Defining $a_t :={\gamma}\frac{\alpha^2 d_t^2}{4 G^2}$, we have:
\[
a_{t+1}  \leq  a_t (1-a_t) \, .
\]
This implies that $a_t \leq \frac{1}{t+1}$, which can be seen by
induction\footnote{That $a_0\leq 1$ follows from Lemma
  \ref{lem:elementary_properties}. For $t=1$, $a_1\leq \frac{1}{2}$ since
  $a_1  \leq  a_0 (1-a_0)$ and $0\leq a_0\leq 1$.
  For the induction step, $
  a_t  \leq  a_{t-1} (1-a_{t-1}) \leq
  \frac{1}{t}(1-\frac{1}{t})
  =\frac{t-1}{t^2}=\frac{1}{t+1}(\frac{t^2-1}{t^2})
  \leq \frac{1}{t+1}$.}. The proof is completed as follows\footnote{This assumes $T$ is even. $T$ odd
    leads to the same constants.} :  
\begin{eqnarray*}
\frac{1}{ T/2 } \sum_{t= T/2 }^T h_t^2 &
\leq& \frac{2G^2}{\gamma  T  }\sum_{t= T/2 }^T ( d_t^2 -
                                   d_{t+1}^2)  \\
  &=&\frac{2 G^2}{\gamma  T } ( d _{ T/2 }^2 - d_T^2)  \\
  &=&\frac{8 G^4}{ \gamma^2 \alpha^2   T} ( a
    _{ T/2 } - a_T)  \\ 
   & \leq &\frac{9 G^4}{\gamma^2 \alpha^2 T ^2} 
  \, .
\end{eqnarray*}
Thus, there exists a $t$ for which $h_t^2 \leq \frac{ 9 G^4}{\gamma^2 \alpha^2
   T^2} $. Taking the square root completes the claim.

\item
For both strongly convex and smooth functions: 
\[ d_{t+1}^2 - d_t^2 \leq - \gamma \frac{h_t^2}{\|\nabla_t\|^2} \leq
 - \frac{\gamma h_t}{2 \beta} \leq
  - \gamma \frac{\alpha}{4\beta} d_t^2
  \]
  Thus,
  \[
    h_{T} \leq \beta d_{T}^2 \leq \beta d_0^2
    \left(1-\gamma\frac{\alpha}{4 \beta}\right)^T\, .
    \]
  \end{enumerate}
This completes the proof of all cases.
\end{proof}

\newpage

\section{Exercises}

\begin{enumerate}

\item
Write an explicit expression for the gradient and projection operation (if needed) for each of the example optimization problems in the first chapter. 

\item
Prove that a differentiable function $f(x) : \mathbb{R} \rightarrow \mathbb{R}$ is convex if and only if for any $x,y\in\mathbb{R}$ it holds that $f(x)-f(y) \leq (x-y)f'(x)$.

\item
Recall that we say that a function $f:\mathbb{R}^n\rightarrow\mathbb{R}$ has a condition number $\gamma = \alpha/\beta$ over $
K \subseteq \reals^d$ if the following two inequalities hold for all $\x,\y \in \K$:
\begin{enumerate}
\item $  f(\y) \geq f(\x) + (\y-\x)^{\top}\nabla{}f(\x) + \frac{\alpha}{2}\Vert{\x-\y}\Vert^2$
\item $  f(\y) \leq f(\x) + (\y-\x)^{\top}\nabla{}f(\x) + \frac{\beta}{2}\Vert{\x-\y}\Vert^2$
\end{enumerate}
For matrices $A,B \in \reals^{n \times n}$ we denote $A \succcurlyeq B$ if $A-B$ is positive semidefinite. 
Prove that if $f$ is twice differentiable and it holds that $\beta\textbf{I} \succcurlyeq \nabla^2f(\x) \succcurlyeq \alpha\textbf{I}$ for any $\x\in \K$, then the condition number of $f$ over $\K$  is $\alpha/\beta$.

\item
Prove: 
\begin{enumerate}
	\item
	The sum of convex functions is convex. 
	\item
	Let $f$ be $\alpha_1$-strongly convex and $g$ be $\alpha_2$-strongly convex. Then $f+g$ is $(\alpha_1+\alpha_2)$-strongly convex. 
	\item
	Let $f$ be $\beta_1$-smooth and $g$ be $\beta_2$-smooth. Then $f+g$ is $(\beta_1+\beta_2)$-smooth. 
	
\end{enumerate}

\item
Let $\K \subseteq \reals^d$ be closed, compact, non-empty and bounded. Prove that a necessary and sufficient condition for  $\proj_K(\x)$ to be a singleton, that is for $|\proj_K(\x)| = 1$, is for  $K$ to be  convex.

\item
Prove that for convex functions, $\nabla f(\x) \in \partial f(\x)$, that is, the gradient belongs to the subgradient set. 

\item
 Let $f(\x):\mathbb{R}^n\rightarrow\mathbb{R}$ be a convex differentiable function and $\K\subseteq \mathbb{R}^n$ be a convex set. Prove that $\x^\star\in \K$ is a minimizer of $f$ over $\K$ if and only if for any $\y\in \K$ it holds that $(\y-\x^\star)^{\top}\nabla f(\x^\star) \geq 0$.

\item
Consider the $n$-dimensional simplex 
$$\Delta_n = \lbrace{\x\in\mathbb{R}^n \, | \, \sum_{i=1}^n \x_i = 1, \, \x_i \geq 0 \ ,  \ \forall{i\in[n]}}\rbrace .$$
Give an algorithm for computing the projection of a point $\x\in\mathbb{R}^n$ onto the set $\Delta_n$ (a near-linear time algorithm exists).

\end{enumerate}

\newpage
\section{Bibliographic remarks}

The reader is referred to dedicated books on convex optimization for much more in-depth treatment of the topics surveyed in this background chapter.  For background in convex analysis see the texts  \cite{borwein2006convex,rockafellar1997convex}. The classic textbook \cite{boyd.convex} gives a broad introduction to convex optimization with numerous applications. 
For an adaptive analysis of gradient descent with the Polyak stepsize see \cite{hazan2019revisiting}.

\chapter{Stochastic Gradient Descent} \label{chapter:SGD}

The most important optimization algorithm in the context of machine learning is stochastic gradient descent (SGD), especially for non-convex optimization and in the context of deep neural networks.  In this chapter we spell out the algorithm and analyze it up to tight finite-time convergence rates.

\section{Training feedforward neural networks }

Perhaps the most common optimization problem in machine learning is that of training feedforward neural networks. In this problem, we are given a set of labelled data points, such as labelled images or text.  
Let $\{\x_i, y_i\}$ be the set of labelled data points, also called the training data. 

The goal is to fit the weights of an artificial neural network in order to minimize the loss over the data. Mathematically, the feedforward network is a given weighted a-cyclic graph $G = (V,E,W)$. Each node $v$ is assigned an activation function, which we assume is the same function for all nodes, denoted $\sigma: \reals^d \mapsto \reals$.  
Using a biological analogy, an activation function $\sigma$ is a function that determines how strongly a neuron (i.e. a node) `fires' for a given input by mapping the result into the desired range, usually $[0,1]$ or $[-1, 1]$ . Some popular examples include:
\begin{itemize}
\item Sigmoid: $\sigma(x) = \frac{1}{1 + e^{-x}}$
\item Hyperbolic tangent: $\tanh(x) = \frac{e^x - e^{-x}}{e^x + e^{-x}}$
\item Rectified linear unit: $ReLU(x) = \max\{0, x\}$ (currently the most widely used of the three) 
\end{itemize}

The inputs to the input layer nodes is a given data point, while the inputs to to all other nodes are the output of the nodes connected to it. We denote by $\rho(v)$ the set of input neighbors to node $v$. The top node output is the input to the loss function, which takes its ``prediction" and the true label to form a loss.

For an input node $v$,  its output as a function of the graph weights and input example $\x$ (of dimension $d$), which we denote as 
$$ v(W,\x) = \sigma \left( \sum_{i \in d} W_{v,i} \x_i \right) $$
The output of an internal node $v$ is a function of its inputs $u \in \rho(v)$ and a given example $\x$, which we denote as 
$$ v( W,\x ) = \sigma  \left(  \sum_{u \in \rho(v)}  W_{uv} u(W,\x) \right)  $$
If we denote the top node as $v^{1}$, then the loss of the network over data point $(\x_i,y_i)$ is given by
$$ \ell(   v^{1} (W,\x_i) , y_i ) . $$
The objective function becomes
$$f(W) =  \E_{\x_i,y_i} \left[ \ell( v^{1} (W,\x_i) , y_i)  \right]  $$

For most commonly-used activation and loss functions, the above function is non-convex. However, it admits important computational properties. The most significant property is given in the following lemma.
\begin{lemma}[Backpropagation lemma]
The gradient of $f$ can be computed in time $O(|E|)$. 
\end{lemma}

The proof of this lemma is left as an exercise, but we sketch the main ideas. 
For every variable $W_{uv}$, we have by linearity of expectation that 
$$ \frac{\partial }{\partial W_{uv} } f (W) =   \E_{\x_i,y_i} \left[ \frac{\partial}{\partial W_{uv} }  \ell( v^{1} (W,\x_i) , y_i)  \right] . $$
Next,  using the chain rule, we claim that it suffices to know the partial derivatives of each node w.r.t. its immediate daughters. To see this, let us write the derivative w.r.t. $W_{uv}$ using the chain rule:
\begin{align*}
 \frac{\partial }{\partial W_{uv} }  \ell( v^{1} (W,\x_i) , y_i)   & = \frac{\partial \ell }{\partial v^1 } \cdot  \frac{\partial v^1}{\partial W_{uv} } \\
& =  \frac{\partial \ell }{\partial v^1 } \cdot \sum_{v^2 \in \rho(v^1) }  \frac{\partial v^1}{\partial v^2 } \cdot   \frac{\partial v_j}{\partial W_{uv}} = ... \\
& =  \frac{\partial \ell }{\partial v^1 } \cdot \sum_{v^2 \in \rho(v^1) }  \frac{\partial v^1}{\partial v^2 } \cdot ... \cdot \sum_{v_j^k \in \rho(v^{k-1}) }  \cdot  \frac{\partial v^{k} }{\partial W_{uv}}  \\
\end{align*}

We conclude that we only need to obtain the $E$ partial derivatives along the edges in order to compute all partial derivatives of the function. The actual product at each node can be computed by a dynamic program in linear time. 

\section{Gradient descent for smooth optimization}

Before moving to stochastic gradient descent, we consider its deterministic counterpart: gradient descent, in the context of smooth non-convex optimization. Our notion of solution is a point with small gradient, i.e. $ \| \nabla f(\x) \| \leq \eps$. 

As we prove below, this requires $O(\frac{1}{\eps^2})$ iterations, each requiring one gradient computation. Recall that gradients can be computed efficiently, linear in the number of edges, in feed forward neural networks. Thus, the time to obtain a $\eps$-approximate solution becomes $O( \frac{|E| m}{\eps^2}) $ for neural networks with $E$ edges and over $m$ examples. 

\begin{algorithm}[h!]
\caption{ Gradient descent }
\label{alg:BasicGD}
\begin{algorithmic}[1]
\STATE Input: $f$, $T$, initial point $\x_1 \in \K$, sequence of step sizes $\{\eta_t\}$
\FOR {$t=1$ to $T$}
\STATE Let $ \y_{t+1} = \x_{t}-\eta_t {\nabla f}(\x_t) , \  \x_{t+1}= \proj_{\K} \left( \y_{t+1}  \right) $
\ENDFOR
\RETURN ${\x}_{T+1} $ 
\end{algorithmic}
\end{algorithm}

Although the choice of $\eta_t$ can make a difference in practice, in theory the convergence of the vanilla GD algorithm is well understood and given in the following theorem. Below we assume that the function is bounded such that $ |f(\x) | \leq M$.

\begin{theorem} \label{thm:basicGDunconstrained}
For unconstrained minimization of $\beta$-smooth functions and $\eta_t = \frac{1}{\beta} $,  GD Algorithm \ref{alg:BasicGD} converges as
$$ \frac{1}{T} \sum_t \| \nabla_t  \|^2 \leq  \frac{4 M \beta}{T} .$$
\end{theorem}
\begin{proof}
Denote by $\nabla_t$ the shorthand for  $\nabla f(\x_t)$, and $h_t = f(\x_t) - f(\x^*)$. The {\bf Descent Lemma} is given in the following simple equation, 
\begin{align*}
h_{t+1} - h_t & =  f(\x_{t+1})  - f(\x_t) \\
& \le   \nabla_t^\top (\x_{t+1} - \x_t) + \frac{\beta}{2} \|\x_{t+1} - \x_t\|^2 & \text{ $\beta$-smoothness} \\
& =  - \eta_t \|\nabla_t \|^2 + \frac{\beta}{2} \eta_t^2  \|\nabla_t\|^2 & \text{ algorithm defn.} \\
& =  - \frac{1}{2\beta} \|\nabla_t \|^2  & \text{ choice of $\eta_t=\frac{1}{\beta}$} 
\end{align*}
Thus, summing up over $T$ iterations, we have 
\begin{eqnarray*}
 \frac{1}{2\beta} \sum_{t=1}^T \|\nabla_t \|^2 \leq \sum_t (h_t - h_{t+1}) =  h_1 - h_{T+1} \leq  2 M
\end{eqnarray*}
\end{proof}

For convex functions, the above theorem implies convergence in function value due to the following lemma, 

\begin{lemma}
A convex function satisfies 
$$ h_t \leq D \| \nabla_t \|  ,   $$
and an $\alpha$-strongly convex function satisfies
$$ h_t \leq  \frac{1}{2 \alpha} \|\nabla_t\|^2 .   $$
\end{lemma}
\begin{proof}
The gradient upper bound for convex functions gives
$$ h_t \leq \nabla_t ( \x^* - \x_t) \leq D  \| \nabla_t \|  $$

The strongly  convex case appears in Lemma \ref{lem:elementary_properties}.

\end{proof}

\section{Stochastic gradient descent}

In the context of training feed forward neural networks, the key idea of Stochastic Gradient Descent is to modify the updates to be:

\begin{equation}
W_{t+1} = W_t - \eta \, \widetilde{\nabla}_t
\end{equation}
where $\widetilde{\nabla}_t$ is a random variable with $\E[\widetilde{\nabla}_t] = \nabla \, \textit{f} \, (W_t)$ and bounded second moment  $\E[\|\widetilde{\nabla}_t\|_2^2] \leq \sigma^2$. 

Luckily, getting the desired $\widetilde{\nabla}_t$ random variable is easy in the posed problem since the objective function is already in expectation form so:
\begin{center}
$\nabla \textit{f}(W) = \nabla \E\limits_{\x_i, y_i} [\ell(v^1(W, \x_i), y_i)] = \E\limits_{\x_i, y_i} [\nabla \ell(v^1(W, \x_i), y_i)]$.
\end{center}

Therefore, at iteration $t$ we can take $\widetilde{\nabla}_t = \nabla \ell(v^1(W, \x_i), y_i)$ where $i \in \{1,..., m\}$ is picked uniformly at random.  Based on the observation above, choosing $\widetilde{\nabla}_t $ this way preserves the desired expectation. So, for each iteration we only compute the gradient w.r.t. to one random example instead of the entire dataset, thereby drastically improving performance for every step. It remains to analyze how this impacts convergence.
\begin{algorithm}[h!]
\caption{ Stochastic gradient descent }
\label{alg:BasicSGD}
\begin{algorithmic}[1]
\STATE Input: $f$, $T$, initial point $\x_1 \in \K$, sequence of step sizes $\{\eta_t\}$
\FOR {$t=1$ to $T$}
\STATE Let $ \y_{t+1} = \x_{t}-\eta_t {\nabla f}(\x_t) , \  \x_{t+1}= \proj_{\K} \left( \y_{t+1}  \right) $
\ENDFOR
\RETURN ${\x}_{T+1} $ 
\end{algorithmic}
\end{algorithm}

\begin{theorem} \label{thm:non-convex-sgd}
For unconstrained minimization of $\beta$-smooth functions and $\eta_t = \eta =  \sqrt{\frac{M}{ \beta \sigma^2 T}}$,  SGD Algorithm \ref{alg:BasicSGD} converges as
$$ \E \left[  \frac{1}{T} \sum_t \| \nabla_t  \|^2 \right]  \leq  2 \sqrt{\frac{M \beta \sigma^2 }{T} }.$$
\end{theorem}
\begin{proof}
Denote by $\nabla_t$ the shorthand for  $\nabla f(\x_t)$, and $h_t = f(\x_t) - f(\x^*)$. The stochastic descent lemma is given in the following  equation, 
\begin{align*}
\E[ h_{t+1} - h_t ] & =  \E [ f(\x_{t+1})  - f(\x_t) ] \\
& \le  \E[  \nabla_t^\top (\x_{t+1} - \x_t) + \frac{\beta}{2} \|\x_{t+1} - \x_t\|^2 ]  & \text{ $\beta$-smoothness} \\
& =  - \E[  \eta \nabla _t^\top  \tilde{\nabla}_t]  + \frac{\beta}{2} \eta^2\E   \|\tilde{\nabla}_t\|^2 & \text{ algorithm defn.} \\
& =  - \eta \|{\nabla}_t \|^2  + \frac{\beta}{2} \eta^2 \sigma^2  & \text{ variance bound.} 
\end{align*}
Thus, summing up over $T$ iterations, we have for $\eta = \sqrt{\frac{M}{ \beta \sigma^2 T}}$, 
\begin{eqnarray*}
\E \left[ \frac{1}{T} \sum_{t=1}^T \|\nabla_t \|^2 \right] & \leq \frac{1}{T \eta} \sum_t \E \left[ h_t - h_{t+1} \right] + \eta   \frac{\beta}{2}  \sigma^2  \leq  \frac{M}{T \eta}  + \eta   \frac{\beta}{2}  \sigma^2 \\
& = \sqrt{\frac{M \beta \sigma^2}{T}} + \frac{1}{2}  \sqrt{ \frac{M  \beta \sigma^2}{T}}  \leq 2 \sqrt{\frac{M \beta \sigma^2 }{ T}}  .
\end{eqnarray*}
\end{proof}

We thus conclude that $O(\frac{1}{\eps^4})$ iterations are needed to find a point with $\|\nabla f(\x)\| \leq \eps$, as opposed to $O(\frac{1}{\eps^2})$. However, each iteration takes $O(|E|)$ time, instead of $O(|E| m)$ time for gradient descent. 

This is why SGD is one of the most useful algorithms in machine learning.

\newpage
\section{Bibliographic remarks}

For in depth treatment of backpropagation and the role of deep neural networks in machine learning the reader is referred to \cite{Goodfellow-et-al-2016}.

For detailed rigorous convergence proofs of first order methods, see lecture notes by Nesterov \cite{NesterovBook} and Nemirovskii \cite{NY83,Nemirovski04lectures}, as well as the recent text \cite{bubeckOPT}.

\chapter{Generalization and Non-Smooth Optimization} \label{chap:first order}
\chaptermark{Generalization}

In previous chapter we have introduced the framework of mathematical optimization within the context of machine learning. We have described the mathematical formulation of several machine learning problems, notably training neural networks, as optimization problems. We then described as well as analyzed the most useful optimization method to solve such formulations: stochastic gradient descent.

However, several important questions arise:
\begin{enumerate}
\item
SGD was analyzed for smooth functions. Can we minimize non-smooth objectives?  

\item
Given an ERM problem (a.k.a. learning from examples, see first chapter), what can we say about generalization to unseen examples?  How does it affect optimization?

\item 
Are there faster algorithms than SGD in the context of ML?

\end{enumerate}

In this chapter we address the first two, and devote the rest of this manuscript/course to the last question.

How many examples are needed to learn a certain concept? This is a fundamental question of statistical/computational learning theory that has been studied for decades (see end of chapter for bibliographic references). 

The classical setting of learning from examples is statistical. It assumes examples are drawn i.i.d from a fixed, arbitrary and unknown distribution. The mathematical optimization formulations that we have derived for the ERM problem assume that we have sufficiently many examples, such that optimizing a certain predictor/neural-network/machine on them will result in a solution that is capable of generalizing to unseen examples.  The number of examples needed to generalize is called the {\it sample complexity} of the problem, and it depends on the concept we are learning as well as the hypothesis class over which we are trying to optimize. 

There are dimensionality notions in the literature, notably the VC-dimension and related notions, that give precise bounds on the sample complexity for various hypothesis classes. In this text we take an algorithmic approach, which is also deterministic. Instead of studying sample complexity, which is non-algorithmic, we study algorithms for regret minimization. We will show that they imply generalization for a broad class of machines.

\section{A note on non-smooth optimization}

Minimization of a function that is both non-convex and non-smooth is in general hopeless, from an information theoretic perspective. The following image explains why. The depicted function on the interval $[0,1]$ has a single local/global minimum, and if the crevasse is narrow enough, it cannot be found by any method other than extensive brute-force search, which can take arbitrarily long.

\begin{figure}[h!]
\begin{center}
\includegraphics[width=3in]{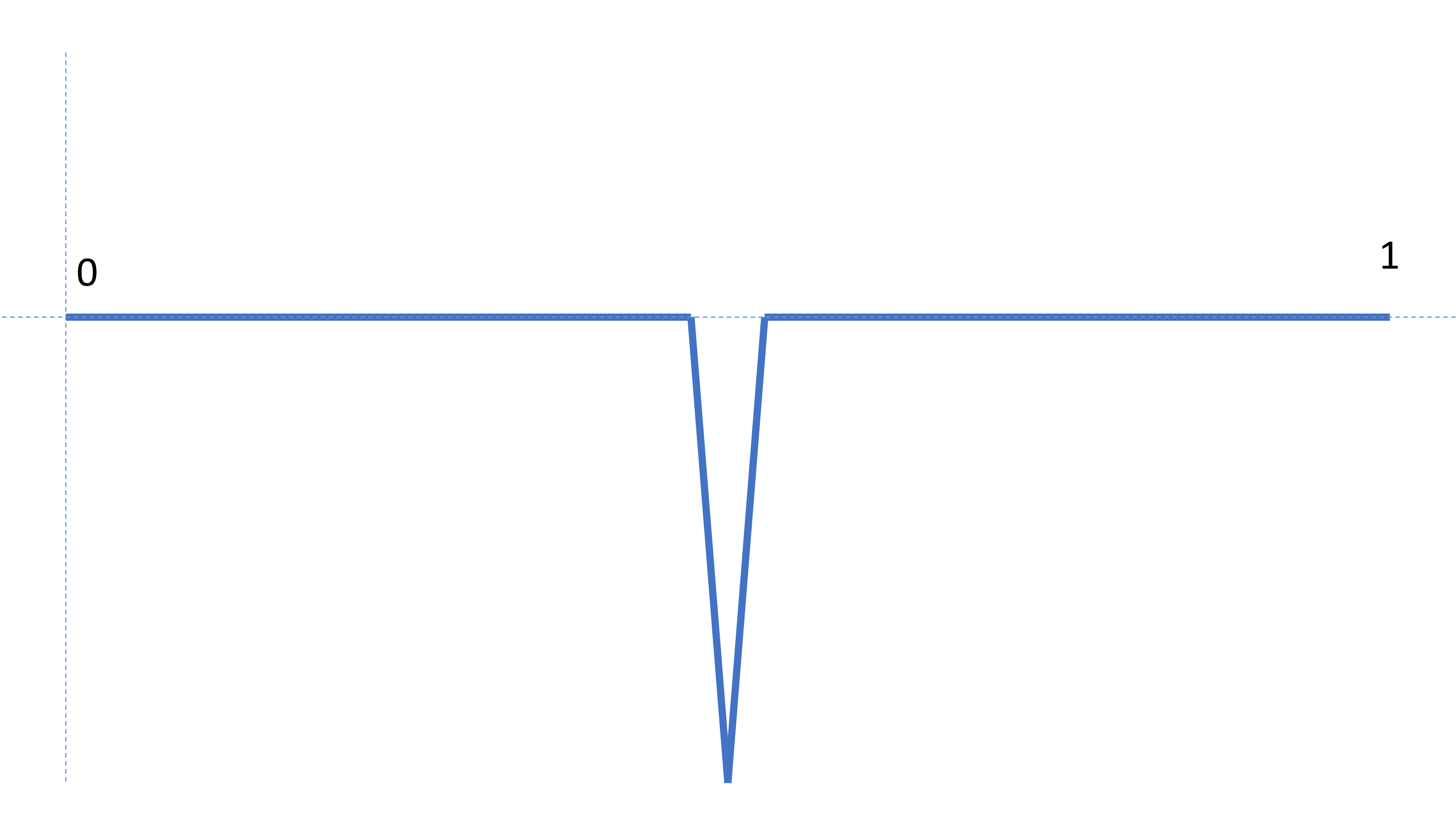}
\end{center}
\caption{Intractability of nonsmooth optimization \label{fig:nonsmooth}}
\end{figure}

Since non-convex and non-smooth optimization is hopeless, in the context of non-smooth functions we only consider  {\bf  convex} optimization.

\section{Minimizing Regret}

The setting we consider for the rest of this chapter is that of online (convex) optimization. In this setting a learner iteratively predicts a point $\x_t \in \K$ in a convex set $\K \subseteq \reals^d$, and then receives a cost according to an adversarially chosen convex function $f_t \in \F$ from family $\F$.

The goal of the algorithms introduced in this chapter is to minimize worst-case {\it regret}, or difference between total cost and that of best point in hindsight:
$$ \regret = \sup_{f_1,...,f_T \in \F} \left\{ \sum_{t=1}^{T} f_t(\bx_t) -\min_{\bx \in \K}\sum_{t=1}^{T} f_t(\bx) \right\} . $$

In order to compare regret to optimization error it is useful to consider the average regret, or ${\regret}/{T} $. Let  $\bar{\x}_T = \frac{1}{T} \sum_{t=1}^T \x_t$ be the average decision. If the functions $f_t$ are all equal to a single function $f : \K\mapsto \reals$, then Jensen's inequality implies that $f( \bar{\x}_T)$ converges to $f(\x^\star)$ if the average regret is vanishing, since
$$ f(\bar{\x}_T) - f(\x^\star ) \leq \frac{1}{T} \sum_{t=1} ^T [f(\x_t)  - f(\x^\star) ] = \frac{\regret}{T} $$

\section{Regret implies generalization}

Statistical learning theory for learning from examples postulates that examples from a certain concept are sampled i.i.d. from a fixed and unknown distribution. The learners' goal is to choose a hypothesis from a certain hypothesis class that can generalize to unseen examples. 

More formally, let $\D$ be a distribution over labelled examples $\{\ba_i  \in \reals^d , b_i \in \reals\} \sim \D$. Let $\H = \{ \x \} \ , \ \x: \reals^d \mapsto \reals$ be a hypothsis class over which we are trying to learn (such as linear separators, deep neural networks, etc.). The {\it generalization error} of a hypothesis is the expected error of a hypothesis over randomly chosen examples according to a given loss function $\ell : \reals \times \reals \mapsto \reals$, which is applied to the prediction of the hypothesis and the true label, $\ell(\x(\ba_i) , b_i)$.  Thus,
$$ \err( \x) = \E_{\ba_i , b_i \sim \D} [ \ell (\x(\ba_i) , b_i ) ] .$$

An algorithm that attains sublinear regret over the hypothesis class $\H$, w.r.t. loss functions given by $f_t(\x) = f_{\ba,b} (\x) = \ell( \x(\ba) , b)$, gives rise to a generalizing hypothesis as follows. 

\begin{lemma}
Let $\xbar =  \x_t$ for $t \in [T] $ be chose uniformly at random from $\{\x_1,...,\x_T\}$.
Then, with expectation taken over random choice of $\xbar$ as well as choices of $f_t \sim \D$,
$$ \E [ \err(\bar{\x}) ]  \leq \E [ \err(\x^* ) ]  +  \frac{\regret}{T} $$
\end{lemma}
\begin{proof}
By random choice of $\xbar$, we have
$$ \E [ f(\xbar) ] =  \E\left[ \frac{1}{T} \sum_t f(\x_t) \right]$$
Using the fact that $f_t \sim \D$, we have
\begin{eqnarray*}
\E  [ \err(\xbar) ]  & = \E_{f \sim \D} [ f(\xbar)]  \\
& =  \E_{f_t} [ \frac{1}{T} \sum_t  f_t(\x_t)]  \\
& \leq  \E_{f_t} [ \frac{1}{T} \sum_t  f_t(\x^\star) ]  + \frac{\regret}{T} \\
& =  \E_{f} [ f(\x^\star) ]  + \frac{\regret}{T} \\
& = \E_f [ \err(\x^\star)] + \frac{\regret}{T} 
\end{eqnarray*}
\end{proof}

\section{Online gradient descent} \label{section:ogd}

Perhaps the simplest algorithm that applies to the most general
setting of online convex optimization is online gradient descent.
This algorithm is an online version of  standard gradient descent
 for offline optimization we have seen in the previous chapter.
Pseudo-code for the algorithm is given in Algorithm \ref{figure:ogd}, and a conceptual illustration is given in Figure \ref{fig:ogd}.

\begin{figure}[h!] 
	\begin{center}
		\includegraphics[width=4in]{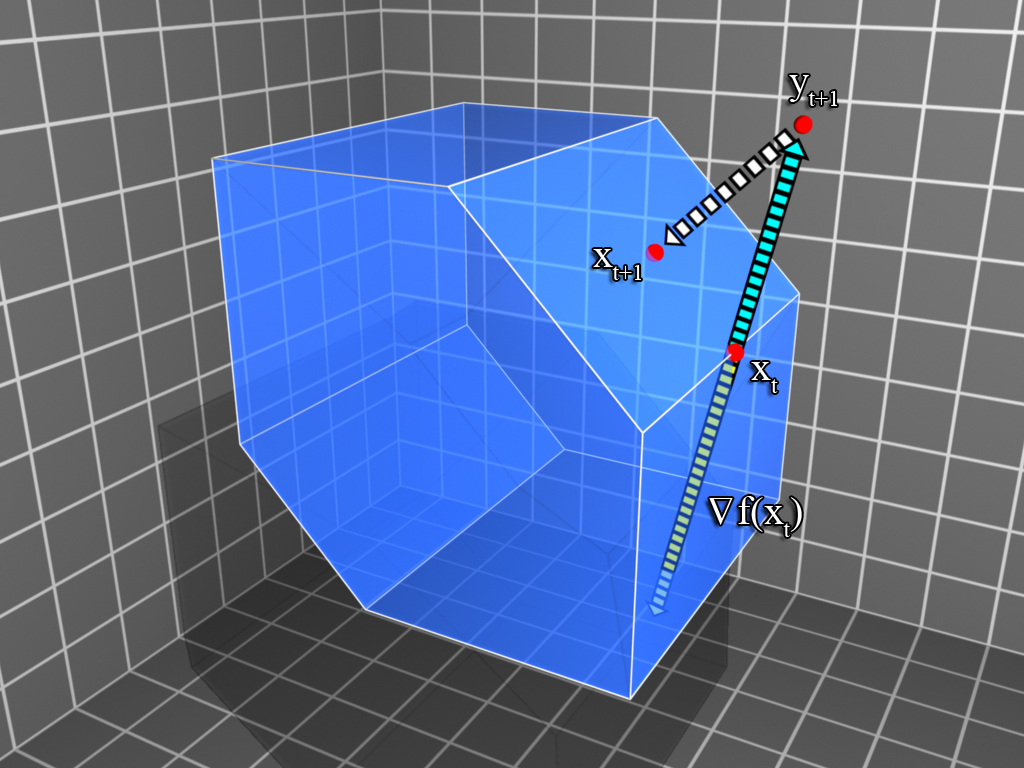}
	\end{center}
	\caption{Online gradient descent: the iterate $\x_{t+1}$ is derived by advancing $\x_t$ in the direction of the current gradient $\nabla_t$, and projecting back into $\K$. \label{fig:ogd}}
\end{figure}

In each iteration, the algorithm takes a step from the previous point in the direction of the gradient of the previous cost.  This step may  result in a point outside of the underlying convex set. In such cases,
the algorithm projects the point back to the convex set, i.e.
finds its closest  point in the convex set.
Despite the fact that the next cost function may be completely
different than the costs observed thus far, the regret
attained by the algorithm is sublinear. This is formalized in the following
theorem (recall the definition of $G$ and $D$ from the previous chapter).

\begin{algorithm}[h!]
		\caption{\label{figure:ogd} \ogd }
		\begin{algorithmic}[1]
			\STATE Input: convex set $\K$, $T$, $\x_1 \in \mathcal{K}$, step sizes $\{ \eta_t \}$
			\FOR {$t=1$ to $T$}
			\STATE Play $\x_t$ and observe cost $f_t(\x_t)$. 
			\STATE Update and project:
			\begin{align*}
			& \y_{t+1} = \bx_{t}-\eta_{t} \nabla f_{t}(\bx_{t}) \\
			& \bx_{t+1} = \proj_\K(\y_{t+1})
			\end{align*}
			\ENDFOR
		\end{algorithmic}
\end{algorithm}

\begin{theorem}\label{thm:gradient}
Online gradient descent with step sizes $\{\eta_t = \frac{D}{G \sqrt{t}} , \ t \in [T] \}$ guarantees the following for all $T \geq 1$:
$$ \regret_T = \sum_{t=1}^{T} f_t(\bx_t) -\min_{\bx^\star \in \K}\sum_{t=1}^{T}
f_t(\bx^\star)\ \leq  {3} {G D}\sqrt{T} $$
\end{theorem}

\begin{proof}
Let $\bx^\star \in \argmin_{\bx \in \K} \sum_{t=1}^T f_t(\bx)$.
Define $\nabla_t \equaltri \nabla f_{t}(\bx_{t})$. By convexity
\begin{eqnarray}  \label{eqn:gradient_inequality}
f_t(\bx_t) - f_t(\bx^\star) \leq   \nabla_t^\top (\bx_t - \bx^\star)
\end{eqnarray}
We first  upper-bound $\nabla_t^\top (\bx_t-\bx^\star)$ using the update
rule for $\bx_{t+1}$ and Theorem  \ref{thm:pythagoras} (the Pythagorean theorem):
\begin{equation} \label{eqn:ogdtriangle}
\| \bx_{t+1}-\bx^\star \|^2\ =\  \left\|\proj_\K (\bx_t - \eta_t
\nabla_{t}) -\bx^\star\right\|^2 \leq  \left\|\bx_t - \eta_t \nabla_t-\bx^\star\right\|^2
\end{equation}

Hence,
\begin{eqnarray} \label{eqn:ogd_eq2}
\|\bx_{t+1}-\bx^\star\|^2\ &\leq&\ \|\bx_t- \bx^\star\|^2 + \eta_t^2
\|\nabla_t\|^2 -2 \eta_t \nabla_t^\top (\bx_t -\bx^\star)\nonumber\\
2 \nabla_t^\top (\bx_t-\bx^\star)\ &\leq&\ \frac{ \|\bx_t-
\bx^\star\|^2-\|\bx_{t+1}-\bx^\star\|^2}{\eta_t} + \eta_t G^2
\end{eqnarray}
Summing \eqref{eqn:gradient_inequality} and \eqref{eqn:ogd_eq2}  from $t= 1$ to $T$, and setting $\eta_t =
\frac{D}{G \sqrt{t}}$ (with $\frac{1}{\eta_0} \equaltri 0$):
\begin{align*}
& 2 \left( \sum_{t=1}^T f_t(\bx_t)-f_t(\bx^\star) \right ) \leq 2\sum_{t=1}^T \nabla_t^\top (\xv - \x^\star) \\
&\leq  \sum_{t=1}^T \frac{ \|\bx_t-
	\bx^\star\|^2-\|\bx_{t+1}-\bx^\star\|^2}{\eta_t} + G^2 \sum_{t=1}^T \eta_t    \\
&\leq  \sum_{t=1}^T \|\bx_t - \bx^\star\|^2 \left(
\frac{1}{\eta_{t}} -
\frac{1}{\eta_{t-1}} \right) + G^2 \sum_{t=1}^T \eta_t & \frac{1}{\eta_0} \equaltri 0, \\
& &  \|\xv[T+1] - \xv[]^* \|^2 \geq 0 \\
&\leq D^2 \sum_{t=1}^T \left(
\frac{1}{\eta_{t}} -
\frac{1}{\eta_{t-1}} \right) + G^2 \sum_{t=1}^T \eta_t \\
& \leq  D^2  \frac{1}{\eta_{T}}  + G^2 \sum_{t=1}^T \eta_t  & \mbox{ telescoping series } \\
& \leq 3 DG \sqrt{T}.
\end{align*}
The last inequality follows since $\eta_t = \frac{D}{G \sqrt{t}}$ and  $\sum_{t=1}^T \frac{1}{\sqrt{t}} \leq 2 \sqrt{T}$.
\end{proof}

The \ogd algorithm is straightforward to implement, and updates
take linear time given the gradient. However, there is a
projection step which may take significantly longer.

\section{Lower bounds} \label{section:lowerbound}

\begin{theorem} \label{thm:lowerbound}
Any algorithm for online convex optimization incurs $\Omega(DG
\sqrt{T})$ regret in the worst case. This is true even if  the cost functions are generated from
a fixed stationary distribution.
\end{theorem}

We give a sketch of the proof; filling in all details is left as an exercise at the end of this chapter. 

Consider an instance of OCO where the convex set $\K$ is the
$n$-dimensional hypercube, i.e.
$$ \K = \{ \bx \in \reals^n \ , \  \|\bx\|_\infty \leq 1 \}.$$
There are $2^n$ linear cost functions, one for each vertex $\bv \in \{
\pm 1\}^n$, defined as
$$ \forall \bv \in \{ \pm 1 \}^n \ , \ f_\bv(\bx)  = \bv^\top \bx. $$
Notice that both the diameter of $\K$ and
the bound on the norm of the cost function gradients, denoted G, are bounded by
$$  D \leq \sqrt{ \sum_{i=1}^n 2^2 } = 2 \sqrt{n} , \ G = \sqrt{ \sum_{i=1}^n (\pm1)^2 } = \sqrt{n}  $$

The cost functions in each iteration are chosen at random, with uniform probability, from the set  $\{ f_\bv , \bv \in \{\pm 1\}^n \}$. Denote by $\bv_t \in \{\pm 1\}^n $ the vertex chosen in
iteration $t$, and denote $f_t = f_{\bv_t}$. By uniformity and independence, for any $t$ and
$\bx_t$ chosen online, $\E_{\bv_t}[f_{t}(\bx_t)]= \E_{\bv_t}[ \bv_t^\top
\bx_t] = 0$. However,
\begin{align*}
\E_{\bv_1,\ldots,\bv_T}\left[\min_{\bx \in \K} \sum_{t=1}^T f_t(\bx)\right] & =
\E \left[\min_{\bx \in \K} \sum_{i \in [n]} \sum_{t=1}^T \bv_t(i) \cdot \bx_i \right] \\
& = n \E\left[-\left|\sum_{t=1}^T \bv_t(1) \right|\right] & \mbox{i.i.d. coordinates}\\
& = -\Omega(n \sqrt{T}).
\end{align*}
The last equality is left as exercise \ref{exercise:ogd-lower-bound}. 

The facts above nearly complete the proof of Theorem \ref{thm:lowerbound}; see the exercises at the end of this chapter.

\section{Online gradient descent for strongly convex functions} \label{section:ogdnew}

The first algorithm that achieves regret logarithmic in the number
of iterations is a twist on the online gradient descent algorithm, changing only the step size. The following theorem establishes logarithmic bounds on the regret if the
cost functions are strongly convex.

\begin{theorem}\label{thm:gradient2}
For  $\alpha$-strongly convex loss functions, 
\ogd with step sizes $\eta_t = \frac{1}{\alpha {t}}$ achieves the
following guarantee for all $T \geq 1$
$$\regret_T\ \leq\  \sum_{t=1}^{T} \frac{1}{\alpha t} \|\nabla_t\|^2 \leq \frac{G^2}{2 \alpha}(1 + \log T).$$
\end{theorem}

\begin{proof}
Let $\bx^\star \in \argmin_{\bx \in \K} \sum_{t=1}^T f_t(\bx)$.
Recall the definition of regret 
$$ \regret_T\ = \sum_{t=1}^{T} f_t(\bx_t) - \sum_{t=1}^{T} f_t(\bx^\star). $$

Define $\nabla_t \equaltri \nabla f_t(\bx_t)$. Applying the definition of $\alpha$-strong convexity to the pair of points $\x_t$,$\x^*$, we have
\begin{eqnarray}
2(f_t(\bx_t)-f_t(\bx^\star)) &\leq& 2\nabla_t^\top (\bx_t-\bx^\star)-\alpha
\|\bx^\star-\bx_t\|^2.\label{eqsz}
\end{eqnarray}
We proceed to upper-bound $\nabla_t^\top
(\bx_t-\bx^\star)$. Using the update rule for $\bx_{t+1}$ and the
Pythagorean theorem \ref{thm:pythagoras}, we
get
$$\| \bx_{t+1}-\bx^\star \|^2\ =\ \|\proj_\K(\bx_t - \eta_{t} \nabla_t)-\bx^\star\|^2 \leq  \|\bx_t - \eta_{t} \nabla_t-\bx^\star\|^2.$$
Hence,
\begin{eqnarray*}
\|\bx_{t+1}-\bx^\star\|^2\ &\leq&\ \|\bx_t- \bx^\star\|^2 + \eta_{t}^2
\|\nabla_t\|^2 -2
\eta_{t} \nabla_t^\top (\bx_t - \bx^\star)\nonumber\\
\end{eqnarray*}
and
\begin{eqnarray}
2 \nabla_t^\top (\bx_t-\bx^\star)\ &\leq&\ \frac{ \|\bx_t-
\bx^\star\|^2-\|\bx_{t+1}-\bx^\star\|^2}{\eta_{t}} + \eta_{t} \| \nabla_t\|^2.
\label{eqer}
\end{eqnarray}
Summing  \eqref{eqer} from $t= 1$ to $T$, setting $\eta_{t} = \frac{1}{\alpha t}$ (define $\frac{1}{\eta_0} \equaltri 0$),
and combining with \eqref{eqsz}, we have:
\begin{eqnarray*}
&  & 2 \sum_{t=1}^T (f_t(\bx_t)-f_t(\bx^\star) ) \\
&\leq &\
 \sum_{t=1}^T \|\bx_t-\bx^\star\|^2
\left(\frac{1}{\eta_{t}}-\frac{1}{\eta_{t-1}}-\alpha\right) + 
\sum_{t=1}^{T} \eta_{t} \|\nabla_t\|^2  \\
& & \mbox{ since }  \frac{1}{\eta_0} \equaltri 0, \|\xv[T+1] - \xv[]^* \|^2 \geq 0 \\ \\
&=&\ 0 +  \sum_{t=1}^{T} \frac{1}{\alpha t} \|\nabla_t\|^2 \\
& \leq &  \frac{G^2}{\alpha}(1 + \log T )
\end{eqnarray*}
\end{proof}

\section{Online Gradient Descent implies SGD}

In this section we notice that OGD and its regret bounds imply the SGD bounds we have studied in the previous chapter. The main advantage are the guarantees for non-smooth stochastic optimization, and constrained optimization. 

Recall that in stochastic optimization, the optimizer attempts to minimize a convex function over a convex domain as given by the mathematical program:
\begin{align*}
\min_{\x \in \K } f(\x). 
\end{align*}
However, unlike standard offline optimization, the optimizer is given access to a noisy gradient oracle, defined by
$$ \mathcal{O}(\x) \equaltri \tilde{\nabla }_\x  \ \mbox{ s.t. } \  \E[\tilde{\nabla }_\x ] = \nabla f(\x) \ , \   \E[ \|\tilde{\nabla }_\x\|^2 ] \leq G^2  $$
That is, given a point in the decision set, a noisy gradient oracle returns a random vector whose expectation is the gradient at the point and whose second moment is bounded by $G^2$. 
	
We will show that regret bounds for OCO  translate to convergence rates for stochastic optimization.  As a special case, consider the online gradient descent algorithm whose regret is bounded by 
$$\regret_{T} = O(DG\sqrt{T}) $$
Applying the OGD algorithm over a sequence of linear functions that are defined by the noisy gradient oracle at consecutive points, and finally returning the average of all points along the way, we obtain the stochastic gradient descent algorithm, presented in Algorithm \ref{alg:sgd}.
\begin{algorithm}[h!]
\caption{stochastic  gradient descent}
\label{alg:sgd}
\begin{algorithmic}[1]
\STATE Input: $f$, $\K$, $T$, $\x_1 \in \mathcal{K}$, step sizes $\{ \eta_t \}$
\FOR {$t=1$ to $T$}
\STATE \label{alg:sgd-defnft} Let $\tilde{\nabla}_t = \mathcal{O}(\x_t)$ and define: $ f_t(\x) \equaltri  \langle \tilde{\nabla}_t , \x \rangle $ 
\STATE Update and project:
$$ \y_{t+1} = \bx_{t}-\eta_t \tilde{\nabla}_t$$
$$ \bx_{t+1} = \proj_\K(\y_{t+1})$$
\ENDFOR
\RETURN $\bar{\x}_T \equaltri \frac{1}{T} \sum_{t=1}^T \x_t$ 
\end{algorithmic}
\end{algorithm}

\begin{theorem} \label{thm:sgd}
Algorithm \ref{alg:sgd} with step sizes $\eta_t = \frac{D}{G \sqrt{t}}$ guarantees
$$ \E[ f(\bar{\x}_T) ]  \leq \min_{\x^\star \in \K } f(\x^\star) + \frac{3 GD }{\sqrt{T}}$$
\end{theorem}
\begin{proof}
By the regret guarantee of OGD, we have
\begin{align*}
&  \E [ f(\bar{\x}_T) ] - f(\x^\star) \\
 & \leq \E [ \frac{1}{T} \sum_t  f(\x_t) ]  - f(\x^\star)  & \mbox{ convexity of $f$ (Jensen) }\\
&\leq \frac{1}{T}  \E [ \sum_t  \langle \nabla f(\x_t) , \x_t - \x^\star\rangle ] & \mbox{ convexity again }\\
& = \frac{1}{T} \E[ \sum_t \langle \tilde{\nabla}_t , \x_t -\x^\star \rangle ] & \mbox{ noisy gradient estimator }\\
& = \frac{1}{T} \E[ \sum_t  f_t( \x_t)  -f_t(\x^\star) ] & \mbox{ Algorithm \ref{alg:sgd}, line \eqref{alg:sgd-defnft} }\\
& \leq  \frac{\regret_T }{T} & \mbox{ definition }\\
& \leq  \frac{3GD }{\sqrt{T}} & \mbox{ theorem \ref{thm:gradient}}
\end{align*}
\end{proof}
	
It is important to note that in the proof above, we have used the fact that the regret bounds of \ogd hold against an adaptive adversary. This need arises since the cost functions $f_t$ defined in Algorithm \ref{alg:sgd} depend on the choice of decision $\x_t \in \K$. 
	
In addition, the careful reader may notice that by plugging in different step sizes (also called learning rates) and applying SGD to strongly convex functions, one can attain $\tilde{O}({1}/{T})$ convergence rates. Details of this derivation are left as exercise \ref{exercise:fast-rate}.

\newpage
\section{Exercises}

\begin{enumerate}
\item \label{exercise:fast-rate}
	
Prove that SGD for a strongly convex function can, with appropriate parameters $\eta_t$, converge as $\tilde{O}(\frac{1}{T})$.  You may assume that the gradient estimators have Euclidean norms bounded by the constant $G$.  

\item
Design an OCO algorithm that attains the same asymptotic regret bound as OGD, up to factors logarithmic in $G$ and $D$,  without knowing the parameters $G$ and $D$ ahead of time.

\item \label{exercise:ogd-lower-bound}
In this exercise we prove a tight lower bound on the regret of any algorithm for online convex optimization.
\begin{enumerate}
\item
For any sequence of $T$ fair coin tosses, let $N_h$ be the number of head outcomes and $N_t$ be the number of tails. Give an asymptotically tight upper and lower bound on $\E[ \left| N_h - N_t \right|]$ (i.e., order of growth of this random variable as a function of $T$, up to multiplicative and additive constants). 
\item
Consider a 2-expert problem, in which the losses are inversely correlated: either expert one incurs a loss of one and the second expert zero, or vice versa. Use the fact above to design a setting in which any experts algorithm incurs regret asymptotically matching the upper bound.
\item
Consider the general OCO setting over a convex set $\mathcal{K}$. Design a setting in which the cost functions have gradients whose norm is bounded by $G$, and obtain a lower bound on the regret as a function of $G$, the diameter of $\mathcal{K}$, and the number of game iterations.
\end{enumerate}

\end{enumerate}

\newpage
\section{Bibliographic remarks}

The OCO framework was introduced by Zinkevich in \cite{Zinkevich03}, where the OGD algorithm was introduced and analyzed. Precursors to this algorithm, albeit for less general settings, were introduced and analyzed in \cite{KivWar97}. 
Logarithmic regret algorithms for Online Convex Optimization were introduced and analyzed in \cite{HAK07}.  For more detailed exposition on this prediction framework and its applications see \cite{OCObook}.

The SGD algorithm dates back to Robbins and Monro \cite{robbins1951}. Application of SGD to soft-margin SVM training was explored in \cite{Shalev-ShwartzSSC11}. Tight convergence rates of SGD for strongly convex and non-smooth functions were only recently obtained in \cite{hazan:beyond},\cite{RSS},\cite{SZ}.

\chapter{Regularization} \label{chap:regularization}

In this chapter we consider a generalization of the gradient descent called by different names in different communities (such as mirrored-descent, or regularized-follow-the-leader).  The common theme of this generalization is called {\it Regularization}, a concept that is founded in generalization theory. Since this course focuses on optimization rather than generalization, we shall refer the reader to the generalization aspect of regularization, and focus hereby on optimization algorithms.

We start by motivating this general family of methods using the fundamental problem of decision theory.

\section{Motivation:  prediction from expert advice} \label{sec:experts}

Consider the following fundamental iterative decision making problem:

At each time step $t=1,2,\ldots,T$, the decision maker faces a choice between two actions  $A$ or $B$ (i.e.,  buy or sell a certain stock). The decision maker has assistance in the form of  $N$  ``experts'' that offer their advice. After a choice between the two actions has been made, the decision maker receives feedback in the form of a loss associated with each decision.  For simplicity one of the actions receives a loss of zero (i.e., the ``correct'' decision) and the other a loss of one. 

We make the following elementary observations:
\begin{enumerate} 
\paragraph{Simple observations:}
\item
A decision maker that chooses an  action uniformly at random each iteration,  trivially attains a loss of $\frac{T}{2} $ and is ``correct''  $50\%$ of the time.
\item
In terms of the number of mistakes, no algorithm can do better in the worst case! In a later  exercise, we will  devise  a randomized setting in which the expected number of mistakes of any algorithm is at least $\frac{T}{2}$. 
\end{enumerate}

We are thus motivated to consider a {\it relative performance metric}: can the decision maker make as few mistakes as the best expert in hindsight? 
The next theorem shows that the answer in the worst case  is negative for a deterministic decision maker.
\begin{theorem}
Let $L \leq \frac{T} {2} $ denote the number of mistakes made by the best expert in hindsight. Then there does not exist a deterministic algorithm that can guarantee less than $2L$ mistakes.
\end{theorem}

\begin{proof}

Assume that there are only two experts and one always chooses option $A$ while the other always chooses option $B$.
Consider the setting in which an adversary always chooses the opposite of our prediction (she can do so, since our algorithm is deterministic).
Then, the total number of mistakes the algorithm makes is $T$.
However, the best expert makes no more than $\frac{T}{2}$ mistakes (at every iteration exactly one of the two experts is mistaken).
Therefore, there is no algorithm that can always guarantee less than $2L$ mistakes.

\end{proof}

This observation motivates the design of random decision making algorithms, and indeed, the OCO framework  gracefully models decisions on a continuous probability space. Henceforth we prove Lemmas \ref{lem:wm} and \ref{lem:rwm} that show the following: 
 
\begin{theorem}
Let $\eps \in (0,\frac{1}{2} )$. Suppose the best expert makes $L$  mistakes. Then:
\begin{enumerate}
\item
There is an efficient deterministic algorithm that can guarantee less than $2(1+\epsilon)L + \frac{2\log N}{\epsilon}$  mistakes;
\item
There is an efficient randomized algorithm for which the expected number of mistakes is at most $(1+\epsilon)L  + \frac{\log N}{\epsilon}$.
\end{enumerate}
\end{theorem}

\subsection {The weighted majority algorithm}

The weighted majority (WM) algorithm is intuitive to describe: 
each expert $i$ is assigned a weight $W_t(i)$ at every  iteration $t$.
Initially, we set $W_1(i) = 1$ for all experts $i \in [N]$.
For all $t \in [T]$ let $S_t(A),S_t(B) \subseteq [N] $ be the set of experts that choose $A$ (and respectively $B$) at time $t$.  Define,
\[
W_t(A) = \smashoperator[r]{\sum_{i \in S_t(A)}} W_t(i) \qquad  \qquad W_t(B) = \smashoperator[r]{\sum_{i \in S_t(B)}} W_t(i) 
\]
and predict according to
\begin{equation*}
a_t =
\begin{cases}
A & \text{if $W_t(A) \ge W_t(B)$}\\
B & \text{otherwise.}
\end{cases}
\end{equation*}
Next,  update the weights $W_t(i)$ as follows:
\begin{equation*}
W_{t+1}(i) =
\begin{cases}
W_t(i) & \text{if expert $i$ was correct}\\
W_t(i)  (1-\eps) & \text{if expert $i$ was wrong}
\end{cases}
,
\end{equation*}
where $\eps$ is a parameter of the algorithm that will affect its performance. This concludes the description of the WM algorithm. We proceed to bound the number of  mistakes it makes. 
\begin{lemma} \label{lem:wm}
Denote by $M_t$ the number of mistakes the algorithm makes until time $t$, and by $M_t(i)$ the number of mistakes made by expert $i$ until time $t$.
Then, for any expert $i \in [N]$ we have
\[
M_T \le 2(1+\epsilon)M_T(i) + \frac{2\log N}{\epsilon} .
\]
\end{lemma}

\noindent We can optimize $\epsilon$ to minimize the above bound.
The expression on the right hand side is of the form $f(x)=ax+b/x$, that reaches its minimum at $x=\sqrt{b/a}$.
Therefore the bound is minimized at $\epsilon^\star = \sqrt{\log N/M_T(i)}$.
Using this optimal value of $\epsilon$, we get that for the best expert $i^\star$ 
\[
M_T \le 2M_T(i^\star) + O\left(\sqrt {M_T(i^\star)\log N}\right).
\]
Of course, this value of  $\epsilon^\star$ cannot be used in advance since we do not know which expert is the best one ahead of time (and therefore we do not know the value of $M_T(i^\star)$). However, we shall see later on that the same asymptotic bound can be obtained even without this prior knowledge. 

Let us now prove Lemma \ref{lem:wm}.

\begin{proof}

Let $\Phi_t= \sum_{i=1}^N W_t(i)$ for all $t \in [T]$, and note that $\Phi_1=N$.

Notice that $\Phi_{t+1} \le \Phi_t$. However, on iterations in which the WM algorithm erred, we have 
$$\Phi_{t+1} \le \Phi_t(1-\frac{\epsilon}{2}) ,$$
the reason being that experts with at least half of total weight were wrong (else WM would not have erred), and therefore
\[
\Phi_{t+1} \le  \frac{1}{2} \Phi_t(1-\epsilon) + \frac {1} {2} \Phi_t =\Phi_t(1-\frac {\epsilon}{2}) .
\]
From both observations,
\[
\Phi_{t} \le \Phi_1 (1-\frac{\epsilon}{2})^{M_t} = N (1-\frac{\epsilon}{2})^{M_t} .
\]
On the other hand, by definition we have for any expert $i$ that
\[
W_T(i) = (1-\epsilon)^{M_T(i)} .
\]
Since the value of $W_T(i)$ is always less than the sum of all weights $\Phi_T$, we conclude that
\[
(1-\epsilon)^{M_T(i)} = W_T(i) \le \Phi_T \le N(1-\frac{\epsilon}{2})^{M_T}.
\]
Taking the logarithm of both sides we get
\[
M_T(i)\log(1-\epsilon) \le \log{N} + M_T\log{(1-\frac{\epsilon}{2})}  .
\]
Next, we use the approximations
\[
-x-x^2 \le \log{(1-x)} \le -x  \qquad  \quad 0 < x < \frac{1}{2},
\]
which follow from the Taylor series of the logarithm function, to obtain that
\[
-M_T(i)(\epsilon+\epsilon^2) \le \log{N} - M_T\frac {\epsilon}{2} ,
\]
and the lemma follows.
\end{proof}

\subsection{Randomized weighted majority}
In the randomized version of the WM algorithm, denoted RWM, we choose expert $i$ w.p. $p_t(i) = W_t(i) /  \sum_{j=1}^N W_t(j)$ at time $t$.
\begin{lemma} \label{lem:rwm}
Let $M_t$ denote the number of mistakes made by RWM until iteration $t$. Then, for any expert $i \in [N]$ we have
\[
\E[ M_T]  \le (1+\epsilon)M_T(i) + \frac{\log N}{\epsilon} .
\]
\end{lemma}
\noindent The proof of this lemma is very similar to the previous one, where the factor of two is saved by the use of randomness:
\begin{proof}

As before, let $\Phi_t= \sum_{i=1}^N W_t(i)$ for all $t \in [T]$, and note that $\Phi_1=N$. Let $\tilde{m}_t = M_t - M_{t-1}$ be the indicator variable that equals one if the RWM algorithm makes a mistake on iteration $t$. Let $m_t(i)$  equal one if the $i$'th expert makes a mistake on iteration $t$ and zero otherwise. 
Inspecting the sum of the weights: 
\begin{align*}
\Phi_{t+1} & = \sum_i W_t(i) (1 - \eps m_t(i)) \\
& = \Phi_t (1 - \epsilon  \sum_i p_t(i) m_t(i)) & \mbox{ $p_t(i) = \frac{W_t(i)}{\sum_j W_t(j) }$} \\
& = \Phi_t ( 1 - \epsilon \E[\tilde{m}_t ]) \\
& \leq \Phi_t e^{-\eps \E[\tilde{m}_t] }. & \mbox{ $1  + x \leq e^x $}  
\end{align*}
On the other hand, by definition we have for any expert $i$ that
\[
W_T(i) = (1-\epsilon)^{M_T(i)} 
\]
Since the value of $W_T(i)$ is always less than the sum of all weights $\Phi_T$, we conclude that
\[
(1-\epsilon)^{M_T(i)} = W_T(i) \le \Phi_T \le N e^{-\eps \E [M_T]}.
\]
Taking the logarithm of both sides we get
\[
M_T(i)\log(1-\epsilon) \le \log{N} - \eps \E[ M_T] 
\]
Next, we use the approximation
\[
-x-x^2 \le \log{(1-x)} \le -x \qquad , \quad 0 < x <  \frac{1}{2}
\]
to obtain
\[
-M_T(i)(\epsilon+\epsilon^2) \le \log{N} - \eps \E[M_T] ,
\]
and the lemma follows.
\end{proof}

\subsection{Hedge}

The RWM algorithm is in fact more general: instead of considering a discrete number of mistakes, we can consider measuring the performance of an expert by a non-negative real number $\ell_t(i)$, which we refer to as the {\it loss} of the expert $i$ at iteration $t$. The randomized weighted majority algorithm guarantees that a decision maker following its advice will incur an average expected loss approaching that of the best expert in hindsight. 

Historically, this was observed by a different and closely related algorithm called Hedge. 

\begin{algorithm}[h!]
	\caption{Hedge}
	\label{alg:Hedge}
	\begin{algorithmic}[1]
		\STATE Initialize: $\forall i\in [N], \ W_1(i) = 1$
		\FOR {$t=1$ to $T$}
		\STATE Pick $i_t \sim_R W_t$, i.e., $i_t = i$ with probability $\x_t(i) = \frac{W_t(i) } {\sum_j W_t(j) }$
		\STATE Incur loss $\ell_t(i_t)$. 
		\STATE Update weights $ W_{t+1}(i) = W_{t}(i) e^{-\eps \ell_t(i)}$
		\ENDFOR
	\end{algorithmic}
\end{algorithm}

Henceforth, denote in vector notation the expected loss of the algorithm by
$$ \E [ \ell_t(i_t) ] = \sum_{i=1}^N \x_t(i) \ell_t(i) = \x_t^\top \ell_t  $$
\begin{theorem} \label{lem:hedge}
Let $\ell_t^2$ denote the $N$-dimensional vector of square losses, i.e., $\ell_t^2(i) = \ell_t(i)^2$,  let $\eps > 0$, and assume all losses to be non-negative.  
The Hedge algorithm satisfies for any expert $i^\star \in [N]$:
\[
  \sum_{t=1}^T  \x_t^\top \ell_t   \le \sum_{t=1}^T \ell_t(i^\star) + \epsilon \sum_{t=1}^T  \x_t^\top \ell_t^2   + \frac{\log N}{\epsilon} 
\]
\end{theorem}

\begin{proof}
	
As before, let $\Phi_t= \sum_{i=1}^N W_t(i)$ for all $t \in [T]$, and note that $\Phi_1=N$. 
	
Inspecting the sum of weights: 
\begin{eqnarray*}
	\Phi_{t+1} & = \sum_i W_t(i) e^{- \eps \ell_t(i)}  \\
	& = \Phi_t  \sum_i \x_t(i) e^{- \eps \ell_t(i)}  & \mbox{ $\x_t(i) = \frac{W_t(i)}{\sum_j W_t(j) }$} \\
	& \leq \Phi_t \sum_i \x_t(i) ( 1 - \epsilon \ell_t(i) + \eps^2 \ell_t(i)^2 ) )  & \mbox{ for $x \geq 0$, }    \\
	&   & \mbox{  $e^{-x} \leq 1 - x + x^2 $}   \\
	& = \Phi_t (  1 - \epsilon \x_t^\top \ell_t  + \eps^2 \x_t^\top \ell_t^2   )  \\
		& \leq \Phi_t e^{-\eps \x_t^\top \ell_t + \epsilon^2 \x_t^\top \ell_t^2  }. & \mbox{ $1  + x \leq e^x $}  
\end{eqnarray*}
On the other hand, by definition, for  expert $i^\star$ we have that
	\[
	W_T(i^\star) = e^{ -\epsilon \sum_{t=1}^T \ell_t(i^\star) } 
	\]
Since the value of $W_T(i^\star)$ is always less than the sum of all weights $\Phi_t$, we conclude that
	\[
	 W_T(i^\star) \le \Phi_T \le N e^{-\eps \sum_t  \x_t^\top \ell_{t} + \epsilon^2 \sum_{t} \x_t^\top \ell_t^2 }.
	\]
	Taking the logarithm of both sides we get
	\[
	-\epsilon \sum_{t=1}^T \ell_t(i^\star)  \le \log{N} - \eps \sum_{t=1}^T  \x_t^\top \ell_t + \epsilon^2 \sum_{t=1}^T \x_t^\top \ell_t^2
	\]
	and the theorem follows by simplifying.
\end{proof}

\section{The Regularization framework}

In the previous section we studied the multiplicative weights update method for decision making. A natural question is: couldn't we have used online gradient descent for the same exact purpose? 

Indeed, the setting of prediction from expert advice naturally follows into the framework of online convex optimization. To see this, consider the loss functions given by
$$ f_t(\x) = \ell_t^\top \x = \E_{i \sim \x}  [  \ell_t(i) ] ,  $$
which capture the expected loss of choosing an expert from distribution $\x \in \Delta_n$ as a linear function. 

The regret guarantees we have studied for OGD imply a regret of 
$$ O(GD \sqrt{T})  = O( \sqrt{nT}) . $$ 
Here we have used the fact that the Eucliean diameter of the simplex is two, and that the losses are bounded by one, hence the Euclidean norm of the gradient vector $\ell_t$ is bounded by $\sqrt{n}$. 

In contrast, the Hedge algorithm attains regret of $O(\sqrt{T \log n})$ for the same problem. How can we explain this discrepancy?! 

\subsection{The RFTL algorithm} 

Both OGD and Hedge are, in fact, instantiations of a more general meta-algorithm called RFTL (Regularized-Follow-The-Leader).

In an OCO setting of regret minimization, the most straightforward approach  for the online player is to use at any time the optimal decision (i.e., point in the convex set) in hindsight. Formally, let
$$\x_{t+1} = \argmin_{\x \in \K} \sum_{\tau=1}^{t} f_\tau(\x).$$
This flavor of strategy is known as ``fictitious play'' in economics, and has been named ``Follow the Leader'' (FTL) in machine learning. It is not hard to see that this simple strategy fails miserably in a worst-case sense. That is, this strategy's regret can be linear in the number of iterations, as the following example shows: Consider $\K = [-1,1]$, let $f_1(x) = \frac{1}{2} x $, and let $f_\tau$ for $\tau=2 , \ldots , T$ alternate between $- x $ or $x $. Thus, 
$$ \sum_{\tau=1}^t f_\tau(x)  = \mycases{ \frac{1}{2} x } {t \mbox{  is odd} } {-\frac{1}{2} x } {\text{otherwise}}  $$
The FTL strategy will keep shifting between $x_t = -1$ and $x_t = 1$, always making the wrong choice.

The intuitive FTL strategy fails in the example above because it is unstable. Can we modify the FTL strategy such that it won't change decisions often,  thereby causing it to attain low regret? 

This question motivates the need for a general means of stabilizing the FTL method. Such a means is referred to as ``regularization''. 

The generic RFTL meta-algorithm is defined in Algorithm \ref{alg:RFTLmain}. The regularization function ${R}$ is assumed to be strongly convex, smooth, and twice differentiable. 

\begin{algorithm}
	[h] \caption{Regularized Follow The Leader} \label{alg:RFTLmain} 
	\begin{algorithmic}[1]
		\STATE Input: $\eta > 0$, regularization function ${R}$, and a convex compact set $\K$.
		\STATE Let $\xv[1]  = \arg\min_{\x \in \K} {\left\{ {R}(\x)\right\} }$.
		\FOR{$t=1$ to $T$}
		\STATE Predict $\xv[t]$.
		\STATE Observe the payoff function $f_t$ and let $\nabla_t = \nabla f_t(\x_t) $.
		\STATE Update
		\begin{align*}
			\xv[t+1] = \argmin_{\x \in \K} {\left\{\eta\sum_{s=1}^t \nabla_s^\top \x + {R}(\x)\right\}}
		\end{align*}
		\ENDFOR
	\end{algorithmic}
\end{algorithm}

\subsection{Mirrored Descent}

An alternative view of this algorithm is in terms of iterative updates, which can be spelled out using the above definition directly. The resulting algorithm is called "Mirrored Descent".

OMD is an iterative algorithm that computes the current decision using a simple  gradient update rule and the previous decision, much like OGD. The generality of the method stems from the update being carried out in a ``dual'' space, where the duality notion is defined by the choice of regularization: the gradient of the regularization function defines a mapping from $\reals^n$ onto itself, which is a vector field. The gradient updates are then carried out in this vector field.

For the RFTL algorithm the intuition was straightforward---the regularization was used to ensure stability of the decision. For OMD, regularization has an additional purpose: regularization transforms the space in which gradient updates are performed. This transformation enables better bounds in terms of the geometry of the space.

The OMD algorithm comes in two flavors:  an agile and a lazy version. The lazy version keeps track of a point in Euclidean space and projects onto the convex decision set $\K$ only at decision time. In contrast, the agile version maintains a feasible point at all times, much like OGD.

\begin{algorithm}
	[H] \caption{Online Mirrored Descent} \label{alg:flpl}
	\begin{algorithmic}
		[1] \STATE Input: parameter $\eta > 0$, regularization function ${R}(\x)$.
		\STATE Let $\yv[1]$ be such that $\nabla {R}(\yv[1]) = \bzero$ and 	$\xv[1] = \arg\min_{\x \in \K} B_{R}(\x||\yv[1])$.
		\FOR{$t=1$ to $T$}
				\STATE Play $\xv[t]$.
				\STATE Observe the payoff function $f_t$ and let $\nabla_t = \nabla f_t(\x_t) $.
		\STATE Update $\y_t$ according to the rule:
		\begin{align*}
			&\text{[Lazy version]} 
			&\nabla {R}(\yv[t+1]) = \nabla {R}(\yv[t]) - \eta\, \nabla_{t}\\
			&\text{[Agile version]}
			&\nabla {R}(\yv[t+1]) = \nabla {R}(\xv[t]) - \eta\, \nabla_{t}
		\end{align*}
Project according to $B_{R}$:
		$$\xv[t+1] = \argmin_{\x \in \K} B_{R}(\x||\yv[t+1])$$
		\ENDFOR
	\end{algorithmic}
\end{algorithm}

A myriad of questions arise, but first, let us see how does this algorithm give rise to both OGD.

We note that there are other important special cases of the RFTL meta-algorithm: those are derived with matrix-norm regularization---namely, the von Neumann entropy function, and the log-determinant function, as well as  self-concordant barrier regularization.  Perhaps most importantly for optimization, also the AdaGrad algorithm is obtained via changing regularization---which we shall explore in detail in the next chapter. 

\subsection{Deriving online gradient descent}

To derive the online gradient descent algorithm, we take ${R}(\x) = \frac{1}{2} \|\x - \x_0\|_2^2$ for an arbitrary $\x_0 \in \K$. Projection with respect to  this divergence is the standard Euclidean projection (left as an exercise), and in addition, $\nabla {R}(\x) = \x - \x_0$. Hence, the update rule for the OMD Algorithm \ref{alg:flpl} becomes:
\begin{align*}
	& \xv = \proj_\K (  \yv)  , \ \yv =  \yv[t-1]  - \eta \nabla_{t-1}  & \mbox{lazy version}  \\
	& \xv = \proj_\K (  \yv)  , \ \yv =  \xv[t-1]  - \eta \nabla_{t-1}  & \mbox{agile version} 
\end{align*}

The latter algorithm is exactly online gradient descent, as described in Algorithm \ref{figure:ogd} in Chapter \ref{chap:first order}. Furthermore, both variants are identical for the case in which $\K$ is the unit ball. 

We later prove general regret bounds that will imply a $O(GD \sqrt{T})$ regret for OGD as a special case of mirrored descent.

\subsection{Deriving multiplicative updates} 

Let   ${R}(\xv[]) =  \xv[] \log \xv[] = \sum_i \x_i \log \x_i$ be the negative entropy function, where $\log \x$ is to be interpreted elementwise. Then $\nabla {R}(\x) = \bone + \log \x$, and hence the update rules for the OMD algorithm become:
\begin{align*}
	& \xv = \argmin_{\x \in \K} B_{R}(\x ||\yv)    , \ \log \yv =  \log \yv[t-1]  - \eta \nabla_{t-1}  & \mbox{lazy version}  \\
	& \xv = \argmin_{\x \in \K} B_{R}(\x ||\yv)    , \ \log \yv =  \log \xv[t-1]  - \eta \nabla_{t-1}  & \mbox{agile version} 
\end{align*}

With this choice of regularizer, a notable special case is the experts problem we encountered in \S \ref{sec:experts}, for which the decision set $\K$ is the $n$-dimensional simplex $ \Delta_n = \{ \x \in \reals^n_+ \ | \ \sum_i \x_i =  1  \}$.
In this special case, the projection according to the negative entropy becomes scaling by the $\ell_1$ norm (left as an exercise), which implies that both update rules amount to the same algorithm: 
$$ \x_{t+1}(i) = \frac{\x_t(i) \cdot e^{-\eta \nabla_t(i)}}{\sum_{j=1}^n \x_t(j) \cdot e^{-\eta \nabla_t(j)} }, $$
which is exactly the Hedge algorithm!  The general theorem we shall prove henceforth recovers the $O(\sqrt{T \log n })$ bound for prediction from expert advice for this algorithm.

\section{Technical background: regularization functions}
\sectionmark{Technical background}

In the rest of this chapter we analyze the mirrored descent algorithm. For this purpose, consider  regularization functions, denoted $R : \K \mapsto \reals $, which are strongly convex and smooth (recall definitions in \S \ref{sec:optdefs}). 

Although it is not strictly necessary, we assume that the regularization functions in this chapter are twice differentiable over $\K$  and, for all points $\x \in \text{int}(\K)$ in the interior of the decision set, have a Hessian $\nabla^2 R(\x)$ that is, by the strong convexity of $R$, positive definite.

We denote the diameter of the set $\K$ relative to the function $R$ as 
$$ D_R = \sqrt{ \max_{\x,\y \in \K} \{ R(\x) - R(\y) \}}$$ 

Henceforth we make use of general norms and their dual. The dual norm to a norm $\| \cdot \|$ is given by the following definition:
$$ \| \y \|^* \equaltri \max_{ \| \x \| \leq 1 }  \langle \x, \y \rangle $$
A positive definite matrix $A$ gives rise to the matrix norm $\|\x\|_A = \sqrt{\x^\top A \x}$. 
The  dual norm of a matrix norm is $\|\x\|_A^*=\|\x\|_{A^{-1}}$. 

The generalized Cauchy-Schwarz theorem asserts $ \langle \x , \y \rangle \leq \| \x \| \| \y \|^*$ and in particular for matrix norms, $\langle \x , \y \rangle \leq \|\x\|_A \| \y\|_A^*$. 

In our derivations, we usually consider matrix norms with respect to $\nabla^2R(\x)$, the Hessian of the regularization function $R(\x)$.
In such cases, we use  the notation 
$$\|\x\|_\y \equaltri \|\x\|_{\nabla^2 {R}(\y)}$$
and similarly 
$$\|\x\|_\y^* \equaltri \|\x\|_{\nabla^{-2} {R}(\y)}$$

A crucial quantity in the analysis with regularization is the remainder term of the Taylor approximation of the regularization function, and especially the remainder term of the first order Taylor approximation. The difference between the value of the regularization function at $\x$ and the value of the first order Taylor approximation is known as the Bregman divergence, given by 
\begin{definition}
	Denote by $B_{R}(\x||\y)$ the
	Bregman divergence with respect to the function ${R}$, defined as
	$$ B_{R}(\x||\y) = {R}(\x) - {R}(\y) - \nabla {R}(\y)^\top  (\x-\y)   $$
\end{definition}

For twice differentiable functions, Taylor expansion and the mean-value theorem assert that the Bregman divergence is equal to the second derivative at an intermediate point, i.e., (see exercises)  
$$ B_{R}(\x||\y) = \frac{1}{2} \|\x - \y\|_\z^2, $$ 
for some point $\z \in [\x,\y]$, meaning there exists some $\alpha \in [0,1]$ such that $\z = \alpha \x + (1-\alpha) \y$. 
Therefore, the Bregman divergence defines a local norm, which has a dual norm. We shall denote this dual norm by 
$$ \| \cdot \|_{\x,\y}^*  \equaltri \| \cdot \|_\z^*.$$ 
With this notation we have
$$ B_{R}(\x||\y) = \frac{1}{2} \|\x - \y\|_{\x,\y} ^2. $$ 
In online convex optimization, we commonly refer to the Bregman divergence between two consecutive decision points $\x_t$ and $\x_{t+1}$. In such cases, we shorthand notation for the norm  defined by the Bregman divergence with respect to  ${R}$ on the intermediate point in $[\x_t,\x_{t+1}]$ as $\| \cdot \|_t \equaltri \| \cdot \|_{\x_t,\x_{t+1}} $. The latter norm is called the local norm at iteration $t$. With this notation, we have $B_{R}(\x_t||\x_{t+1}) = \frac{1}{2} \|\x_t - \x_{t+1}\|_t^2 $. 

Finally, we consider below generalized projections that use the Bregman divergence as a distance instead of a norm. Formally, the projection of a point $\y$ according to the Bregman divergence with respect to  function $R$ is given by
$$\argmin_{\x \in \K} B_{R}(\x||\y)$$

\section{Regret bounds for Mirrored Descent}

In this subsection we prove regret bounds for the agile version of the RFTL algorithm. The analysis is quite different than the one for the lazy version, and of independent interest. 

\begin{theorem} \label{thm:mirrordescent}
The RFTL  Algorithm \ref{alg:flpl} attains for every $\uv \in \K$ the following bound on the regret:
$$  \regret_T \le    2   \eta  \sum_{t=1}^T \| \nabla_t \|_t^{* 2} + \frac{R(\uv) - R(\x_1)}{\eta }  . $$ 
\end{theorem}
If an upper bound on the local norms is known, i.e. $\| \nabla_t\|_t^* \leq G_R$ for all times $t$, then we  can further optimize over the choice of $\eta$ to obtain
$$ \regret_T \leq  2  D_R G_R \sqrt{ 2T  } .$$

\begin{proof}
Since  the functions $\fv$ are convex, for any $\x^* \in K$,
$$ \fv(\x_t) - \fv(\x^*) \leq \nabla \fv(\x_t)^\top (\x_t - \x^*)  .$$
The following property of Bregman divergences follows easily from the definition: for any vectors $\x,\y,\z$,
$$ (\x - \y)^\top (\nabla \R(\z) - \nabla \R(\y)) = B_\R(\x,\y)-B_\R(\x,\z) +
B_\R(\y,\z). $$
Combining both observations,
\begin{align*}
2(\fv(\x_t) - \fv(\x^*)) & \leq 2\nabla \fv(\x_t)^\top (\x_t - \x^*)  \\
& =   \frac{1}{\eta}  (\nabla \R(\y_{t+1}) - \nabla \R(\x_{t}))^\top(\x^* - \x_t) \\
& =  \frac{1}{\eta} [B_\R(\x^*,\xv)-B_\R(\x^*,\y_{t+1}) + B_\R(\x_t,\y_{t+1})]   \\
& \leq  \frac{1}{\eta} [B_\R(\x^*,\xv)-B_\R(\x^*,\x_{t+1}) +
B_\R(\x_t,\y_{t+1})]
\end{align*}
where the last inequality follows from the generalized Pythagorean inequality (see \cite{CesaBianchiLugosi06book} Lemma 11.3),  as $\x_{t+1}$ is the projection w.r.t the Bregman divergence of $\y_{t+1}$ and $\x^* \in K$ is in the convex set. Summing over all iterations,
\begin{eqnarray} \label{eq:general1}
2\regret & \leq & \frac{1}{\eta} [ B_\R(\x^*,\x_1) -  B_\R(\x^*,\x_T) ] + \sum_{t=1}^T \frac{1}{\eta} B_\R(\x_t,\y_{t+1}) \notag \\
& \leq & \frac{1}{\eta} D^2  + \sum_{t=1}^T \frac{1}{\eta} B_\R(\x_t,\y_{t+1})
\end{eqnarray}

We proceed to bound $B_\R(\x_t,\y_{t+1})$. By definition of Bregman divergence, and the generalized Cauchy-Schwartz inequality,
\begin{align*}
 B_\R(\x_t,\y_{t+1}) + B_\R(\y_{t+1},\x_t) &= (\nabla \R(\x_t) - \nabla \R(\y_{t+1}))^\top (\x_t - \y_{t+1}) \\
 &=  \eta \nabla \fv(\x_t)^\top(\x_t - \y_{t+1}) \\
 & \leq \eta \| \nabla \fv(\x_t) \|^* \| \x_t - \y_{t+1} \| \\
 &\leq  \frac{1}{2} \eta^2 G_*^{ 2} + \frac{1}{2} \|\x_t - \y_{t+1}\|^2.
\end{align*}
where in the last inequality follows from $(a-b)^2 \geq 0$. 
Thus, by our assumption $B_\R(\x,\y) \geq \frac{1}{2} \|\x-\y\|^2$, we have
$$ B_\R(\x_t,\y_{t+1}) \leq \frac{1}{2} \eta^2 G_*^2 + \frac{1}{2} \|\x_t -
\y_{t+1}\|^2 - B_\R(\y_{t+1},\x_t) \leq \frac{1}{2} \eta^2 G^2_*. $$

Plugging back into Equation \eqref{eq:general1}, and by non-negativity of the Bregman divergence, we get
$$ \regret \leq  \frac{1}{2} [\frac{1}{\eta} D^2 + \frac{1}{2} \eta T G_{*}^{2} ] \leq  D G_* \sqrt{T} \ ,  $$
by taking $\eta = \frac{D}{2 \sqrt{T} G_*}$

\end{proof}

\newpage
\section{Exercises}

\begin{enumerate}
	\item \label{exercise:dualnorm}
	\begin{enumerate}
	\item
	Show that the dual norm to a matrix norm given by $A \succ 0$ corresponds to the matrix norm of $A^{-1}$.
	\item
	Prove the generalized Cauchy-Schwarz inequality for any norm, i.e., 
	$$ \langle \x , \y \rangle \leq \|\x \| \|\y \|^*$$ 
	\end{enumerate}

	\item
	Prove that the Bregman divergence is equal to the local norm at an intermediate point, that is:
	$$ B_{R}(\x||\y) = \frac{1}{2} \|\x - \y\|_\z^2, $$ 
	where  $\z \in [\x,\y]$ and the interval $[\x,\by]$ is defined as 
	$$ [\bx,\by] = \{ \vv  = \alpha \bx + (1-\alpha) \by \ , \ \alpha \in [0,1] \}$$
	
	\item \label{exercise:bregman-Euclid}
	Let ${R}(\x) = \frac{1}{2} \|\x - \x_0\|^2$ be the (shifted) Euclidean regularization function. Prove that the corresponding Bregman divergence is the Euclidean metric. Conclude that projections with respect to  this divergence are standard Euclidean projections.
	
	\item \label{exercise:equiv-lazy-agile}
	 Prove that both agile and lazy versions of the OMD meta-algorithm are equivalent in the case that the regularization is Euclidean and the decision set is the Euclidean ball. 
	
\item \label{exercise:bregman-entropy}
For this problem the decision set is the $n$-dimensional simplex.  Let ${R}(\x) = \x \log \x $ be the negative entropy regularization function. Prove that the corresponding Bregman divergence is the relative entropy, and prove that the diameter $D_R$ of the $n$-dimensional simplex with respect to  this function is bounded by $\log n$. Show that projections with respect to  this divergence over the simplex amounts to scaling by the $\ell_1$ norm.

\item $^*$   A  set $\K \subseteq \reals^d$ is symmetric if $\x \in \K$ implies $-\x \in \K$. Symmetric sets gives rise to a natural definition of a norm. 
Define the function $\| \cdot \|_\K : \reals^d \mapsto \reals$ as
$$ \| \x \|_\K = \arg \min_{\alpha > 0 }  \left \{ \frac{1}{\alpha} \x \in \K \right \} $$ 
Prove that $\| \cdot \|_\K$ is a norm if and only if $\K$ is convex.

\end{enumerate}

\newpage
\section{Bibliographic Remarks}

Regularization in the context of online learning was first studied in \cite{GroveLS01} and \cite{KivinenW01}. The influential paper of Kalai and Vempala \cite{KV-FTL} coined the term ``follow-the-leader'' and introduced many of the techniques that followed in OCO. The latter paper studies random perturbation as a regularization and analyzes the follow-the-perturbed-leader algorithm, following an early development by \cite{Hannan57} that was overlooked in learning for many years.

In the context of OCO, the term follow-the-regularized-leader was coined in \cite{ShwartzS07,ShalevThesis}, and at roughly the same time an essentially identical algorithm was called ``RFTL'' in \cite{AbernethyHR08}. The equivalence of RFTL and Online Mirrored Descent was observed by  \cite{DBLP:conf/colt/HazanK08}.

\chapter{Adaptive Regularization}\label{sec:adagrad}

In the previous chapter we have studied a geometric extension of online / stochastic / determinisitic gradient descent. The technique to achieve it is called regularization, and we have seen how for the problem of prediction from expert advice, it can potentially given exponential improvements in the dependence on the dimension. 

A natural question that arises is whether we can automatically learn the optimal regularization, i.e. best algorithm from the mirrored-descent class, for the problem at hand?  

The answer is positive in a strong sense: it is theoretically possible to learn the optimal regularization online and in a data-specific way. Not only that, the resulting algorithms exhibit the most significant speedups in training deep neural networks from all accelerations studied thus far.

\section{Adaptive Learning Rates: Intuition}

The intuition for adaptive regularization is simple: consider an optimization problem which is axis-aligned, in which each coordinate is independent of the rest. It is reasonable to fine tune the learning rate for each coordinate separately - to achieve optimal convergence in that particular subspace of the problem, independently of the rest. 

Thus, it is reasonable to change the SGD update rule from $\x_{t+1} \leftarrow \x_t - \eta \nabla_t$, to the more robust
$$ \x_{t+1} \leftarrow \x_t - D_t \nabla_t , $$
where $D_t$ is a diagonal matrix that contains in coordinate $(i,i)$ the learning rate for coordinate $i$ in the gradient.  Recall from the previous sections that the optimal learning rate for stochastic non-convex optimization is of the order $O(\frac{1}{\sqrt{t}})$. More precisely, in Theorem \ref{thm:non-convex-sgd}, we have seen that this learning rate should be on the order of $O(\frac{1}{\sqrt{t \sigma^2}})$, where $\sigma^2$ is the variance of the stochastic gradients. The empirical estimator of the latter is  
$\sum_{i < t} \|\nabla_i\|^2 $. 

Thus, the robust version of stochastic gradient descent for smooth non-convex optimization should behave as the above equation, with 
$$ D_t(i,i) = \frac{1}{\sqrt{\sum_{i < t} \nabla_t(i)^2}} .$$
This is exactly the diagonal version of the AdaGrad algorithm!  We continue to rigorously derive it and prove its performance guarantee.

\section{A Regularization Viewpoint  } 

In the previous chapter we have introduced regularization as a general methodology for deriving online convex optimization algorithms. 
Theorem \ref{thm:mirrordescent} bounds the regret of the Mirrored Descent  algorithm for any strongly convex regularizer as 
$$ \regret_T \leq  \max_{\uv \in \K} \sqrt{ 2 \sum_t \|\nabla_t \|_t^{* 2}  B_{R}( \uv||\x_1) }. $$
In addition, we have seen how to derive the online gradient descent and the multiplicative weights algorithms as special cases of the RFTL methodology. 

We consider the following question: thus far we have thought of $R$ as a strongly convex function. But which strongly convex function should we choose to minimize regret? 
This is a deep and difficult question which has been considered in the optimization literature since its early developments. 

The ML approach is to learn the optimal regularization online. 
That is, a regularizer that adapts to the sequence of cost functions and is in a sense the ``optimal'' regularization to use in hindsight. We formalize this in the next section.

\section{Tools from Matrix Calculus}

Many of the inequalities that we are familiar with for positive real numbers hold for positive semi-definite matrices as well. We henceforth need the following inequality, which is left as an exercise,  
\begin{proposition} \label{proposition:psd-shalom}
For positive definite matrices $ A \succcurlyeq B \succ 0$:
$$ 2 \trace( ({A - B})^{1/2}  )  + \trace( A^{-1/2}B)  \leq 2 \trace( {A}^{1/2} ).  $$
\end{proposition}

Next, we require a structural result which explicitly gives the optimal regularization as a function of the gradients of the cost functions. 	For a proof see the exercises. 
\begin{proposition}
\label{proposition:solution-inv-trace}
Let $A \succcurlyeq 0$. The minimizer of the  following minimization problem:
\begin{align*}
\min_X \ \  &  \trace(X^{-1}A)  \\
\mbox{subject to  }   & X \succcurlyeq 0 \\
&  \trace(X) \le 1 ,
\end{align*}
is $X = {A^{1/2}}/{ \trace(A^{1/2})}$, and the minimum objective value is  $\trace^2(A^{1/2})$.
\end{proposition}

\section{The AdaGrad Algorithm and Its Analysis}

To be more formal, let us consider the set of all strongly convex regularization functions with a fixed and bounded Hessian in the set 
$$\forall \x \in \K \ . \ \nabla^2 R(\x) = \nabla^2 \in \H \equaltri \{ X \in \reals^{n \times n} \ ; \ \trace(X) \leq 1 \ , \ X \succcurlyeq 0 \}$$ 

The set $\H$ is a restricted class of regularization functions (which does not include the entropic regularization). However, it is a general enough class to capture online gradient descent along with any rotation of the Euclidean regularization.  

\begin{algorithm}
	[H] \caption{AdaGrad (Full Matrix version) } \label{alg:adagrad}
	\begin{algorithmic}
		[1] \STATE Input: parameters $\eta, \x_1 \in \K$.
		\STATE Initialize: $S_0 = G_0 = \bzero $, 
		\FOR{$t=1$ to $T$}
		\STATE Predict $\x_t$, suffer loss $f_t(\x_t)$.
		\STATE Update: 
		$$S_t = S_{t-1} + \nabla_t \nabla_t^\top, \ G_t = {S_t}^{1/2}$$
		$$ \yv[t+1] = \xv - \eta G_t^{-1} \nabla_t $$ 
		$$ \xv[t+1] = \argmin_{\x \in \K} \| \yv[t+1]  - \x\|^2_{G_t} $$ 
				\ENDFOR
	\end{algorithmic}
\end{algorithm}

The problem of learning the optimal regularization has given rise to Algorithm \ref{alg:adagrad}, known as the AdaGrad (Adaptive subGradient method) algorithm. In the algorithm definition and throughout this chapter, the notation $A^{-1}$ refers to the Moore-Penrose pseudoinverse of the matrix $A$. 
Perhaps surprisingly, the regret of AdaGrad is at most a constant factor larger than the minimum regret of all RFTL algorithm with regularization functions whose Hessian is fixed and belongs to the class $\H$. The regret bound on AdaGrad is formally stated in the following theorem.  
\begin{theorem}
\label{theorem:adagrad-main}
Let	$\{\x_t\}$ be defined by Algorithm~\ref{alg:adagrad} with parameters $ \eta  = {D}$, where 
$$D = \max_{\uv \in \K}  \|\uv - \x_1\|_2 .$$  
Then for any $\x^\star \in \K$,
\begin{equation*}
\regret_{T}(\mbox{AdaGrad}) \le 2 D    \sqrt{  \min_{H \in \H} \sum_t \|\nabla_t \|_H^{* 2}  } .
\end{equation*}
\end{theorem}

Before proving this theorem, notice that it delivers on one of the promised accounts: comparing to the bound of Theorem \ref{thm:mirrordescent} and ignoring the diameter $D$ and dimensionality, the regret bound is as good as the regret of RFTL for the class of regularization functions. 

We proceed to prove Theorem \ref{theorem:adagrad-main}. 
First, a direct corollary of Proposition \ref{proposition:solution-inv-trace}  is that 
\begin{corollary}
\begin{eqnarray*}
\sqrt{  \min_{H \in \H} \sum_t \|\nabla_t \|_H^{* 2} }  & = \sqrt{ \min_{H \in \H } \trace (  H^{-1} \sum_t \nabla_t \nabla_t^\top  ) }  \\
& =  \trace{ \sqrt{ \sum_t \nabla_t \nabla_t^\top } } = \trace(G_T)  
\end{eqnarray*}
\end{corollary}

Hence, to prove Theorem \ref{theorem:adagrad-main}, it suffices to prove the following lemma. 
\begin{lemma}
$$ \regret_T(\text{AdaGrad}) \leq 2  D \trace(G_T) = 2D  \sqrt{  \min_{H \in \H} \sum_t \|\nabla_t \|_H^{* 2} }  .$$
\end{lemma}
\begin{proof}
By the
definition of $\by_{t+1}$:
\begin{equation} \label{eq:update-rule-adagrad}
\by_{t+1} - \bx^\star = \bx_{t} - \bx^\star - \eta {G_t}^{-1}
\nabla_t,
\end{equation}
and
\begin{equation} \label{eq:A_t-multiply-adagrad}
G_t (\by_{t+1} - \bx^\star) = G_t (\bx_t - \bx^\star) - \eta 
\nabla_t.
\end{equation}
Multiplying the transpose of \eqref{eq:update-rule-adagrad} by
\eqref{eq:A_t-multiply-adagrad} we get
\begin{gather}
(\by_{t+1} - \bx^\star)^\top G_t(\by_{t+1} - \bx^\star) = \notag \\
(\bx_t\! -\! \bx^\star)^\top G_t(\bx_t\! -\! \bx^\star) -
2 \eta  \nabla_t^\top (\bx_t\! -\! \bx^\star) +
\eta^2 \nabla_t^\top G_t^{-1} \nabla_t.
\label{eq:multiplied-adagrad}
\end{gather}
Since $\bx_{t+1}$ is the projection of $\by_{t+1}$ in the norm induced by
$G_t$, we have (see \S \ref{sec:projections})
\begin{align*}
 (\by_{t+1} - \bx^\star)^\top G_t(\by_{t+1} - \bx^\star)  & = \| \by_{t+1} - \bx^\star \|_{G_t}^2  \ge  \| \bx_{t+1} - \bx^\star \|_{G_t}^2   .
\end{align*}
This inequality is the reason for using generalized projections as
opposed to standard projections, which were used in the analysis
of \ogd (see \S \ref{section:ogd} Equation
\eqref{eqn:ogdtriangle}). This fact together with
\eqref{eq:multiplied-adagrad} gives
\begin{align*}
\nabla_t^\top (\bx_t \! -\! \bx^\star) &\leq \ \frac{\eta}{2}
\nabla_t^\top G_t^{-1} \nabla_t  +  \frac{1}{2 \eta} \left( \| \bx_{t} - \bx^\star \|_{G_t}^2 - \| \bx_{t+1} - \bx^\star \|_{G_{t}}^2  \right) .
\end{align*}
Now, summing up over $t=1$ to $T$ we get that
\begin{align}  \label{eqn:adagrad-shalom}
&\sum_{t=1}^T \nabla_t^\top (\bx_t - \bx^\star)
 \leq \frac{\eta}{2} \sum_{t=1}^T \nabla_t^\top G_t^{-1} \nabla_t +
\frac{1}{2\eta}  \| \bx_{1} - \bx^\star \|_{G_{0}}^2  \\
& + \frac{1}{2 \eta} \sum_{t=1}^T  \left( \| \bx_{t} - \bx^\star \|_{G_t}^2 - \| \bx_{t} - \bx^\star \|_{G_{t-1}}^2 \right)   - \frac{1}{2 \eta} \| \bx_{T+1}   - \bx^\star \|_{G_{T}}^2  \notag \\
&\leq \frac{\eta}{2} \sum_{t=1}^T \nabla_t^\top G_t^{-1}
\nabla_t + \frac{1}{2\eta} \sum_{t=1}^{T} (\bx_t\! -\! \bx^\star)^\top
(G_t - G_{t-1} ) (\bx_t\! -\! \bx^\star) . \notag
\end{align}
In the last inequality we use the fact that $G_0 = \bzero $. 
We proceed to bound each of the terms above separately. 

\begin{lemma}
\label{lemma:trace-bound-adagrad}
With  $S_t,G_t$  as defined in
Algorithm \ref{alg:adagrad},
\begin{equation*}
\sum_{t=1}^T \nabla_t^\top  G_t^{-1} \nabla_t \le 2 \sum_{t=1}^T \nabla_t^\top G_T^{-1} \nabla_t \leq  2\trace(G_T). 
\end{equation*}
\end{lemma}
\begin{proof}
We prove the lemma by induction. The base case follows since 
\begin{align*}
\nabla_1^\top G_1^{-1}  \nabla_1  & = \trace( G_1^{-1}  \nabla_1 \nabla_1^\top) \\
& =   \trace  (G_1^{-1}  G_1^2 ) \\
& =  \trace(G_1). 
\end{align*}

Assuming the lemma holds for $T - 1$, we get by  the inductive hypothesis 
\begin{align*}
\sum_{t=1}^T \nabla_t^\top  G_t^{-1} \nabla_t & \le 2 \trace(G_{T-1} )  + \nabla_T^\top  G_T^{-1} \nabla_T \\
& = 2 \trace(   ({G_{T}^2 - \nabla_T \nabla_T^\top})^{1/2}  )  + \trace( G_T^{-1}  \nabla_T \nabla_T^\top) \\
& \leq 2 \trace( G_{T} ).  
\end{align*}
Here, the last inequality is due to the matrix inequality \ref{proposition:psd-shalom}. 
\end{proof}

\begin{lemma}
\label{lemma:opt-reg-bound2-adagrad}
\begin{equation*}
\sum_{t=1}^{T} (\bx_t\! -\! \bx^\star)^\top (G_t - G_{t-1} ) (\bx_t\! -\! \bx^\star) \leq D^2 \trace(G_T). 
\end{equation*}
\end{lemma}
\begin{proof}
By definition $S_t \succcurlyeq S_{t-1}$, and hence $G_t \succcurlyeq G_{t-1}$. Thus, 
\begin{align*}
& \sum_{t=1}^{T} (\bx_t\! -\! \bx^\star)^\top (G_t - G_{t-1} ) (\bx_t\! -\! \bx^\star) \\
& \leq \sum_{t=1}^{T} D^2  \lambda_{\max}( G_t - G_{t-1} ) \\
& \leq D^2 \sum_{t=1}^{T}  \trace (G_t - G_{t-1})  & A \succcurlyeq 0 \ \Rightarrow \  \lambda_{\max}(A) \leq \trace(A) \\
& = D^2 \sum_{t=1}^{T}  (\trace (G_t ) - \trace( G_{t-1}))  &  \mbox{ linearity of the trace} \\
& \leq D^2 \trace(G_T). 
\end{align*}
\end{proof}

Plugging both lemmas into Equation \eqref{eqn:adagrad-shalom},  we obtain
\begin{align*}
&\sum_{t=1}^T \nabla_t^\top (\bx_t - \bx^\star)  \leq \ {\eta} \trace(G_T) + \frac{1}{2\eta} D^2 \trace(G_T)   \leq 2 D \trace(G_T).  
\end{align*} 
\end{proof}

\section{Diagonal AdaGrad}

The AdaGrad algorithm maintains potentially dense matrices, and requires the computation of the square root of these matrices. This is usually prohibitive in machine learning applications in which the dimension is very large.  Fortunately, the same ideas can be applied with almost no computational overhead on top of vanilla SGD, using the diagonal version of AdaGrad given by:
\begin{algorithm}
	[H] \caption{AdaGrad (diagonal version) } \label{alg:diag-adagrad}
	\begin{algorithmic}
		[1] \STATE Input: parameters $\eta, \x_1 \in \K$.
		\STATE Initialize: $S_0 = G_0 = \bzero $, 
		\FOR{$t=1$ to $T$}
		\STATE Predict $\x_t$, suffer loss $f_t(\x_t)$.
		\STATE Update: 
		$$S_t = S_{t-1} + \diag(\nabla_t \nabla_t^\top), \ G_t = {S_t}^{1/2}$$
		$$ \yv[t+1] = \xv - \eta G_t^{-1} \nabla_t $$ 
		$$ \xv[t+1] = \argmin_{\x \in \K} \| \yv[t+1]  - \x\|^2_{G_t} $$ 
				\ENDFOR
	\end{algorithmic}
\end{algorithm}

In contrast to the full-matrix version, this version can be implemented in linear time and space, since diagonal matrices can be manipulated as vectors. Thus, memory overhead is only a single $d$-dimensional vector, which is used to represent the diagonal preconditioning (regularization) matrix, and the computational overhead is a few vector manipulations per iteration. 

Very similar to the full matrix case, the diagonal AdaGrad algorithm can be analyzed and the following performance bound obtained:

\begin{theorem}
\label{theorem:diagonal-adagrad-main}
Let	$\{\x_t\}$ be defined by Algorithm~\ref{alg:diag-adagrad} with parameters $ \eta  = {D_\infty}$, where 
$$D_\infty = \max_{\uv \in \K}  \|\uv - \x_1\|_\infty ,$$  
and let $\diag(\H)$ be the set of all diagonal matrices in $\H$. Then for any $\x^\star \in \K$,
\begin{equation*} 
\regret_{T}(\mbox{D-AdaGrad}) \le 2 D_\infty    \sqrt{  \min_{H \in \diag(\H)} \sum_t \|\nabla_t \|_H^{* 2}  } .
\end{equation*}
\end{theorem}

\section{State-of-the-art: from Adam to Shampoo and beyond} 

Since the introduction of the adaptive regularization technique in the context of regret minimization, several improvements were introduced that now compose state-of-the-art. A few notable advancements include: 

\begin{itemize}

		
\item[\bf{AdaDelta}:] The algorithm keeps an exponential average of past gradients and uses that in the update step.
		
\item[\bf{Adam}:] Adds a sliding window to AdaGrad, as well as adding a form of momentum via estimating the second moments of past gradients and adjusting the update accordingly. 
		
\item[\bf{Shampoo}:] Interpolates between full-matrix and diagonal adagrad in the context of deep neural networks: use of the special layer structure to reduce memory constraints.  
		
\item[\bf{AdaFactor}:] Suggests a Shampoo-like approach to reduce memory footprint even further, to allow the training of huge models.  
		
\item[\bf{GGT}:] While full-matrix AdaGrad is computationally slow due to the cost of manipulating matrices, this algorithm uses recent gradients (a thin matrix $G$), and via linear algebraic manipulations reduces computation by never computing $GG^\top$, but rather only $G^\top G$, which is low dimensional. 
		
\item[\bf{SM3 , ET}:] Diagonal AdaGrad requires an extra $O(n)$ memory to store $\text{diag}(G_t)$. These algorithms, inspired by AdaFactor, approximate $G_t$ as a low rank tensor to save memory and computation.


\end{itemize}

\newpage
\section{Exercises}

\begin{enumerate}

\item $^*$
Prove that for positive definite matrices $A \succcurlyeq B \succ 0$ it holds that
\begin{enumerate}
\item
$A^{1/2} \succcurlyeq B^{1/2} $
\item
$ 2 \trace( ({A - B})^{1/2}  )  + \trace( A^{-1/2}B)  \leq 2 \trace( {A}^{1/2} ).  $
\end{enumerate}

\item $^*$ 
Consider the  following minimization problem where $A \succ 0$:
\begin{align*}
 \min_X \ \ &  \trace(X^{-1}A) \\
  \text{subject to  } \ \  & X \succ 0 \\
 &  \trace(X) \le 1.
\end{align*}
Prove that its minimizer is given by $X = A^{1/2} / \trace(A^{1/2})$, and the minimum is obtained at $\trace^2(A^{1/2})$.

\end{enumerate}

\newpage
\section{Bibliographic Remarks}

The AdaGrad algorithm was introduced in \cite{DuchiHS10,duchi2011adaptive},  its diagonal  version was also discovered in parallel in \cite{McMahanS10}.  Adam \cite{kingma2014adam} and RMSprop \cite{hinton2012neural} are widely used methods based on adaptive regularization.  A cleaner analysis was recently proposed in \cite{gupta2017unified}, see also \cite{deng2018optimal}. 

Adaptive regularization has received much attention recently, see e.g., \cite{OrabonaC10,ward2018adagrad}.  Newer algorithmic developments on adaptive regularization include Shampoo \cite{gupta2018shampoo},  GGT \cite{agarwal2018case}, AdaFactor \cite{shazeer2018adafactor}, Extreme Tensoring \cite{chenET} and SM3 \cite{anil2019memory}.

\chapter{Variance Reduction}  \label{chapter:variancereduction}

In the previous chapter we have studied the first of our three acceleration techniques over SGD, adaptive regularization, which is a geometric tool for acceleration. 
In this chapter we introduce the second first-order acceleration technique, called variance reduction. This technique is probabilistic in nature, and applies to more restricted settings of mathematical optimization in which the objective function has a finite-sum structure. Namely, we consider optimization problems of the form
\begin{equation}
\min_{\x \in \K} f(\x) \ , \  f(\x) = \frac{1}{m} \sum_{i=1}^m f_i(\x) \ .  \label{eqn:ERM}
\end{equation}

Such optimization problems are canonical in training of ML models, convex and non-convex. However, in the context of machine learning we should remember that the ultimate goal is generalization rather than training.

\section{Variance reduction: Intuition}

The intuition for variance reduction is simple, and comes from trying to improve the naive convergence bounds for SGD that we have covered in the first lesson. 

Recall  the SGD  update rule $\x_{t+1} \leftarrow \x_t - \eta \hat{\nabla}_t$, in which $\hat{\nabla}_t$ is an unbiased estimator for the gradient such that
$$  \E[\hat{\nabla}_t] = \nabla_t \ ,  \ \E[\| \hat{\nabla}_t \|_2^2] \leq \sigma^2 . $$ 
We have seen in Theorem \ref{thm:non-convex-sgd}, that for this update rule, 
$$ \E \left[ \frac{1}{T} \sum_t \| \nabla_t  \|^2 \right]  \leq  2 \sqrt{\frac{M \beta \sigma^2 }{T} }.$$
The convergence is proportional to the second moment of the gradient estimator, and thus it makes sense to try to reduce this second moment. The variance reduction technique attempts to do so by using the average of all previous gradients, as we show next.

\section{Setting and definitions}

We consider the ERM optimization problem over an average of loss functions. 
Before we begin, we need a few preliminaries and assumptions:
\begin{enumerate}
\item
We denote distance to optimality according to function value as   
$$ h_t = f(\x_t) - f(\x^*) , $$
and in the $k$'th epoch of an algorithm, we denote $h_t^k =  f(\x_t^k) - f(\x^*)$.

\item
We denote 
$\tilde{h}_k = \max \left\{ 4 h_0^k \ , \  8 \alpha D_k^2 \right\} $ over an epoch.

\item
Assume all stochastic gradients have bounded second moments 
$$\| \hat{\nabla_t}\|_2^2 \leq \sigma^2 . $$

\item
We will assume that the individual functions $f_i$ in formulation \eqref{eqn:ERM} are also $\hat{\beta}$-smooth and have $\hat{\beta}$-Lipschitz gradient, namely
$$ \|  {\nabla} f_i(\x) -  {\nabla} f_i(\y)  \| \leq \hat{\beta} \| \x-\y \| . $$

\item
We will use, proved in Lemma \ref{lem:elementary_properties}, that for $\beta$-smooth and $\alpha$-strongly convex $f$ we have
$$ h_t \geq \frac{1}{2 \beta} \| \nabla_t \|^2 $$ 
and
$$\frac{\alpha}{2} d_t^2 =  \frac{ \alpha}{2} \| \x_t - \x^*\|^2  \leq h_t \leq  \frac{1}{2 \alpha} \|\nabla_t\|^2 .   $$

\item
Recall that a function $f$ is $\gamma$-well-conditioned if it is $\beta$-smooth, $\alpha$-strongly convex and $\gamma \leq \frac{\alpha}{\beta}$.

\end{enumerate}

\section{The variance reduction advantage}

Consider gradient descent for $\gamma$-well conditioned functions, and specifically used for ML training as in formulation \eqref{eqn:ERM} . It is well known that GD attains linear convergence rate as we now prove for completeness:  

\begin{theorem} \label{thm:GD-unconstrained-well-conditioned}
For unconstrained minimization of $\gamma$-well-conditioned functions and $\eta_t = \frac{1}{\beta} $,  the Gradient Descent Algorithm \ref{alg:BasicGD} converges as
$$ h_{t+1} \leq  h_1  e^{- \gamma t} .$$
\end{theorem}
\begin{proof}
\begin{align*}
h_{t+1} - h_t & =  f(\x_{t+1})  - f(\x_t) \\
& \le   \nabla_t^\top (\x_{t+1} - \x_t) + \frac{\beta}{2} \|\x_{t+1} - \x_t\|^2 & \text{ $\beta$-smoothness} \\
& =  - \eta_t \|\nabla_t \|^2 + \frac{\beta}{2} \eta_t^2  \|\nabla_t\|^2 & \text{ algorithm defn.} \\
& =  - \frac{1}{2\beta} \|\nabla_t \|^2  & \text{ choice of $\eta_t=\frac{1}{\beta}$} \\
& \leq  - \frac{\alpha}{\beta} h_t.   & \text{by \eqref{eqn:gradlowerbound} } 
\end{align*}
Thus,
\begin{eqnarray*}
h_{t+1}  \leq h_t ( 1 - \frac{\alpha}{ \beta} ) \leq  \cdots \le  h_1 ( 1 - {\gamma})^t \leq h_1 e^{-{\gamma t}} 
\end{eqnarray*}
where the last inequality follows from $1 - x \leq e^{-x}$ for all $x \in \reals$. 
\end{proof}

However, what is the overall computational cost? Assuming that we can compute the gradient of each loss function corresponding to the individual training examples in $O(d)$ time, the overall running time to compute the gradient is $O(md)$. 

In order to attain approximation $\eps$ to the objective, the algorithm requires $O(\frac{1}{\gamma} \log \frac{1}{\eps})$ iterations, as per the Theorem above. Thus, the overall running time becomes $O(\frac{md}{\gamma} \log \frac{1}{\eps})$.
As we show below, variance reduction can reduce this running time to be $O(( m + \frac{1}{\tilde{\gamma}^2} ) d \log \frac{1}{\eps})$, where $\tilde{\gamma}$ is a different condition number for the same problem, that is in general smaller than the original.  Thus, in one line, the variance reduction advantage can be summarized as: \\
\begin{center}
\framebox{ $\frac{md}{\gamma} \log \frac{1}{\eps}  $    $\mapsto$ $( m + \frac{1}{\tilde{\gamma}^2} ) d \log \frac{1}{\eps}$  . }
\end{center}

\section{A simple variance-reduced algorithm}

The following simple variance-reduced algorithm illustrates the main ideas of the technique. The algorithm is a stochastic gradient descent variant which proceeds in epochs. Strong convexity implies that the distance to the optimum shrinks with function value, so it is safe to decrease the distance upper bound every epoch. 

The main innovation is in line \ref{line:main-VR}, which constructs the gradient estimator. Instead of the usual trick - which is to sample one example at random - here the estimator uses the entire gradient computed at the beginning of the current epoch. 

\begin{algorithm}[h!]
\caption{ Epoch GD  }
\label{alg:EpochGD}
\begin{algorithmic}[1]
\STATE Input: $f$, $T$, $\x_0^1 \in \K$, upper bound $D_1 \geq \|\x_0^1 - \x^* \|$, step sizes $\{\eta_t\}$ 
\FOR {$k=1$ to $ \log \frac{1}{\epsilon} $}
\STATE Let $B_{D_k}(\x_0^k)$ be the ball of radius $D_k$ around $\x_0^k$. 
\STATE compute full gradient $\nabla_0^k = \nabla f(\x_0^k) $
\FOR {$t=1$ to $T $}
\STATE Sample $i_t \in [m]$ uniformly at random, let $f_t = f_{i_t}$. 
\STATE \label{line:main-VR} construct stochastic gradient $\hat{\nabla}_t^k =  \nabla f_t (\x_t^k) - \nabla  f_t (\x_0^k) + \nabla_0^k$
\STATE Let $ \y_{t+1}^k = \x_{t}^k -\eta_t {\hat{\nabla}_t^k }    , \  \x_{t+1}= \proj_{B_{D_k}(\x_0^k)} \left( \y_{t+1}  \right) $
\ENDFOR
\STATE Set $\x_0^{k+1} = \frac{1}{T} \sum_{t=1}^T \x_t^k$.  $D_{k+1} \leftarrow D_k / 2$. 
\ENDFOR
\RETURN ${\x}_{T+1}^0 $ 
\end{algorithmic}
\end{algorithm}

The main guarantee for this algorithm is the following theorem, which delivers upon the aforementioned improvement, 
\begin{theorem}
Algorithm \ref{alg:EpochGD} returns an $\eps$-approximate solution to optimization problem \eqref{eqn:ERM} in total time 
$$ O\left( \left(m+ \frac{1}{\tilde{\gamma}^2} \right) d \log \frac{1}{\eps} \right) . $$
\end{theorem}

Let $\tilde{\gamma} = \frac{\alpha}{\hat{\beta}} < \gamma$. Then the proof of this theorem follows from the following lemma. 

\begin{lemma}
For $T = \tilde{O}\left(\frac{1}{\tilde{\gamma}^2} \right)$, we have 
$$ \E[ \tilde{h}_{k+1} ] \leq \frac{1}{2} \tilde{h}_k . $$
\end{lemma}
\begin{proof}
As a first step, we bound the variance of the gradients. Due to the fact that $\x_t^k \in B_{D_k} ( \x_0^k)$, we have that for $k' > k$, $\|\x_t^k - \x_t^{k'} \|^2 \leq 4 D_k^2$. Thus, 
\begin{eqnarray*}
\| \hat{ \nabla}_t^k \|^2  
& =   \| { \nabla}f_t (\x_t^k) - {\nabla} f_t (\x_0^k) + \nabla f (\x_0^k)  \|^2  & \mbox{ definition} \\
& \leq   2 \| { \nabla}f_t (\x_t^k) - {\nabla} f_t (\x_0^k) \|^2 + 2 \| \nabla f (\x_0^k)  \|^2 &  (a+b)^2 \leq 2a^2 + 2b^2  \\
& \leq   2 \hat{\beta}^2 \| \x_t^k - \x_0^k \|^2 +  4 \beta h^k_0  & \mbox{ smoothness}  \\
& \leq   8 \hat{\beta}^2 D_k^2  +  4 \beta h^k_0  & \mbox{ projection step }  \\
& \leq    \hat{\beta}^2 \frac{1}{\alpha} \tilde{h}_k  + 4 \beta h^k_0  \leq \tilde{h}_k ( \frac{\hat{\beta}^2 }{\alpha} + \beta)  \\
\end{eqnarray*}
Next, using the regret bound for strongly convex functions, we have 
\begin{eqnarray*}
\E[ h^{k+1}_0]  & \leq  \E[ \frac{1}{T} \sum_t h_t^k ] & \mbox{Jensen} \\
&  \leq \frac{1}{\alpha T } \E[ \sum_t \frac{1}{t} \| \hat{ \nabla}_t^k \|^2]  & \mbox{ Theorem \ref{thm:gradient2} } \\
& \leq  \frac{1}{\alpha T} \sum_t \frac{1}{t}   \tilde{h}_k ( \frac{\hat{\beta}^2}{\alpha} + \beta) & \mbox{ above} \\
& \leq \frac{\log T }{ T} \tilde{h}_k  (\frac{1}{ \tilde{\gamma}^2} + \frac{1}{\gamma} ) & \tilde{\gamma} = \frac{\alpha}{\hat{\beta}}  
\end{eqnarray*}
Which implies the Lemma by choice of $T$, definition of $\tilde{h}_k = \max \left\{ 4 h_0^k \ , \ 8 \alpha D_k^2 \right\} $, and exponential decrease of $D_k$. 

The expectation is over the stochastic gradient definition, and is required for using Theorem \ref{thm:gradient2}. 
\end{proof}

To obtain the theorem from the lemma above, we need to strengthen it to a high probability statement using a martingale argument. This is possible since the randomness in construction of the stochastic gradients is i.i.d. 

The lemma now implies the theorem by noting that $O(\log \frac{1}{\eps})$ epochs suffices to get $\eps$-approximation. Each epoch requires the computation of one full gradient, in time $O(md)$, and $\tilde{O}(\frac{1}{\tilde{\gamma}^2})$ iterations that require stochastic gradient computation, in time $O(d)$.

\newpage
\section{Bibliographic Remarks}

The variance reduction technique was first introduced as part of the SAG algorithm \cite{schmidt2017minimizing}.  Since then a host of algorithms were developed using the technique. The simplest exposition of the technique was given in \cite{johnson2013accelerating}. The exposition in this chapter is developed from the Epoch GD algorithm \cite{hazan2014beyond}, which uses a related technique for stochastic strongly convex optimization, as developed in \cite{zhang2013linear}.

\chapter{Nesterov Acceleration}

In previous chapters we have studied our bread and butter technique, SGD, as well as two acceleration techniques of adaptive regularization and variance reduction. In this chapter we study the historically earliest acceleration technique, known as Nesterov acceleration, or simply ``acceleration". 

For smooth and convex functions, Nesterov acceleration improves the convergence rate to optimality to $O(\frac{1}{T^2})$, a quadratic improvement over vanilla gradient descent. Similar accelerations are possible when the function is also strongly convex: an accelerated rate of $e^{-\sqrt{\gamma} T}$, where $\gamma$ is the condition number, vs. $e^{-\gamma T}$ of vanilla gradient descent. This improvement is theoretically very significant.

However, in terms of applicability, Nesterov acceleration is theoretically the most restricted in the context of machine learning: it requires a smooth and convex objective. More importantly, the learning rates of this method are very brittle, and the method is not robust to noise. Since noise is predominant in machine learning, the theoretical guarantees in stochastic optimization environments are very restricted. 

However, the heuristic of momentum, which historically inspired acceleration, is extremely useful for non-convex stochastic optimization (although  not known to yield significant improvements in theory).

\section{Algorithm and implementation}

Nesterov acceleration applies to the general setting of constrained smooth convex optimization:
\begin{equation} \label{eqn:shalom5}
\min_{\x \in \reals^d} f(\x) .
\end{equation}
For simplicity of presentation, we restrict ourselves to the unconstrained convex and smooth case. Nevertheless, the method can be extended to constrained smooth convex, and potentially strongly convex, settings. 

The simple method presented in Algorithm \ref{alg:nesterov} below is computationally equivalent to gradient descent. The only overhead is saving three state vectors (that can be reduced to two) instead of one for gradient descent. 
The following simple accelerated algorithm illustrates the main ideas of the technique. 

\begin{algorithm}[h!]
\caption{Simplified Nesterov Acceleration }
\label{alg:nesterov}
\begin{algorithmic}[1]
\STATE Input: $f$, $T$, initial point $\x_0 $, parameters $\eta,\beta,\tau$. 
\FOR {$t=1$ to $ T$}
\STATE Set $\x_{t+1} = \tau \z_t + (1-\tau) \y_t $, and denote $\nabla_{t+1} = \nabla f(\x_{t+1}) $.
\STATE Let $ \y_{t+1} = \x_{t+1} - \frac{1}{\beta} {\nabla}_{t+1}   $
\STATE Let  $\z_{t+1}=  \z_{t} - \eta \nabla_{t+1} $
\ENDFOR
\RETURN $\bar{\x} = \frac{1}{T} \sum_t \x_t  $ 
\end{algorithmic}
\end{algorithm}

\section{Analysis}

The main guarantee for this algorithm is the following theorem. 
\begin{theorem}
Algorithm \ref{alg:nesterov} converges to an $\eps$-approximate solution to optimization problem \eqref{eqn:shalom5} in 
$ O( \frac{1}{\sqrt{\eps} } )  $ iterations.
\end{theorem}

The proof starts with the following lemma which follows from our earlier standard derivations.
\begin{lemma} \label{lem:shalom3}
$$ \eta \nabla_{t+1}^\top  ( \z_{t} - \x^*)  \leq 2{ \eta^2 \beta} (f(\x_{t+1} )- f(\y_{t+1}) ) +  \left[ \| \z_t - \x^*\|^2 - \| \z_{t+1} -\x^*\|^2 \right] . $$
\end{lemma}
\begin{proof}
The proof is very similar to that of Theorem \ref{thm:gradient}. By definition of $\z_t$, \footnote{Henceforth we use Lemma \ref{lem:elementary_properties} part 3. This proof of this Lemma shows that for $\y = \x - \frac{1}{\beta} \nabla f(\x)$, it holds that $f(\x) - f(\y) \geq \frac{1}{2\beta} \|\nabla f(\x)\|^2$.}
\begin{eqnarray*}
\| \z_{t+1} - \x^*\|^2 & = \| \z_{t} - \eta \nabla_{t+1} - \x^* \|^2 \\
& = \| \z_{t} - \x^*\|^2 - \eta \nabla_{t+1}^\top ( \z_{t} - \x^*) + \eta^2 \| \nabla_{t+1}\|^2 \\
& \leq \| \z_{t} - \x^*\|^2 - \eta \nabla_{t+1}^\top( \z_{t} - \x^*) + 2 {\eta^2 \beta} (f(\x_{t+1}) - f(\y_{t+1}))   & \mbox{ Lemma \ref{lem:elementary_properties} part 3  } 
\end{eqnarray*}

\end{proof}

\begin{lemma} \label{lem:shalom4}
For $2 {\eta \beta} = \frac{1 - \tau}{\tau}$, we have that 
$$ \eta \nabla_{t+1}^\top ( \x_{t+1} - \x^*)  \leq 2 \eta^2 {\beta} (f(\y_t )- f(\y_{t+1}) ) +  \left[ \| \z_t - \x^*\|^2 - \| \z_{t+1} -\x^*\|^2 \right] . $$
\end{lemma}
\begin{proof}
\begin{eqnarray*}
& \eta  \nabla_{t+1}^\top  ( \x_{t+1} - \x^*) - \eta  \nabla_{t+1}^\top  ( \z_{t} - \x^*) \\
& = \eta  \nabla_{t+1}^\top  ( \x_{t+1} - \z_t ) \\
& = \frac{(1 - \tau) \eta}{\tau}  \nabla_{t+1}^\top  ( \y_t - \x_{t+1} ) &  \tau(\x_{t+1} - \z_t) = (1-\tau) (\y_t - \x_{t+1}) \\
& \leq \frac{(1 - \tau) \eta}{\tau} ( f(  \y_t ) - f( \x_{t+1} )) . &  \mbox{ convexity}
\end{eqnarray*}
Thus, in combination with Lemma \ref{lem:shalom3}, and the condition of the Lemma, we get the inequality. 
\end{proof}

We can now sketch the proof of the main theorem.
\begin{proof}
Telescope Lemma \ref{lem:shalom4} for all iterations to obtain:
\begin{eqnarray*}
T h_T & = T ( f(\bar{\x}) - f(\x^*)) \\ 
 & \leq  \sum_t \nabla_t^\top (\x_t - \x^*) \\
& \leq 2 \eta  {\beta} \sum_t (f(\y_t )- f(\y_{t+1}) ) +  \frac{1}{\eta} \sum_t \left[ \| \z_t - \x^*\|^2 - \| \z_{t+1} -\x^*\|^2 \right] \\
& \leq  2 \eta  {\beta}  (f(\y_1 )- f(\y_{T+1}) ) +  \frac{1}{\eta}  \left[ \| \z_1 - \x^*\|^2 - \| \z_{T+1} -\x^*\|^2 \right] \\
& \leq  \sqrt{2  \beta  h_1  D} , & \mbox{optimizing $\eta$}
\end{eqnarray*}
where $h_1$ is an upper bound on the distance $f(\y_1) - f(\x^*)$, and $D$ bounds the Euclidean distance of $\z_t$ to the optimum.  Thus, we get a recurrence of the form
$$ h_T \leq \frac{\sqrt{h_1}}{T} .$$
Restarting Algorithm \ref{alg:nesterov} and adapting the learning rate according to $h_T$ gives a rate of convergence of $O(\frac{1}{T^2})$ to optimality.
\end{proof}

\newpage
\section{Bibliographic Remarks}

Accelerated rates of order $O(\frac{1}{T^2})$ were obtained by Nemirovski as early as the late seventies. The first practically efficient accelerated algorithm is due to Nesterov \cite{Nesterov} , see also \cite{NesterovBook}.  The simplified proof presented hereby is due to \cite{allen2014linear}.


\chapter{The conditional gradient method}\label{subsec:cond_grad_intro}

In many computational and learning scenarios the main bottleneck of optimization, both online and offline, is the computation of projections onto the underlying decision set (see \S \ref{sec:projections}). In this chapter we discuss projection-free methods in convex optimization, and some of their applications in machine learning.

The motivating example throughout this chapter is the problem of matrix completion, which is a widely used and accepted model in the construction of recommendation systems. For matrix completion and related problems, projections amount to expensive linear algebraic operations and avoiding them is crucial in big data applications. 

Henceforth we describe the conditional gradient algorithm, also known as the Frank-Wolfe algorithm. Afterwards, we describe problems for which linear optimization can be carried out much more efficiently than projections. We conclude with an application to exploration in reinforcement learning. 

\section{Review: relevant concepts from linear algebra}

This chapter addresses rectangular matrices, which model  applications  such as recommendation systems naturally.  Consider a matrix $X \in \reals^{n \times m}$. A non-negative number $\sigma \in \reals_+$ is said to be a singular value for $X$ if there are two vectors $\uv \in \reals^n, \vv \in \reals^m$ such that 
$$  X^\top \uv  = \sigma \vv   , \quad X \vv = \sigma \uv. $$
The vectors $\uv,\vv$ are called the left and right singular vectors respectively. The non-zero singular values are the square roots of the eigenvalues of the matrix $X X^\top$ (and $X^\top X$).  The matrix $X$ can be written as 
$$ X = U \Sigma V^\top  \ , \ U \in \reals ^{n \times \rho} \ , \ V^\top  \in \reals^{ \rho \times m} ,$$
where $\rho = \min\{n,m\}$, the matrix $U$ is an orthogonal basis of the left singular vectors of $X$, the matrix $V$ is an orthogonal basis of right singular vectors, and $\Sigma$ is a diagonal matrix of singular values. This form is called the singular value decomposition for $X$. 

The number of non-zero singular  values for $X$ is called its rank, which we denote by $k \leq \rho$. 
The nuclear norm of $X$ is defined as the $\ell_1$ norm of its singular values, and denoted by
$$ \|X \|_* = \sum_{i=1}^\rho \sigma_i $$
It can be shown (see exercises) that the nuclear norm is equal to the trace of the square root of the matrix times its transpose, i.e., 
$$ \|X\|_* = \trace( \sqrt{ X^\top X}  ) $$
We denote by $A \bullet B$ the inner product of two matrices as vectors in $\reals^{n \times m}$, that is
$$A \bullet B = \sum_{i = 1}^n \sum_{j=1}^m A_{ij} B_{ij} = \trace(AB^\top) $$

\section{Motivation: matrix completion and recommendation systems}
\sectionmark{Motivation}

Media recommendations have changed significantly with the advent of the Internet and rise of online media stores. The large amounts of data collected allow for efficient clustering and accurate prediction of users' preferences for a variety of media. A well-known example is the so called ``Netflix challenge''---a competition of automated tools for recommendation from a large dataset of users' motion picture preferences.

One of the most successful approaches for automated recommendation systems, as proven in the Netflix competition, is matrix completion. Perhaps the simplest version of the problem can be described as follows.  

The entire dataset of user-media preference pairs is thought of as a partially-observed matrix. Thus, every person is represented by a row in the matrix, and every column represents a media item (movie). For simplicity, let us think of the observations as binary---a person either likes or dislikes a particular movie. Thus, we have a matrix $M \in \{0,1,*\}^{n \times m}$  where $n$ is the number of persons considered, $m$ is the number of movies at our library, and $0/1$ and $*$ signify ``dislike'', ``like'' and ``unknown'' respectively:
$$ M_{ij} = \mythreecases {0}{\mbox{person $i$ dislikes movie $j$}}{1}{\mbox{person $i$ likes movie $j$}}{*}{\mbox{preference unknown}} .$$ 

The natural goal is to complete the matrix, i.e. correctly assign $0$ or $1$ to the unknown entries. As defined so far, the problem is ill-posed, since any completion would be equally good (or bad), and no restrictions have been placed on the completions.  

The common restriction on completions is that the ``true'' matrix has low rank. Recall that a matrix $X \in \reals^{n \times m}$ has rank $k < \rho = \min \{n,m\} $ if and only if it can be written as 
$$ X = U V \ , \ U \in \reals^{n \times k} , V \in \reals^{k \times m}.  $$

The intuitive interpretation of this property is that each entry in $M$ can be explained by only $k$ numbers. In matrix completion this means, intuitively, that there are only $k$ factors that determine a persons preference over movies, such as genre, director, actors and so on. 

Now the simplistic matrix completion problem can be well-formulated as in the following mathematical program. Denote by $\| \cdot \|_{OB}$ the Euclidean norm only on the observed (non starred) entries of $M$, i.e., 
$$\|X\|_{OB}^2 = \sum_{M_{ij} \neq *} X_{ij}^2.$$ 
The mathematical program for matrix completion is given by
\begin{align*}
& \min_{X \in \reals^{n \times m} } \frac{1}{2} \| X - M \|_{OB}^2 \\
& \text{s.t.} \quad \rank(X) \leq k. 
\end{align*}  

Since the constraint over the rank of a matrix is non-convex, it is standard to consider a relaxation that replaces the rank constraint by the nuclear norm. 
It is known that the nuclear norm is a lower bound on the matrix rank if the singular values are bounded by one (see exercises). 
Thus, we arrive at the following convex program for matrix completion: 
\begin{align} \label{eqn:matrix-completion}
& \min_{X \in \reals^{n \times m} }  \frac{1}{2} \| X - M \|_{OB}^2 \\
& \text{s.t.} \quad \|X\|_* \leq k. \notag 
\end{align}  

We consider algorithms to solve this convex optimization problem next.

\section{The Frank-Wolfe method} 

In this section we consider minimization of a convex function over a convex domain. 

The conditional gradient (CG) method, or Frank-Wolfe algorithm, is a simple algorithm for minimizing a smooth convex function $f$ over a convex set $\K \subseteq \reals^n$. The appeal of the method is that it is a first order interior point method - the iterates always lie inside the convex set, and thus no projections are needed, and the update step on each iteration simply requires minimizing a linear objective over the set. The basic method is given in Algorithm \ref{alg:condgrad}.
\begin{algorithm}[H]
	\caption{Conditional gradient}
	\label{alg:condgrad}
	\begin{algorithmic}[1]
		\STATE Input: step sizes $\{ \eta_t \in (0,1] , \ t \in [T]\}$, initial point $\x_1 \in \K$. 
		\FOR{$t = 1$ to $T$}
		\STATE $\vv_{t} \gets \arg \min_{\x \in \K} \left\{\x^\top \nabla{}f(\x_t) \right\} $. \label{algstep:linearopt}
		\STATE $\x_{t+1} \gets \x_t + \eta_t(\vv_t - \x_t)$.
		\ENDFOR
	\end{algorithmic}
\end{algorithm}

Note that in the CG method, the update to the iterate $\x_t$ may be not be in the direction of the gradient, as $\vv_t$ is the result of a linear optimization procedure in the direction of the negative gradient. This is depicted in Figure \ref{fig:OFW}. 

\begin{figure}[h!]
\begin{center}
\includegraphics[width=3.5in]{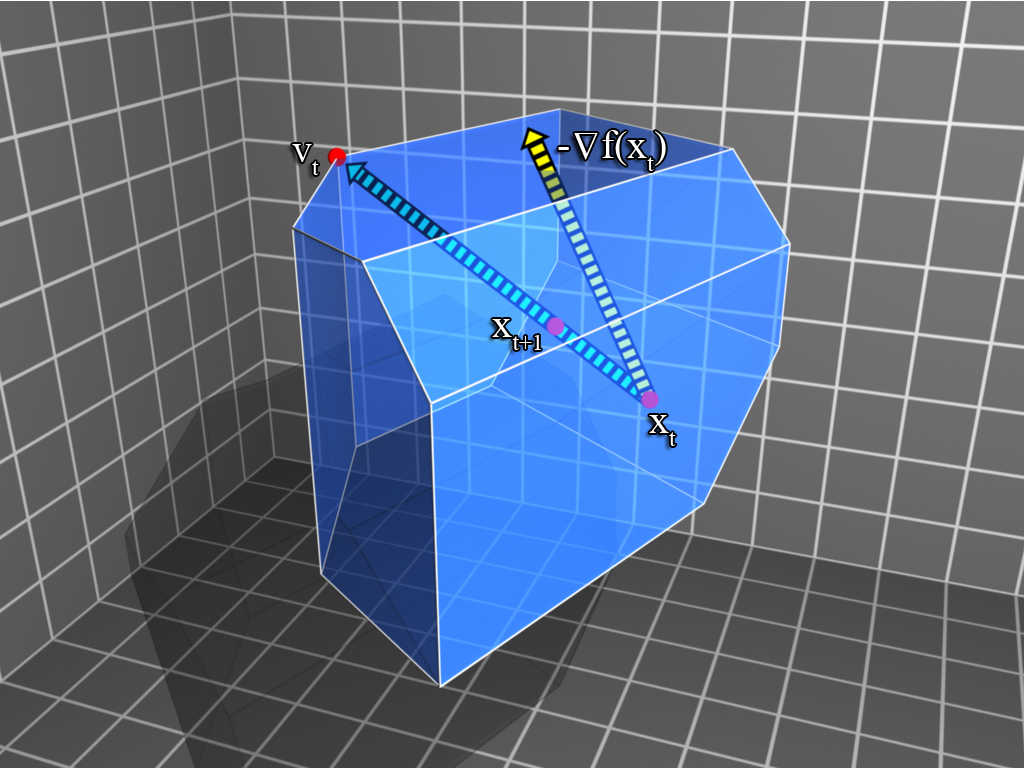}
\end{center}
\caption{Direction of progression of the conditional gradient algorithm. \label{fig:OFW}}
\end{figure}

The following theorem gives an essentially tight performance guarantee of this algorithm over smooth functions. Recall our notation from Chapter \ref{chap:opt}: $\x^\star$ denotes the global minimizer of $f$ over $\K$, $D$ denotes the diameter of the set $\K$, and $h_t = f(\x_t) - f(\x^\star)$  denotes the suboptimality of the objective value in iteration $t$. 
\begin{theorem} \label{thm:offlineFW}
The CG algorithm applied to $\beta$-smooth functions with step sizes $\eta_t =  \min\{\frac{2H}{t},1\}$, for $H \geq \max\{1,h_1\} $,  attains the following convergence guarantee:
$$ h_t \leq \frac{2 \beta H D^2 }{t} $$
\end{theorem}
\begin{proof}
As done before in this manuscript, we denote $\nabla_t = \nabla f(\x_t)$, and also denote $H \geq \max \{h_1,1\}$, such that $\eta_t = \min\{ \frac{2 H}{t},1\}$.  For any set of step sizes, we have
\begin{align}\label{old_fw_anal}
	&  f(\x_{t+1}) - f(\x^\star) 
	= f(\x_t + \eta_t(\vv_t - \x_t)) - f(\x^\star) \notag \\
	&\leq  f(\x_t) - f(\x^\star) + \eta_t(\vv_t-\x_t)^{\top}\nabla_t + \eta_t^2 \frac{\beta}{2}\Vert{\vv_t-\x_t}\Vert^2 &  \textrm{$\beta$-smoothness } \nonumber \\
	&\leq  f(\x_t) - f(\x^\star) + \eta_t(\x^\star-\x_t)^{\top}\nabla_t + \eta_t^2 \frac{\beta}{2}\Vert{\vv_t-\x_t}\Vert^2 &  \textrm{$\vv_t$ optimality} \nonumber \\
	&\leq  f(\x_t) - f(\x^\star) + \eta_t(f(\x^\star)-f(\x_t)) + \eta_t^2 \frac{\beta}{2}\Vert{\vv_t-\x_t}\Vert^2 & \textrm{convexity of $f$} \nonumber \\
	&\leq  (1-\eta_t)(f(\x_t)-f(\x^\star)) + \frac{\eta_t^2\beta}{2} D^2. 
\end{align}
We reached the recursion $ h_{t+1} \leq (1- \eta_t) h_t + \eta_t^2\frac{  \beta D^2}{2} $, and by induction,
\begin{align*}
 h_{t+1} & \leq (1- \eta_t) h_t + \eta_t^2 \frac{\beta D^2}{2}  \\
 & \leq (1- \eta_t) \frac{2  \beta H D ^2 }{ t} + \eta_t^2 \frac{ \beta  D^2}{2} & \mbox{induction hypothesis}\\
 & \leq (1- \frac{2 H}{ t}) \frac{2  \beta H D^2}{t} + \frac{4 H^2}{  t^2} \frac{\beta D^2 }{2}& \mbox{value of $\eta_t$}\\
 & = \frac{2  \beta H D^2 }{t} -  \frac{2 H^2 \beta D^2 }{ t^2 } \\
 & \leq \frac{2  \beta H D^2 }{t} (1 - \frac{1}{t} ) & \mbox{since $H \geq 1$} \\
 & \leq \frac{2  \beta H D^2 }{t+1 }.  & \mbox{$\frac{t-1}{t} \leq \frac{t}{t+1} $ } \\
 \end{align*}

\end{proof}

\section{Projections vs. linear optimization}

The conditional gradient (Frank-Wolfe) algorithm described before does not resort to projections, but rather computes a linear optimization problem of the form
\begin{equation} \label{eqn:linopt}
 \arg \min_{\x \in \K} \left\{\x^\top \uv \right\}. 
\end{equation}
When is the CG method computationally preferable?  The overall computational complexity of an iterative optimization algorithm is the product of the number of iterations and the computational cost per iteration. The CG method does not converge as well as the most efficient gradient descent algorithms, meaning it requires more iterations to produce a solution of a comparable level of accuracy. However, for many interesting scenarios the computational cost of a linear optimization step \eqref{eqn:linopt} is {\em significantly} lower than that of a projection step. 


Let us  point out several examples of problems for which we have very efficient linear optimization algorithms, whereas our state-of-the-art algorithms for computing projections  are significantly slower.

\paragraph*{Recommendation systems and matrix prediction.}

In the example pointed out in the preceding section of matrix completion, known methods for projection onto the spectahedron, or more generally the bounded  nuclear-norm ball, require singular value decompositions, which  take superlinear time via our best known methods. In contrast, the CG method requires maximal eigenvector computations which can be carried out in linear time via the power method (or the more sophisticated Lanczos algorithm).

\paragraph*{Network routing and convex graph problems.}

Various routing and graph problems can be modeled as convex optimization problems over a convex set called the flow polytope. 

Consider  a directed acyclic graph with $m$ edges, a source node marked $s$ and a target node marked $t$. Every path from $s$  to $t$ in the graph can be represented by its identifying vector, that is a vector in $\lbrace{0,1}\rbrace^m$ in which the entries that are set to 1 correspond to edges of the path. The flow polytope of the graph is the convex hull of all such identifying vectors  of the simple paths from $s$ to $t$. This polytope is also exactly the set of all unit $s$--$t$ flows in the graph if we assume that each edge has a unit flow capacity (a flow is represented here as a vector in $\mathbb{R}^m$ in which each entry is the amount of flow through the corresponding edge). 

Since the flow polytope is just the convex hull of $s$--$t$ paths in the graph, minimizing a linear objective over it amounts to finding a minimum weight path given weights for the edges. For the shortest path problem we have  very efficient combinatorial optimization algorithms, namely Dijkstra's algorithm. 

Thus, applying the CG algorithm to solve {\bf any} convex optimization problem over the flow polytope will only require iterative shortest path computations.

\paragraph*{Ranking and permutations. }

A common way to represent a permutation or ordering is by a permutation matrix. Such are square matrices over $\{0,1\}^{n \times n}$ that contain exactly one $1$ entry in each row and column.

Doubly-stochastic matrices are square, real-valued matrices with non-negative entries, in which the sum of entries of each row and each column amounts to 1. The polytope that defines all doubly-stochastic matrices is called the Birkhoff-von Neumann polytope. 
The Birkhoff-von Neumann theorem states that this polytope is the convex hull of exactly all $n\times{n}$ permutation matrices. 

Since a permutation matrix corresponds to a perfect matching in a fully connected bipartite graph, linear minimization over this polytope corresponds to finding a minimum weight perfect matching in a bipartite graph.

Consider a convex optimization problem over the Birkhoff-von Neumann polytope. The CG algorithm will iteratively solve a linear optimization problem over the BVN polytope, thus iteratively solving a minimum weight perfect matching in a bipartite graph problem, which is a well-studied combinatorial optimization problem for which we know of efficient algorithms. In contrast, other gradient based methods will require projections, which are quadratic optimization problems over the BVN polytope.

\paragraph*{Matroid polytopes.}

A matroid is pair $(E,I)$ where $E$ is a set of elements and $I$ is a set of subsets of $E$ called the independent sets which satisfy various interesting proprieties that resemble the concept of linear independence in vector spaces. 
Matroids have been studied extensively in combinatorial optimization and a key example of a matroid is the graphical matroid in which the set $E$ is the set of edges of a given graph and the set $I$ is the set of all subsets of $E$ which are cycle-free. In this case, $I$ contains all the spanning trees of the graph. A subset $S\in{I}$ could be represented by its identifying vector which lies in $\lbrace{0,1}\rbrace^{\vert{E}\vert}$ which also gives rise to the matroid polytope which is just the convex hull of all identifying vectors of sets in $I$. It can be shown that some matroid polytopes are defined by exponentially many linear inequalities (exponential in $\vert{E}\vert$), which makes optimization over them difficult. 

On the other hand, linear optimization over matroid polytopes is easy using a simple greedy procedure which runs in nearly linear time. Thus, the CG method serves as an efficient algorithm to solve any convex optimization problem over matroids iteratively using only a simple greedy procedure.

\newpage
\section{Exercises}

\begin{enumerate}

\item
Prove that if the singular values are smaller than or equal to one, then the nuclear norm is a lower bound on the rank, i.e., show
$$ \rank(X) \geq \|X\|_* .$$ 

\item
Prove that the trace  is related to the nuclear norm via
$$ \| X \|_* = \trace( \sqrt{X X^\top} ) = \trace( \sqrt{ X^\top X} ) .$$

\item
Show that maximizing a linear function over the spectahedron is equivalent to a maximal eigenvector computation. That is, show that the following mathematical program:
\begin{align*} 
& \min  X \bullet C  \\
& X \in S_d = \{ X \in \reals^{d \times d} \ , \ X \succcurlyeq 0 \ , \ \trace(X)  \leq 1 \} ,
\end{align*}  
is equivalent to the  following: 
\begin{align*} 
& \min_{\x \in \reals^d}  \x^\top C \x  \\
& \mbox{s.t.  }   \|\x\|_2  \leq 1 .
\end{align*}

\item

Download the MovieLens dataset from the web. Implement an online recommendation system based on the matrix completion model: implement the OCG and OGD algorithms for matrix completion. Benchmark your results.

\end{enumerate}

\newpage
\section{Bibliographic Remarks}

The matrix completion model has been extremely popular since its inception in the context of recommendation systems \cite{SrebroThesis,Rennie:2005,salakhutdinov:collaborative,lee:practical,CandesR09,ShamirS11}.

The conditional gradient algorithm was devised in the seminal paper by Frank and Wolfe \cite{FrankWolfe}. Due to the applicability of the FW algorithm to large-scale constrained problems, it has been a method of choice in recent machine learning applications, to name a few: 
\cite{Jaggi10, Jaggi13a, Jaggi13b, Dudik12a, Dudik12b, Hazan12, ShalevShwartz11, Bach12, Tewari11, Garber11, Garber13, Florina14}.

The online conditional gradient algorithm is due to \cite{Hazan12}. An optimal regret algorithm, attaining the $O(\sqrt{T})$ bound, for the special case of polyhedral sets was devised in \cite{Garber13}.


\chapter{Second order methods for machine learning }\label{sec:newton}
\chaptermark{Second order methods}

At this point in our course, we have exhausted the main techniques in first-order (or gradient-based) optimization. We have studied the main workhorse - stochastic gradient descent, the three acceleration techniques, and projection-free gradient methods. Have we exhausted optimization for ML? 

In this section we discuss using higher derivatives of the objective function to accelerate optimization. The canonical method is Newton's method, which involves the second derivative or Hessian in high dimensions. The vanilla approach is computationally expensive since it involves matrix inversion in high dimensions that  machine learning problems usually require. 

However, recent progress in random estimators gives rise to linear-time second order methods, for which each iteration is as computationally cheap as gradient descent.

\section{Motivating example: linear regression}

In the problem of linear regression we are given a set of measurements $\{\ba_i \in \reals^d, b_i \in \reals\} $, and the goal is to find a set of weights that explains them best in the mean squared error sense.  As a mathematical program, the goal is to optimize: 
$$ \min_{\x \in \reals^d} \left\{ \frac{1}{2} \sum_{i \in [m]} \left( \ba_i^\top \x - b_i \right)^2 \right\} , $$
or in matrix form, 
$$ \min_\x f(\x) =  \left\{  \frac{1}{2} \| A \x - \bb \|^2 \right\} . $$
Here $A \in \reals^{m \times d}, \bb \in \reals^m$. Notice that the objective function $f$ is smooth, but not necessarily strongly convex. Therefore, all  algorithms that we have studied so far without exception, which are all first order methods, attain rates which are $\poly(\frac{1}{\eps})$. 

However, the linear regression problem has a closed form solution that can be computed by taking the gradient to be zero, i.e. $ (A \x - \bb)^\top A   = 0$, which gives 
$$ \x = (A^\top A)^{-1} A^\top \bb . $$

The Newton direction is given by the inverse Hessian multiplied by the gradient, $\nabla^{-2} f(\x) \nabla f(\x)$. 
Observe that a single Newton step, i.e. moving in the Newton direction with step size one, from any direction gets us directly to the optimal solution in one iteration! (see exercises)

More generally, Newton's method yields $O(\log \frac{1}{\eps})$ convergence rates for a large class of functions without dependence on the condition number of the function! We study this property next.

\section{Self-Concordant Functions}

In this section we define and collect some of the properties of
a special class of functions, called self-concordant functions.  These functions allow Newton's method to run in time which is independent of the condition number. The class of self-concordant functions is expressive and includes quadratic functions, logarithms of inner products, a variety of barriers such as the log determinant, and many more. 

An excellent reference for this material is the
lecture notes on this subject by Nemirovski \cite{NemirovskiBook}. We begin by
defining self-concordant functions. 
\begin{definition}[Self-Concordant Functions]
Let $\K \subseteq \reals^n$ be a non-empty open convex set, and and let $f : \K \mapsto \reals$ be a $C^3$ convex function. Then, $f$ is said to be \text{self-concordant} if
$$|\nabla^3 f(\x)[\bh, \bh, \bh]| \leq 2(\bh^{\top} \hess f(\x) \bh)^{3/2},$$
where we have
\begin{equation*}
\nabla^k f(\x)[\bh_1, \dots, \bh_k] \defeq \frac{\partial^k}{\partial
t_1\dots\partial t_k} |_{t_1=\dots=t_k} f(\x+t_1\bh_1 + \dots + t_k\bh_k).
\end{equation*}
\end{definition}
Another key object in the analysis of self concordant functions is the notion of
a Dikin Ellipsoid, which is the unit ball around a point in the norm given by
the Hessian $\|\cdot\|_{\hess f}$ at the point. We will refer to this norm
as the \textit{local norm} around a point and denote it as $\|\cdot\|_{\x}$.
Formally, \begin{definition}[Dikin ellipsoid]
The Dikin ellipsoid of radius $r$ centered at a point $\x$ is defined as 
$$\ellipsoid_r(\x) \defeq \{\y\ |\ \|\y - \x\|_{\hess f(\x)} \leq r\}$$ 
\end{definition}
One of the key properties of self-concordant functions that we use is that
inside the Dikin ellipsoid, the function is well conditioned with respect to the local norm
at the center. The next lemma makes this formal. The proof of this lemma can
be found in \cite{NemirovskiBook}.
\begin{lemma} [See \cite{NemirovskiBook}] For all $\h$ such that
$\|\h\|_{\x} < 1$ we have that \[ (1 - \|\h\|_{\x})^2 \hess f(\x) \preceq \hess
f(\x + \h) \preceq \frac{1}{(1 - \|\h\|_{\x})^2}\hess f(\x)\]
\end{lemma}
Another key quantity, which is used both as a potential function as well as a
dampening for the step size in the analysis of Newton's method, is the
Newton Decrement:
$$\lambda_\x \defeq \|\nabla f(\x)\|_{\x}^{*} =
\sqrt{\nabla f(\x)^{\top} \hessinv f(\x) \nabla f(\x)} .$$ 
The following lemma
quantifies how $\lambda_\x$ behaves as a potential by showing that once it
drops below 1, it ensures that the minimum of the function lies in the current Dikin
ellipsoid. This is the property which we use crucially in our analysis. The
proof can be found in \cite{NemirovskiBook}.
\begin{lemma} [See \cite{NemirovskiBook}]
\label{lemma:lambdalemma}
If $\lambda_\x < 1$ then 
\[ \|\x - \x^*\|_{\x}  \leq \frac{\lambda_\x}{1 - \lambda_\x}\]
\end{lemma}

\sectionmark{Newton's method}
\section{Newton's method for self-concordant functions}

Before introducing the linear time second order methods, we start by introducing a robust Newton's method and its properties.  The pseudo-code is given in Algorithm \ref{alg:newton}.  

The usual analysis of Newton's method allows for quadratic convergence, i.e. error $\eps$ in $O(\log \log \frac{1}{\eps})$ iterations for convex objectives. However, we prefer to present a  version of Newton's method which is robust to certain random estimators of the Newton direction. This yields a slower rate of $O(\log \frac{1}{\eps})$.  The faster running time per iteration, which does not require matrix manipulations, more than makes up for this.

\begin{algorithm}[h!]
\caption{\textbf{Robust Newton's method}}
\label{alg:newton}
\begin{algorithmic}
\STATE {\bfseries Input: $T, \x_1$}
\FOR{$t = 1$ to $T$}
\STATE{Set $c = \frac{1}{8}$, $\eta = \min\{  c ,  \frac{c}{8\lambda_{\x_t} } \} $. 
Let $\frac{1}{2} {\nabla}^{-2}f(\x_t) \preceq \tilde{\nabla}_t^{-2} \preceq 2 {\nabla}^{-2}f(\x_t) $. }
\STATE{$\x_{t+1} = \x_t - \eta \tilde{\nabla}^{-2}_t \nabla f(\x_t) $}
\ENDFOR

\STATE{\textbf{return} $\x_{T+1}$}

\end{algorithmic}
\end{algorithm}

It is important to notice that every two consecutive points are within the same Dikin ellipsoid of radius $\frac{1}{2}$. Denote $\nabla_t = \nabla_{\x_t}$, and similarly for the Hessian. Then we have:
$$ \| \x_t - \x_{t+1} \|_{\x_t}^2 =  \eta^2 \nabla_t^\top  \tilde{\nabla}^{-2}_t  \nabla_t^2 \tilde{\nabla}^{-2}_t \nabla_t \leq 4 \eta^2  \lambda_t^2 \leq \frac{1}{2} . $$ 

The advantage of Newton's method as applied to self-concordant functions is its linear convergence rate, as given in the following theorem.
\begin{theorem}
Let $f$ be self-concordant, and $f(\x_1) \leq M$, then
$$ h_t = f(\x_t) - f(\x^*) \leq O( {M} +   \log \frac{1}{\eps} ) $$
\end{theorem}

The proof of this theorem is composed of two steps, according to the magnitude of the Newton decrement. 

\paragraph{Phase 1: damped Newton}

\begin{lemma}
As long as $\lambda_\x \geq \frac{1}{8}$, we have that 
$$ h_t \leq - \frac{1}{4} c $$
\end{lemma}
\begin{proof}
Using similar analysis to the descent lemma we have that
\begin{eqnarray*}
& f(\x_{t+1}) - f(\x_t) \\
& \leq \nabla_t^\top (\x_{t+1} - \x_t) + \frac{1}{2} (\x_t - \x_{t+1})^\top \nabla^2 (\zeta) (\x_t - \x_{t+1} ) & \mbox{Taylor} \\
& \leq \nabla_t^\top (\x_{t+1} - \x_t) +  \frac{1}{4} (\x_t - \x_{t+1})^\top \nabla^2 (\x_t) (\x_t - \x_{t+1} ) &  \x_{t+1} \in \ellipsoid_{1/2}(\x_t)   \\
& = - \eta \nabla_{t}^\top \tilde{\nabla}_{t}^{-2} \nabla_{t}  + \frac{1}{4} \eta^2 \nabla_t ^\top  \tilde{\nabla}^{-2}_t  \nabla_t^2 \tilde{\nabla}^{-2}_t \nabla_t    \\
& = - \eta \lambda_t^2 + \frac{1}{4} \eta^2 \lambda_t^2 \leq - \frac{1}{16} c
\end{eqnarray*}
\end{proof}
The conclusion from this step is that after $O(M)$ steps, Algorithm \ref{alg:newton} reaches a point for which $\lambda_\x \leq \frac{1}{8}$. According to Lemma \ref{lemma:lambdalemma}, we also have that $\|\x - \x^* \|_\x \leq \frac{1}{4}$, that is, the optimum is in the same Dikin ellipsoid as the current point.

\paragraph{Phase 2: pure Newton} 

In the second phase our step size is changed to be larger. In this case, we are guaranteed that the Newton decrement is less than one, and thus we know that the global optimum is in the same Dikin ellipsoid as the current point. In this ellipsoid, all Hessians are equivalent up to a factor of two, and thus Mirrored-Descent with the inverse Hessian as preconditioner becomes gradient descent. We make this formal below. 

\begin{algorithm}[h!]
\caption{\textbf{Preconditioned Gradient Descent}}
\label{alg:MD-sc}
\begin{algorithmic}
\STATE {\bfseries Input: $P,T$}
\FOR{$t = 1$ to $T$}
\STATE{$\x_{t+1} = \x_t - \eta P^{-1} \nabla f(\x_t) $}
\ENDFOR
\STATE{\textbf{return} $\x_{T+1}$}
\end{algorithmic}
\end{algorithm}

\begin{lemma} \label{thm:MD-sc}
Suppose that $ \frac{1}{2} P \preceq \nabla^2 f(\x) \preceq 2 P$, and $\| \x_1 - \x^*\|_P \leq \frac{1}{2}$, then Algorithm \ref{alg:MD-sc} converges as
$$ h_{t+1} \leq  h_1  e^{- \frac{1}{8}  t} .$$
\end{lemma}

This theorem follows from noticing that the function $g(\z) = f(P^{-1/2} \x)$ is $\frac{1}{2}$-strongly convex and $2$-smooth, and using Theorem \ref{thm:basicGDunconstrained}. It can be shown that gradient descent on $g$ is equivalent to Newton's method in $f$. Details are left as an exercise. 

An immediate corollary is that Newton's method converges at a rate of $O(\log \frac{1}{\eps})$ in this phase.

\section{Linear-time second-order methods}

Newton's algorithm is of foundational importance in the study of mathematical programming in general. A major application are interior point methods for convex optimization, which are the most important polynomial-time algorithms for general constrained convex optimization. 

However, the main downside of this method is the need to maintain and manipulate matrices - namely the Hessians. This is completely impractical for machine learning applications in which the dimension is huge. 

Another significant downside is the non-robust nature of the algorithm, which makes applying it in stochastic environments challenging.

In this section we show how to apply Newton's method to machine learning problems. This involves relatively new developments that allow for linear-time per-iteration complexity, similar to SGD, and theoretically superior running times. At the time of writing, however, these methods are practical only for convex optimization, and have not shown superior performance on optimization tasks involving deep neural networks. 

The first step to developing a linear time Newton's method is an efficient stochastic estimator for the Newton direction, and the Hessian {\bf inverse}.

\subsection{Estimators for the Hessian Inverse}
\label{sec:estimators}

The key idea underlying the construction is the following well known fact about the
Taylor series expansion of the matrix inverse.
\begin{lemma}
\label{fact:inverse}
For a matrix $A \in \reals^{d\times d}$ such that $A \succeq 0 \text{ and } \|A\| \leq
1$, we have that \[ A^{-1} = \sum_{i=0}^{\infty} (I - A)^i\]
\end{lemma}

We propose two unbiased estimators based on the above series. 
To define the first estimator pick a probability distribution over
non-negative integers $\{p_i\}$ and sample $\hat{i}$ from the above
distribution. Let $X_1, \ldots X_{\hat{i}}$ be independent samples of the
Hessian $\hess f$ and define the estimator as 
\begin{definition}[Estimator 1]\label{def:estimator1}
\[\hestinv f = \frac{1}{p_{\hat{i}}}
\prod_{j=1}^{\hat{i}}(I - X_j)\]\end{definition} 
Observe that our estimator of the Hessian inverse is
unbiased, i.e.
$\E [\hat{X}] = \hessinv f$ at any point.  Estimator 1 has the disadvantage that in
a single sample it incorporates only one term of the Taylor series.

The second estimator below is based
on the observation that the above series has the following succinct recursive
definition, and is more efficient.

For a matrix $A$ define \[A^{-1}_j = \sum_{i=0}^{j} (I - A)^i\] i.e.
the first $j$ terms of the above Taylor expansion. It is easy to see that
the following recursion holds for $A^{-1}_j$
\[A^{-1}_j = I + (I-A)A^{-1}_{j-1}\]

Using the above recursive formulation, we now describe an unbiased estimator of
$\hessinv f$ by deriving an unbiased estimator $\hestinv f_j$ for $\hessinv f_j$.
\begin{definition}[Estimator 2]
\label{def:estimator2}
Given $j$ independent and unbiased samples $\{X_1 \ldots X_j\}$ of the hessian
$\hess f$.
Define $\{\hestinv f_0 \ldots \hestinv f_j\}$ recursively as follows
\[ \hestinv f_0 = I\]
\[ \hestinv f_t = I + (I-X_j)\hestinv f_{t-1}\]
\end{definition}

It can be readily seen that $\E[\hestinv f_j] = \hessinv f_j$ and therefore
$\E[\hestinv f_j] \rightarrow \hessinv f$ as $j \rightarrow \infty$ giving us
an unbiased estimator in the limit.

\subsection{Incorporating the estimator} 

Both of the above estimators can be computed using only Hessian-vector products, rather than matrix manipulations. For many machine learning problems, Hessian-vector products can be computed in linear time. Examples include:
\begin{enumerate}
\item
Convex regression and SVM objectives over training data have the form 
$$ \min_{\w} f(\w) =  \E_i [ \ell(\w^\top \x_i) ] , $$
where $\ell$ is a convex function. The Hessian can thus be written as 
$$ \nabla^2 f(\w) = \E_i [ \ell'' (\w^\top \x_i) \x_i \x_i^\top ] $$ 

Thus, the first Newton direction estimator can now be written as
$$ \tilde{\nabla}^2 f(\w)  \nabla_\w = \E_{j \sim \D }  [ \prod_{i = 1}^j  (I -  \ell'' (\w^\top \x_i) \x_i \x_i^\top ) ] \nabla_\w . $$

Notice that this estimator can be computed using $j$ vector-vector products if the ordinal $j$ was randomly chosen. 

\item

Non-convex optimization over neural networks: a similar derivation as above shows that the estimator can be computed only using Hessian-vector products. The special structure of neural networks allow this computation in a constant number of backpropagation steps, i.e. linear time in the network size, this is called the ``Pearlmutter trick", see \cite{HessianPearlmutter}.

We note that non-convex optimization presents special challenges for second order methods, since the Hessian need not be positive semi-definite. Nevertheless, the techniques presented hereby can still be used to provide theoretical speedups for second order methods over first order methods in terms of convergence to local minima. The details are beyond our scope, and can be found in \cite{agarwal2017finding}. 

\end{enumerate}

\paragraph{Putting everything together.} 
These estimators we have studied can be used to create unbiased estimators to the Newton direction of the form 
$ \tilde{\nabla}_\x^{-2} \nabla_x $
for $\tilde{\nabla}^{-2}_\x $ which satisfies
$$\frac{1}{2} {\nabla}^{-2}f(\x_t) \preceq \tilde{\nabla}_t^{-2} \preceq 2 {\nabla}^{-2}f(\x_t) . $$

These can be incorporated into Algorithm \ref{alg:newton}, which we proved is capable of obtaining fast convergence with approximate Newton directions of this form.

\newpage
\section{Exercises}

\begin{enumerate}

\item
Prove that a single Newton step for linear regression yields the optimal solution.

\item
Let $f :\reals^d \mapsto \reals$, and consider the affine transformation $\y = A \x$, for $A \in \reals^{d \times d}$ being a symmetric matrix. Prove that 
$$ \y_{t+1} \leftarrow \y_t - \eta \nabla f(\y_t)$$
is equivalent to 
$$ \x_{t+1} \leftarrow \x_t - \eta A^{-2} \nabla f(\x_t) .$$

\item
Prove that the function $g(\z)$ defined in phase 2 of the robust Newton algorithm is $\frac{1}{2}$-strongly convex and $2$-smooth. Conclude with a proof of Theorem \ref{thm:MD-sc}.

\end{enumerate}

\newpage
\section{Bibliographic Remarks}

The modern application of Newton's method to convex optimization was put forth in the seminal work of Nesterov and Nemirovski \cite{NesterovNemirovskii94siam} on interior point methods.  A wonderful exposition is Nemirovski's lecture notes \cite{NemirovskiBook}. 

The fact that Hessian-vector products can be computed in linear time for feed forward neural networks was described in \cite{HessianPearlmutter}. 
Linear time second order methods for machine learning and the Hessian-vector product model in machine learning was introduced in \cite{agarwal2017second}.  This was extended to non-convex optimization for deep learning in \cite{agarwal2017finding}.


\chapter{Hyperparameter Optimization} 

Thus far in this class, we have been talking about continuous mathematical optimization, where the search space of our optimization problem is continuous and mostly convex.
For example, we have learned about how to optimize the weights of a deep neural network, which take continuous real values, via various optimization algorithms (SGD, AdaGrad,  Newton's method, etc.).

However, in the process of training a neural network, there are some meta parameters, which we call \emph{hyperparameters}, that have a profound effect on the final outcome. These are global, mostly discrete, parameters that are treated differently by algorithm designers as well as by engineers.
Examples include the architecture of the neural network (number of layers, width of each layer, type of activation function, ...), the optimization scheme for updating weights (SGD/AdaGrad, initial learning rate, decay rate of learning rate, momentum parameter, ...), and many more.
Roughly speaking, these hyperparameters are chosen before the training starts.

The purpose of this chapter is to formalize this problem as an optimization problem in machine learning, which requires a different methodology than we have treated in the rest of this course. We remark that hyperparameter optimization is still an active area of research and its theoretical properties are not well understood as of this time. 

\section{Formalizing the problem}

What makes hyperparameters different from ``regular" parameters? 
\begin{enumerate}
	\setlength{\itemsep}{0pt}

	\item The search space is often discrete (for example, number of layers). As such, there is no natural notion of gradient or differentials and it is not clear how to apply the iterative methods we have studied thus far. 
	
	\item Even evaluating the objective function is extremely expensive (think of evaluating the test error of the trained neural network). Thus it is crucial to minimize the number of function evaluations, whereas other computations are significantly less expensive. 
	
	\item Evaluating the function can be done in parallel. As an example, training feedforward deep neural networks over different architectures can be done in parallel. 
\end{enumerate}

More formally, we consider the following optimization problem
\begin{equation*}
    \min_{\x_i \in GF(q_i) } \quad f(\x),
\end{equation*}
where $\x$ is the representation of discrete hyperparameters, each taking value from $q_i \geq 2 $ possible discrete values and thus in $GF(q)$, the 
Galois field of order $q$.  The example to keep in mind is that the objective $f(\x)$ is the test error of the neural network trained with
hyperparameters $\x$. Note that $\x$ has a search space of size $\prod_i q_i \geq 2^n $,  exponentially large in the number of different 
hyperparameters.

\section{Hyperparameter optimization algorithms}

The properties of the problem mentioned before prohibits the use of the algorithms we have studied thus far, which are all suitable for continuous optimization.  A naive method is to perform a  grid search over all hyperparameters, but this quickly becomes infeasible.
An emerging field of research in recent years, called \emph{AutoML}, aims to choose hyperparameters automatically. The following techniques are in common use: 

\begin{itemize}
	\setlength{\itemsep}{0pt}
	\item {\bf Grid search}, try all possible assignments of hyperparameters and return the best. This becomes infeasible very quickly with $n$ - the number of hyperparameters. 
	\item {\bf Random search}, where one randomly picks some choices of hyperparameters, evaluates their function objective, and chooses the one choice of hyperparameters giving best performance.  An advantage of this method is that it is easy to implement in parallel. 
	\item {\bf Successive Halving and Hyperband},  random search combined with early stopping using multi-armed bandit techniques. These gain a small constant factor improvement over random search. 
	\item {\bf Bayesian optimization}, a statistical approach which has a prior over the objective and tries to iteratively pick an evaluation point which reduces the variance in objective value. Finally it picks the point that attains the lowest objective objective with highest confidence. This approach is sequential in nature and thus difficult to parallelize. Another important question is how to choose a good prior. 
\end{itemize}

The hyperparameter optimization problem is essentially a combinatorial optimization problem with exponentially large search space.
Without further assumptions, this optimization problem is information-theoretically  hard. Such assumptions are explored in the next section with an accompanying algorithm. 

Finally, we note that a simple but hard-to-beat benchmark is random search with double budget. That is, compare the performance of a method to that of random search, but allow random search double the query budget of your own method.

\section{A Spectral Method} 

For simplicity, in this section we consider the case in which hyperparameters are binary. This retains the difficulty of the setting, but makes the mathematical derivation simpler. The optimization problem now becomes
\begin{equation}
	\label{hyp_opt}
    \min_{\x \in \{-1, 1\}^n} \quad f(\x). 
\end{equation}

The method we describe in this section is inspired by the following key observation: 
\emph{although the whole search space of hyperparameters is exponentially large, it is often the case in practice that only a few hyperparameters together play a significant role in the performance of a deep neural network}.

To make this intuition more precise, we need some definitions and facts from Fourier analysis of Boolean functions.
\begin{fact}
	Any function $f: \{-1, 1\}^n \rightarrow [-1, 1]$ can be uniquely represented in the Fourier basis
	$$
	f(\x) = \sum_{S \subseteq [n]} \alpha_s \hat{\chi}_S (\x),
	$$
	where each Fourier basis function
	$$
	\hat{\chi}_S (\x) = \prod_{i \in S} x_i.
	$$
	is a monomial, and thus $f(\x)$ has a polynomial representation.
\end{fact}
Now we are ready to formalize our key observation in the following assumption:
\begin{assumption}
	\label{asp_sparsity}
	The objective function $f$ in the hyperparameter optimization problem \eqref{hyp_opt} is low degree and sparse in the Fourier basis, i.e.
	\begin{equation}
	\label{asp}		
	f(x) \approx \sum_{|S| \le d} \alpha_S \hat{\chi}_S (x), \quad \|\pmb{\alpha}\|_1 \le k,
	\end{equation}
	where $d$ is the upper bound of polynomial degree, and $k$ is the sparsity of Fourier coefficient $\pmb{\alpha}$ (indexed by $S$) in $\ell_1$ sense (which is a convex relaxation of $\|\pmb{\alpha}\|_0$, the true sparsity).
\end{assumption}

\begin{remark}
	Clearly this assumption does not always hold.
	For example, many deep reinforcement learning algorithms nowadays rely heavily on the choice of the random seed, which can also be seen as a hyperparameter.
	If $\x \in \{-1,1\}^{32}$ is the bit representation of a int32 random seed, then there is no reason to assume that a few of these bits should play a more significant role than the others. 
\end{remark}

Under this assumption, all we need to do now is to find out the few important sets of variables $S$'s, as well as their coefficients $\alpha_S$'s, in the approximation \eqref{asp}.
Fortunately, there is already a whole area of research, called \emph{compressed sensing}, that aims to recover a high-dimensional but sparse vector, using only a few linear measurements. 
Next, we will briefly introduce the problem of compressed sensing, and one useful result from the literature.
After that, we will introduce the Harmonica algorithm, which applies compressed sensing techniques to solve the hyperparameter optimization problem \eqref{hyp_opt}.

\subsection{Background: Compressed Sensing} 
\label{sub:background_compressed_sensing}

The problem of compressed sensing is as follows.
Suppose there is a hidden signal $\x \in \reals^n$ that we cannot observe.
In order to recover $\x$, we design a measurement matrix $\bA \in \reals^{m \times n}$, and obtain noisy linear measurements $\y = \bA \x + \pmb{\eta} \in \reals^m$, where $\pmb{\eta}$ is some random noise.
The difficulty arises when we have a limited budget for measurements, i.e. $m \ll n$.
Note that even without noise, recovering $\x$ is non-trivial since $\y = \bA \x$ is an underdetermined linear system, therefore if there is one solution $\x$ that solves this linear system, there will be infinitely many solutions.
The key to this problem is to assume that $\x$ is $k$-sparse, that is, $\|\x\|_0 \le k$.
This assumption has been justified in various real-world applications; for example, natural images tend to be sparse in the Fourier/wavelet domain, a property which forms the bases of many image compression algorithms.

Under the assumption of sparsity, the natural way to recover $\x$ is to solve a least squares problem, subject to some sparsity constraint $\|\x\|_0 \le k$. 
However, $\ell_0$ norm is difficult to handle, and it is often replaced by $\ell_1$ norm, its convex relaxation.
One useful result from the literature of compressed sensing is the following.
\begin{proposition}[Informal statement of Theorem 4.4 in \cite{rauhut2010compressive}]
\label{cs_prop}
Assume the ground-truth signal $\x \in \reals^n$ is $k$-sparse.
Then, with high probability, using a randomly designed $\bA \in \reals^{m \times n}$ that is ``near-orthogonal'' (random Gaussian matrix, subsampled Fourier basis, etc.), with $m = O(k \log(n) / \epsilon)$ and $\|\pmb{\eta}\|_2 = O(\sqrt{m})$, $\x$ can be recovered by a convex program
\begin{equation}
\label{cs_cvx}
\min_{\z \in \reals^n} \|\y - \bA \z\|_2^2 \quad {\rm s.t.} \quad \|\z\|_1 \le k,
\end{equation}
with accuracy $\|\x - \z\|_2 \le \epsilon$.	
\end{proposition}
This result is remarkable; in particular, it says that the number of measurements needed to recover a sparse signal is independent of the dimension $n$ (up to a logarithm term), but only depends on the sparsity $k$ and the desired accuracy $\epsilon$. \footnote{It also depends on the desired high-probability bound, which is omitted in this informal statement.}
\begin{remark}
	The convex program \eqref{cs_cvx} is equivalent to the following LASSO problem
	$$
	\min_{\z \in \reals^n} \|\y - \bA \z\|_2^2 + \lambda \|\z\|_1,
	$$
	with a proper choice of regularization parameter $\lambda$.
	The LASSO problem is an unconstrained convex program, and has efficient solvers, as per the algorithms we have studied in this course.
\end{remark}

\subsection{The Spectral Algorithm} 

The main idea is that, under Assumption \ref{asp_sparsity}, we can view the problem of hyperparameter optimization as recovering the sparse signal $\pmb{\alpha}$ from linear measurements.
More specifically, we need to query $T$ random samples, $f(\x_1), \dots, f(\x_T)$, and then solve the LASSO problem
\begin{equation}
\label{lasso_hyp}
\min_{\pmb{\alpha}} \quad \sum_{t=1}^{T} (\sum_{|S| \le d} \alpha_S \hat{\chi}_S (\x_t) - f(\x_t) )^2 + \lambda \|\pmb{\alpha}\|_1,
\end{equation}
where the regularization term $\lambda \|\pmb{\alpha}\|_1$ controls the sparsity of $\pmb{\alpha}$.
Also note that the constraint $|S| \le d$ not only implies that the solution is a low-degree polynomial, but also helps to reduce the ``effective'' dimension of $\pmb{\alpha}$ from $2^n$ to $O(n^d)$, which makes it feasible to solve this LASSO problem.

Denote by $S_1, \dots, S_s$ the indices of the $s$ largest coefficients of the LASSO solution, and define
$$
g(\x) = \sum_{i \in [s]} \alpha_{S_i} \hat{\chi}_{S_i} (\x),
$$
which involves only a few dimensions of $\x$ since the LASSO solution is sparse and low-degree.
The next step  is to set the variables outside $\cup_{i\in [s]} S_i$ to arbitrary values, and compute a minimizer $\x^{*} \in \arg \min g(\x)$.
In other words, we have reduced the original problem of optimizing $f(\x)$ over $n$ variables, to the problem of optimizing $g(\x)$ (an approximation of $f(\x)$) over only a few variables (which is now feasible to solve).
One remarkable feature of this algorithm is that the returned solution $\x^{*}$ may not belong to the samples $\{\x_1, \dots, \x_T\}$, which is not the case for  other existing methods (such as random search).

Using theoretical results from compressed sensing (e.g. Proposition \ref{cs_prop}), we can derive the following guarantee for the sparse recovery of $\pmb{\alpha}$ via LASSO.
\begin{theorem}[Informal statement of Lemma 7 in \cite{hazan2018hyperparameter}]
	Assume $f$ is $k$-sparse in the Fourier expansion.
	Then, with $T = O(k^2 \log(n) / \epsilon)$ samples, the solution of the LASSO problem \eqref{lasso_hyp} achieves $\epsilon$ accuracy.
\end{theorem}

Finally, the above derivation can be considered as only one stage in a multi-stage process, each iteratively setting the value of a few more variables that are the most significant.

\newpage
\section{Bibliographic Remarks}

For a nice exposition on hyperparameter optimization see \cite{recht1, recht2}, in which the the benchmark of comparing to Random Search with double queries was proposed. 

Perhaps the simplest approach to HPO is random sampling of different choices of parameters and picking the best amongst the chosen evaluations \cite{Bergstra12}. Successive Halving (SH) algorithm was introduced \cite{successive}.  Hyperband further improves SH by automatically tuning the hyperparameters in SH \cite{hyperband}.

The Bayesian optimization (BO) methodology is currently the most studied in HPO. For recent studies and algorithms of this flavor see \cite{tpe,bayesianOPT,multitaskBO,inputBO,inequBO,highDim,rbfbayesian}. 

The spectral approach for hyperparameter optimization was introduced in \cite{hazan2018hyperparameter}. 
For an in-depth treatment of compressed sensing see the survey of \cite{rauhut2010compressive}, and for Fourier analysis of Boolean functions see \cite{booleananalysis}.

\backmatter
\bibliographystyle{plain}
\bibliography{bookbib}

\begin{thebibliography}{10}

\bibitem{AbernethyHR08}
Jacob Abernethy, Elad Hazan, and Alexander Rakhlin.
\newblock Competing in the dark: An efficient algorithm for bandit linear
  optimization.
\newblock In {\em Proceedings of the 21st Annual Conference on Learning
  Theory}, pages 263--274, 2008.

\bibitem{agarwal2017finding}
Naman Agarwal, Zeyuan Allen-Zhu, Brian Bullins, Elad Hazan, and Tengyu Ma.
\newblock Finding approximate local minima faster than gradient descent.
\newblock In {\em Proceedings of the 49th Annual ACM SIGACT Symposium on Theory
  of Computing}, pages 1195--1199. ACM, 2017.

\bibitem{agarwal2018case}
Naman Agarwal, Brian Bullins, Xinyi Chen, Elad Hazan, Karan Singh, Cyril Zhang,
  and Yi~Zhang.
\newblock The case for full-matrix adaptive regularization.
\newblock {\em arXiv preprint arXiv:1806.02958}, 2018.

\bibitem{agarwal2017second}
Naman Agarwal, Brian Bullins, and Elad Hazan.
\newblock Second-order stochastic optimization for machine learning in linear
  time.
\newblock {\em The Journal of Machine Learning Research}, 18(1):4148--4187,
  2017.

\bibitem{allen2014linear}
Zeyuan Allen-Zhu and Lorenzo Orecchia.
\newblock Linear coupling: An ultimate unification of gradient and mirror
  descent.
\newblock {\em arXiv preprint arXiv:1407.1537}, 2014.

\bibitem{anil2019memory}
Rohan Anil, Vineet Gupta, Tomer Koren, and Yoram Singer.
\newblock Memory-efficient adaptive optimization for large-scale learning.
\newblock {\em arXiv preprint arXiv:1901.11150}, 2019.

\bibitem{Bach12}
Francis Bach, Simon Lacoste-Julien, and Guillaume Obozinski.
\newblock On the equivalence between herding and conditional gradient
  algorithms.
\newblock In John Langford and Joelle Pineau, editors, {\em Proceedings of the
  29th International Conference on Machine Learning (ICML-12)}, ICML '12, pages
  1359--1366, New York, NY, USA, July 2012. Omnipress.

\bibitem{Florina14}
Aur{\'e}lien Bellet, Yingyu Liang, Alireza~Bagheri Garakani, Maria-Florina
  Balcan, and Fei Sha.
\newblock Distributed frank-wolfe algorithm: A unified framework for
  communication-efficient sparse learning.
\newblock {\em CoRR}, abs/1404.2644, 2014.

\bibitem{Bergstra12}
James Bergstra and Yoshua Bengio.
\newblock Random search for hyper-parameter optimization.
\newblock {\em J. Mach. Learn. Res.}, 13:281--305, February 2012.

\bibitem{tpe}
James~S. Bergstra, R\'{e}mi Bardenet, Yoshua Bengio, and Bal\'{a}zs K\'{e}gl.
\newblock Algorithms for hyper-parameter optimization.
\newblock In J.~Shawe-Taylor, R.~S. Zemel, P.~L. Bartlett, F.~Pereira, and
  K.~Q. Weinberger, editors, {\em Advances in Neural Information Processing
  Systems 24}, pages 2546--2554. Curran Associates, Inc., 2011.

\bibitem{borwein2006convex}
J.M. Borwein and A.S. Lewis.
\newblock {\em Convex Analysis and Nonlinear Optimization: Theory and
  Examples}.
\newblock CMS Books in Mathematics. Springer, 2006.

\bibitem{boyd.convex}
S.~Boyd and L.~Vandenberghe.
\newblock {\em {Convex Optimization}}.
\newblock Cambridge University Press, March 2004.

\bibitem{bubeckOPT}
S{\'{e}}bastien Bubeck.
\newblock Convex optimization: Algorithms and complexity.
\newblock {\em Foundations and Trends in Machine Learning}, 8(3--4):231--357,
  2015.

\bibitem{CandesR09}
E.~Candes and B.~Recht.
\newblock Exact matrix completion via convex optimization.
\newblock {\em Foundations of Computational Mathematics}, 9:717--772, 2009.

\bibitem{CesaBianchiLugosi06book}
Nicol{\`o} Cesa-Bianchi and G\'abor Lugosi.
\newblock {\em Prediction, Learning, and Games}.
\newblock Cambridge University Press, 2006.

\bibitem{chenET}
Xinyi Chen, Naman Agarwal, Elad Hazan, Cyril Zhang, and Yi~Zhang.
\newblock Extreme tensoring for low-memory preconditioning.
\newblock {\em arXiv preprint arXiv:1902.04620}, 2019.

\bibitem{deng2018optimal}
Qi~Deng, Yi~Cheng, and Guanghui Lan.
\newblock Optimal adaptive and accelerated stochastic gradient descent.
\newblock {\em arXiv preprint arXiv:1810.00553}, 2018.

\bibitem{duchi2011adaptive}
John Duchi, Elad Hazan, and Yoram Singer.
\newblock Adaptive subgradient methods for online learning and stochastic
  optimization.
\newblock {\em The Journal of Machine Learning Research}, 12:2121--2159, 2011.

\bibitem{DuchiHS10}
John~C. Duchi, Elad Hazan, and Yoram Singer.
\newblock Adaptive subgradient methods for online learning and stochastic
  optimization.
\newblock In {\em {COLT} 2010 - The 23rd Conference on Learning Theory, Haifa,
  Israel, June 27-29, 2010}, pages 257--269, 2010.

\bibitem{Dudik12a}
Miroslav Dud\'{\i}k, Za\"{\i}d Harchaoui, and J{\'e}r{\^o}me Malick.
\newblock Lifted coordinate descent for learning with trace-norm
  regularization.
\newblock {\em Journal of Machine Learning Research - Proceedings Track},
  22:327--336, 2012.

\bibitem{FrankWolfe}
M.~Frank and P.~Wolfe.
\newblock An algorithm for quadratic programming.
\newblock {\em Naval Research Logistics Quarterly}, 3:149--154, 1956.

\bibitem{Garber11}
Dan Garber and Elad Hazan.
\newblock Approximating semidefinite programs in sublinear time.
\newblock In {\em NIPS}, pages 1080--1088, 2011.

\bibitem{Garber13}
Dan Garber and Elad Hazan.
\newblock Playing non-linear games with linear oracles.
\newblock In {\em FOCS}, pages 420--428, 2013.

\bibitem{inequBO}
Jacob~R. Gardner, Matt~J. Kusner, Zhixiang~Eddie Xu, Kilian~Q. Weinberger, and
  John~P. Cunningham.
\newblock Bayesian optimization with inequality constraints.
\newblock In {\em Proceedings of the 31th International Conference on Machine
  Learning, {ICML} 2014, Beijing, China, 21-26 June 2014}, pages 937--945,
  2014.

\bibitem{Goodfellow-et-al-2016}
Ian Goodfellow, Yoshua Bengio, and Aaron Courville.
\newblock {\em Deep Learning}.
\newblock MIT Press, 2016.
\newblock \url{http://www.deeplearningbook.org}.

\bibitem{GroveLS01}
A~.J.\ Grove, N.~Littlestone, and D.~Schuurmans.
\newblock General convergence results for linear discriminant updates.
\newblock {\em Machine Learning}, 43(3):173--210, 2001.

\bibitem{gupta2017unified}
Vineet Gupta, Tomer Koren, and Yoram Singer.
\newblock A unified approach to adaptive regularization in online and
  stochastic optimization.
\newblock {\em arXiv preprint arXiv:1706.06569}, 2017.

\bibitem{gupta2018shampoo}
Vineet Gupta, Tomer Koren, and Yoram Singer.
\newblock Shampoo: Preconditioned stochastic tensor optimization.
\newblock {\em arXiv preprint arXiv:1802.09568}, 2018.

\bibitem{Hannan57}
James Hannan.
\newblock Approximation to bayes risk in repeated play.
\newblock {\em In M. Dresher, A. W. Tucker, and P. Wolfe, editors,
  Contributions to the Theory of Games, volume 3}, pages 97--139, 1957.

\bibitem{Dudik12b}
Za\"{\i}d Harchaoui, Matthijs Douze, Mattis Paulin, Miroslav Dud\'{\i}k, and
  J{\'e}r{\^o}me Malick.
\newblock Large-scale image classification with trace-norm regularization.
\newblock In {\em CVPR}, pages 3386--3393, 2012.

\bibitem{OCObook}
Elad Hazan.
\newblock Introduction to online convex optimization.
\newblock {\em Foundations and Trends{\^A}{\textregistered} in Optimization},
  2(3-4):157--325, 2016.

\bibitem{HAK07}
Elad Hazan, Amit Agarwal, and Satyen Kale.
\newblock Logarithmic regret algorithms for online convex optimization.
\newblock In {\em Machine Learning}, volume 69(2--3), pages 169--192, 2007.

\bibitem{hazan2019revisiting}
Elad Hazan and Sham Kakade.
\newblock Revisiting the polyak step size.
\newblock {\em arXiv preprint arXiv:1905.00313}, 2019.

\bibitem{DBLP:conf/colt/HazanK08}
Elad Hazan and Satyen Kale.
\newblock Extracting certainty from uncertainty: Regret bounded by variation in
  costs.
\newblock In {\em The 21st Annual Conference on Learning Theory (COLT)}, pages
  57--68, 2008.

\bibitem{hazan:beyond}
Elad Hazan and Satyen Kale.
\newblock Beyond the regret minimization barrier: an optimal algorithm for
  stochastic strongly-convex optimization.
\newblock {\em Journal of Machine Learning Research - Proceedings Track}, pages
  421--436, 2011.

\bibitem{Hazan12}
Elad Hazan and Satyen Kale.
\newblock Projection-free online learning.
\newblock In {\em ICML}, 2012.

\bibitem{hazan2014beyond}
Elad Hazan and Satyen Kale.
\newblock Beyond the regret minimization barrier: optimal algorithms for
  stochastic strongly-convex optimization.
\newblock {\em The Journal of Machine Learning Research}, 15(1):2489--2512,
  2014.

\bibitem{hazan2018hyperparameter}
Elad Hazan, Adam Klivans, and Yang Yuan.
\newblock Hyperparameter optimization: A spectral approach.
\newblock {\em ICLR}, 2018.

\bibitem{hinton2012neural}
Geoffrey Hinton, Nitish Srivastava, and Kevin Swersky.
\newblock Neural networks for machine learning lecture 6a overview of
  mini-batch gradient descent.
\newblock {\em Cited on}, 14, 2012.

\bibitem{rbfbayesian}
Ilija Ilievski, Taimoor Akhtar, Jiashi Feng, and Christine~Annette Shoemaker.
\newblock Efficient hyperparameter optimization for deep learning algorithms
  using deterministic {RBF} surrogates.
\newblock In {\em Proceedings of the Thirty-First {AAAI} Conference on
  Artificial Intelligence, February 4-9, 2017, San Francisco, California,
  {USA.}}, pages 822--829, 2017.

\bibitem{Jaggi13b}
Martin Jaggi.
\newblock Revisiting frank-wolfe: Projection-free sparse convex optimization.
\newblock In {\em ICML}, 2013.

\bibitem{Jaggi10}
Martin Jaggi and Marek Sulovsk{\'y}.
\newblock A simple algorithm for nuclear norm regularized problems.
\newblock In {\em ICML}, pages 471--478, 2010.

\bibitem{successive}
Kevin~G. Jamieson and Ameet Talwalkar.
\newblock Non-stochastic best arm identification and hyperparameter
  optimization.
\newblock In {\em Proceedings of the 19th International Conference on
  Artificial Intelligence and Statistics, {AISTATS} 2016, Cadiz, Spain, May
  9-11, 2016}, pages 240--248, 2016.

\bibitem{johnson2013accelerating}
Rie Johnson and Tong Zhang.
\newblock Accelerating stochastic gradient descent using predictive variance
  reduction.
\newblock In {\em Advances in neural information processing systems}, pages
  315--323, 2013.

\bibitem{KV-FTL}
Adam Kalai and Santosh Vempala.
\newblock Efficient algorithms for online decision problems.
\newblock {\em Journal of Computer and System Sciences}, 71(3):291--307, 2005.

\bibitem{kingma2014adam}
Diederik~P Kingma and Jimmy Ba.
\newblock Adam: A method for stochastic optimization.
\newblock {\em arXiv preprint arXiv:1412.6980}, 2014.

\bibitem{KivWar97}
Jyrki Kivinen and Manfred~K. Warmuth.
\newblock Exponentiated gradient versus gradient descent for linear predictors.
\newblock {\em Inf. Comput.}, 132(1):1--63, 1997.

\bibitem{KivinenW01}
Jyrki Kivinen and Manfred~K. Warmuth.
\newblock Relative loss bounds for multidimensional regression problems.
\newblock {\em Machine Learning}, 45(3):301--329, 2001.

\bibitem{Jaggi13a}
Simon Lacoste{-}Julien, Martin Jaggi, Mark~W. Schmidt, and Patrick Pletscher.
\newblock Block-coordinate frank-wolfe optimization for structural svms.
\newblock In {\em Proceedings of the 30th International Conference on Machine
  Learning, {ICML} 2013, Atlanta, GA, USA, 16-21 June 2013}, pages 53--61,
  2013.

\bibitem{lee:practical}
J.~Lee, B.~Recht, R.~Salakhutdinov, N.~Srebro, and J.~A. Tropp.
\newblock Practical large-scale optimization for max-norm regularization.
\newblock In {\em NIPS}, pages 1297--1305, 2010.

\bibitem{hyperband}
L.~{Li}, K.~{Jamieson}, G.~{DeSalvo}, A.~{Rostamizadeh}, and A.~{Talwalkar}.
\newblock {Hyperband: A Novel Bandit-Based Approach to Hyperparameter
  Optimization}.
\newblock {\em ArXiv e-prints}, March 2016.

\bibitem{McMahanS10}
H.~Brendan McMahan and Matthew~J. Streeter.
\newblock Adaptive bound optimization for online convex optimization.
\newblock In {\em {COLT} 2010 - The 23rd Conference on Learning Theory, Haifa,
  Israel, June 27-29, 2010}, pages 244--256, 2010.

\bibitem{NY83}
Arkadi~S. Nemirovski and David~B. Yudin.
\newblock {\em Problem Complexity and Method Efficiency in Optimization}.
\newblock John Wiley UK/USA, 1983.

\bibitem{Nemirovski04lectures}
A.S. Nemirovskii.
\newblock Interior point polynomial time methods in convex programming, 2004.
\newblock Lecture Notes.

\bibitem{NemirovskiBook}
AS~Nemirovskii.
\newblock Interior point polynomial time methods in convex programming.
\newblock {\em Lecture Notes}, 2004.

\bibitem{Nesterov}
Y.~Nesterov.
\newblock A method of solving a convex programming problem with convergence
  rate {$O(1/k^2)$}.
\newblock {\em Soviet Mathematics Doklady}, 27(2):372--376, 1983.

\bibitem{NesterovBook}
Y.~Nesterov.
\newblock {\em Introductory Lectures on Convex Optimization: A Basic Course}.
\newblock Applied Optimization. Springer, 2004.

\bibitem{NesterovNemirovskii94siam}
Y.~E. Nesterov and A.~S. Nemirovskii.
\newblock {\em Interior Point Polynomial Algorithms in Convex Programming}.
\newblock SIAM, Philadelphia, 1994.

\bibitem{booleananalysis}
Ryan O'Donnell.
\newblock {\em Analysis of Boolean Functions}.
\newblock Cambridge University Press, New York, NY, USA, 2014.

\bibitem{OrabonaC10}
Francesco Orabona and Koby Crammer.
\newblock New adaptive algorithms for online classification.
\newblock In {\em Proceedings of the 24th Annual Conference on Neural
  Information Processing Systems 2010.}, pages 1840--1848, 2010.

\bibitem{HessianPearlmutter}
Barak~A Pearlmutter.
\newblock Fast exact multiplication by the hessian.
\newblock {\em Neural computation}, 6(1):147--160, 1994.

\bibitem{RSS}
Alexander Rakhlin, Ohad Shamir, and Karthik Sridharan.
\newblock Making gradient descent optimal for strongly convex stochastic
  optimization.
\newblock In {\em ICML}, 2012.

\bibitem{rauhut2010compressive}
Holger Rauhut.
\newblock Compressive sensing and structured random matrices.
\newblock {\em Theoretical foundations and numerical methods for sparse
  recovery}, 9:1--92, 2010.

\bibitem{recht1}
Benjamin Recht.
\newblock Embracing the random.
\newblock \url{http://www.argmin.net/2016/06/23/hyperband/}, 2016.

\bibitem{recht2}
Benjamin Recht.
\newblock The news on auto-tuning.
\newblock \url{http://www.argmin.net/2016/06/20/hypertuning/}, 2016.

\bibitem{Rennie:2005}
Jasson D.~M. Rennie and Nathan Srebro.
\newblock Fast maximum margin matrix factorization for collaborative
  prediction.
\newblock In {\em Proceedings of the 22Nd International Conference on Machine
  Learning}, ICML '05, pages 713--719, New York, NY, USA, 2005. ACM.

\bibitem{robbins1951}
Herbert Robbins and Sutton Monro.
\newblock A stochastic approximation method.
\newblock {\em The Annals of Mathematical Statistics}, 22(3):400--407, 09 1951.

\bibitem{rockafellar1997convex}
R.T. Rockafellar.
\newblock {\em Convex Analysis}.
\newblock Convex Analysis. Princeton University Press, 1997.

\bibitem{salakhutdinov:collaborative}
R.~Salakhutdinov and N.~Srebro.
\newblock Collaborative filtering in a non-uniform world: Learning with the
  weighted trace norm.
\newblock In {\em NIPS}, pages 2056--2064, 2010.

\bibitem{schmidt2017minimizing}
Mark Schmidt, Nicolas Le~Roux, and Francis Bach.
\newblock Minimizing finite sums with the stochastic average gradient.
\newblock {\em Mathematical Programming}, 162(1-2):83--112, 2017.

\bibitem{ShalevThesis}
Shai Shalev-Shwartz.
\newblock {\em Online Learning: Theory, Algorithms, and Applications}.
\newblock PhD thesis, The Hebrew University of Jerusalem, 2007.

\bibitem{ShalevShwartz11}
Shai Shalev-Shwartz, Alon Gonen, and Ohad Shamir.
\newblock Large-scale convex minimization with a low-rank constraint.
\newblock In {\em ICML}, pages 329--336, 2011.

\bibitem{ShwartzS07}
Shai Shalev-Shwartz and Yoram Singer.
\newblock A primal-dual perspective of online learning algorithms.
\newblock {\em Machine Learning}, 69(2-3):115--142, 2007.

\bibitem{Shalev-ShwartzSSC11}
Shai Shalev-Shwartz, Yoram Singer, Nathan Srebro, and Andrew Cotter.
\newblock Pegasos: primal estimated sub-gradient solver for svm.
\newblock {\em Math. Program.}, 127(1):3--30, 2011.

\bibitem{ShamirS11}
O.~Shamir and S.~Shalev-Shwartz.
\newblock Collaborative filtering with the trace norm: Learning, bounding, and
  transducing.
\newblock {\em JMLR - Proceedings Track}, 19:661--678, 2011.

\bibitem{SZ}
Ohad Shamir and Tong Zhang.
\newblock Stochastic gradient descent for non-smooth optimization: Convergence
  results and optimal averaging schemes.
\newblock In {\em ICML}, 2013.

\bibitem{shazeer2018adafactor}
Noam Shazeer and Mitchell Stern.
\newblock Adafactor: Adaptive learning rates with sublinear memory cost.
\newblock {\em arXiv preprint arXiv:1804.04235}, 2018.

\bibitem{bayesianOPT}
Jasper Snoek, Hugo Larochelle, and Ryan~P. Adams.
\newblock Practical bayesian optimization of machine learning algorithms.
\newblock In {\em Advances in Neural Information Processing Systems 25: 26th
  Annual Conference on Neural Information Processing Systems 2012. Proceedings
  of a meeting held December 3-6, 2012, Lake Tahoe, Nevada, United States.},
  pages 2960--2968, 2012.

\bibitem{inputBO}
Jasper Snoek, Kevin Swersky, Richard~S. Zemel, and Ryan~P. Adams.
\newblock Input warping for bayesian optimization of non-stationary functions.
\newblock In {\em Proceedings of the 31th International Conference on Machine
  Learning, {ICML} 2014, Beijing, China, 21-26 June 2014}, pages 1674--1682,
  2014.

\bibitem{SrebroThesis}
Nathan Srebro.
\newblock {\em Learning with Matrix Factorizations}.
\newblock PhD thesis, Massachusetts Institute of Technology, 2004.

\bibitem{multitaskBO}
Kevin Swersky, Jasper Snoek, and Ryan~Prescott Adams.
\newblock Multi-task bayesian optimization.
\newblock In {\em Advances in Neural Information Processing Systems 26: 27th
  Annual Conference on Neural Information Processing Systems 2013. Proceedings
  of a meeting held December 5-8, 2013, Lake Tahoe, Nevada, United States.},
  pages 2004--2012, 2013.

\bibitem{Tewari11}
Ambuj Tewari, Pradeep~D. Ravikumar, and Inderjit~S. Dhillon.
\newblock Greedy algorithms for structurally constrained high dimensional
  problems.
\newblock In {\em NIPS}, pages 882--890, 2011.

\bibitem{turing}
A.~M. Turing.
\newblock Computing machinery and intelligence.
\newblock {\em Mind}, 59(236):433--460, 1950.

\bibitem{highDim}
Ziyu Wang, Masrour Zoghi, Frank Hutter, David Matheson, and Nando de~Freitas.
\newblock Bayesian optimization in high dimensions via random embeddings.
\newblock In {\em {IJCAI} 2013, Proceedings of the 23rd International Joint
  Conference on Artificial Intelligence, Beijing, China, August 3-9, 2013},
  pages 1778--1784, 2013.

\bibitem{ward2018adagrad}
Rachel Ward, Xiaoxia Wu, and Leon Bottou.
\newblock Adagrad stepsizes: Sharp convergence over nonconvex landscapes, from
  any initialization.
\newblock {\em arXiv preprint arXiv:1806.01811}, 2018.

\bibitem{zhang2013linear}
Lijun Zhang, Mehrdad Mahdavi, and Rong Jin.
\newblock Linear convergence with condition number independent access of full
  gradients.
\newblock In {\em Advances in Neural Information Processing Systems}, pages
  980--988, 2013.

\bibitem{Zinkevich03}
Martin Zinkevich.
\newblock Online convex programming and generalized infinitesimal gradient
  ascent.
\newblock In {\em Proceedings of the 20th International Conference on Machine
  Learning}, pages 928--936, 2003.

\end{thebibliography}



\end{document}